\documentclass[]{article}

\usepackage{hyperref}       
\usepackage{url}            
\usepackage{booktabs}       
\usepackage{amsfonts}       

\usepackage{graphicx}
\usepackage{subcaption}

\usepackage{tcolorbox}
\definecolor{darkred}{RGB}{150,0,0}
\definecolor{darkgreen}{RGB}{0,150,0}
\definecolor{darkblue}{RGB}{0,0,150}
\hypersetup{colorlinks=true, linkcolor=red, citecolor=darkgreen, urlcolor=darkblue}

\usepackage{tikz}
\usepackage{amssymb}
\usepackage{amsmath}
\usepackage{amsmath,amsfonts,amssymb,amsthm}
\usepackage{mathtools}
\usepackage{commath}
\usepackage{flexisym}
\usepackage{csquotes}
\usepackage[english]{babel}
\usepackage[utf8]{inputenc}
\usepackage{color}
\usepackage{bbm}
 
\usepackage{amsmath}
\DeclareMathOperator*{\argmax}{arg\,max}

\newtheorem{lemma}{Lemma}

\newtheorem{theorem}{Theorem}
\newtheorem{myassum}{Assumption}
\theoremstyle{definition}
\newtheorem{remark}{Remark}
\theoremstyle{definition}

\newtheorem{proposition}{Proposition}

\usepackage{lipsum}

\usepackage{enumitem}
\usepackage{xcolor}
\usepackage[linesnumbered,ruled,vlined]{algorithm2e}

\SetCommentSty{mycommfont}

\SetKwInput{KwInput}{Input}                
\SetKwInput{KwOutput}{Output} 

\usepackage[margin=1in]{geometry}

\title{Decentralized Multi-Agent Linear Bandits with Safety Constraints}

\newcommand{\Dts}{\Dc_{i,t}^{\rm{s}}}
\newcommand{\Ds}{\Dc^{\rm{s}}}

\newcommand{\la}{\lambda}

\newcommand{\nn}{\nonumber}

\newcommand{\bal}{\begin{align}}
\newcommand{\eal}{\end{align}}

\DeclarePairedDelimiterX{\inp}[2]{\langle}{\rangle}{#1, #2}

\newcommand{\M}{\mathbf{M}}
\newcommand{\W}{\mathbf{W}}

\newcommand{\D}{\mathbf{D}}
\newcommand{\C}{\mathbf{C}}

\newcommand{\X}{\mathbf{X}}
\newcommand{\A}{\mathbf{A}}
\newcommand{\B}{\mathbf{B}}
\newcommand{\Vb}{\mathbf{V}}

\newcommand{\x}{\mathbf{x}}

\newcommand{\w}{\mathbf{w}}

\newcommand{\vb}{\mathbf{v}}
\newcommand{\bb}{\mathbf{b}}

\newcommand{\y}{\mathbf{y}}
\newcommand{\s}{\mathbf{s}}
\newcommand{\z}{\mathbf{z}}

\newcommand{\rb}{\mathbf{r}}
\newcommand{\n}{\mathbf{n}}

\newcommand{\Bc}{{\mathcal{B}}}

\newcommand{\Dc}{\mathcal{D}}

\newcommand{\Nc}{\mathcal{N}}
\newcommand{\Rc}{\mathcal{R}}

\newcommand{\Cc}{\mathcal{C}}

\newcommand{\Ac}{\mathcal{A}}
\newcommand{\Pc}{\mathcal{P}}
\newcommand{\Ec}{\mathcal{E}}

\newcommand{\Oc}{\mathcal{O}}

\newcommand{\beq}{\begin{equation}}
\newcommand{\eeq}{\end{equation}}
\newcommand{\bea}{\begin{align}}
\newcommand{\eea}{\end{align}}

\newcommand{\xx}{\tilde \x}

\newcommand{\Bgood}{\Bc_{\rm{good}}}

\newcommand{\U}{{\rm{Comm}}}

\author{%
  Sanae Amani\\
  University of California, Santa Barbara\\
  \texttt{samanigeshnigani@ucsb.edu}
  \and
  Christos Thrampoulidis \\
   University of California, Santa Barbara\\
  \texttt{cthrampo@ucsb.edu}
}
\begin{document}

\date{}
\maketitle
\begin{abstract}
We study decentralized stochastic linear bandits, where a network of $N$ agents acts cooperatively to efficiently solve a linear bandit-optimization problem over a $d$-dimensional space. For this problem, we propose DLUCB: a fully decentralized algorithm that minimizes the
cumulative regret over the entire network. At each round of the algorithm each agent chooses its actions following an upper confidence bound (UCB) strategy and agents share information with their immediate neighbors through a carefully designed consensus procedure that repeats over cycles. Our analysis adjusts the duration of these communication cycles ensuring near-optimal regret performance $\Oc(d\log{NT}\sqrt{NT})$ at a communication rate of $\Oc(dN^2)$ per round. The structure of the network affects the regret performance via a small additive term – coined the regret of delay – that depends on the spectral gap of the underlying graph. Notably, our results apply to arbitrary network topologies without a requirement for a dedicated agent acting as a server. In consideration of situations with high communication cost, we propose RC-DLUCB: a modification of DLUCB with rare communication among agents. The
new algorithm trades off regret performance for a significantly reduced total communication cost of $\Oc(d^3N^{2.5})$ over all $T$ rounds. Finally, we show that our ideas extend naturally to the emerging, albeit more challenging, setting of safe bandits. For the recently studied problem of
linear bandits with unknown linear safety constraints, we propose the first safe decentralized
algorithm. Our study contributes towards applying bandit techniques in safety-critical
distributed systems that repeatedly deal with unknown stochastic environments. We present
numerical simulations for various network topologies that corroborate our theoretical findings.
\end{abstract}
\section{Introduction}

Linear stochastic bandits (LB) provide simple, yet commonly encountered, models for a variety of sequential decision-making problems under uncertainty. Specifically, LB generalizes the classical multi-armed bandit (MAB) problem of $K$ arms that each yields reward sampled independently from an underlying distribution with unknown parameters, to a setting where the expected reward of each arm is a linear function that  depends on the same unknown parameter vector \cite{dani2008stochastic,abbasi2011improved,rusmevichientong2010linearly}. LBs have been successfully applied over the years in online advertising, recommendation services, resource allocation, etc. \cite{lattimore2018bandit}. More recently, researchers have explored the potentials of such algorithms in more complex systems, such as in robotics, wireless networks, the power grid, medical trials, e.g., \cite{li2013medicine,avner2019multi,berkenkamp2016bayesian,sui2018stagewise}. 
%
%
A distinguishing feature of many of these --perhaps less conventional-- bandit applications, is their \emph{distributive} nature. For example, in sensor/wireless networks \cite{avner2019multi}, a collaborative behavior is required for decision-makers/agents to select better actions as individuals, but each of them is only able to share information about the unknown environment with a subset of neighboring agents.
While a distributed nature is inherent in certain systems, distributed solutions 
might also be preferred in broader settings, as they can lead to speed-ups of the learning process.
This calls for extensions of the traditional bandit setting to networked systems. 
At the same time, in many of these applications the unknown system might be \emph{safety-critical}, i.e., the algorithm's chosen actions need to satisfy certain constraints that, importantly, are often unknown. This leads to the challenge of balancing the  goal of reward maximization with the restriction of playing safe actions.  The past few years have seen a surge of research activity in these two areas: (i) distributed \cite{wang2019distributed,martinez2019decentralized,szorenyi2013gossip,landgren2016distributed}; and (ii) safe bandits \cite{sui2015safe,sui2018stagewise,kazerouni2017conservative,amani2019linear,amani2020regret,amani2020generalized,moradipari2019safe,khezeli2019safe,pacchiano2020stochastic}.


This paper contributes to the intersection of these two emerging lines of work. Concretely, we consider the problem of decentralized multi-agent linear bandits for a general (connected) network structure of $N$ agents, who can only communicate messages with their immediate neighbors. For this, we propose and analyze the first fully-decentralized algorithm. We also present a communication-efficient version and discuss key trade-offs between regret, communication cost and graph structure. Finally, we present the first simultaneously distributed and safe bandit algorithm for a setting with unknown linear constraints.

\textbf{Notation.} 
We use lower-case letters for scalars, lower-case bold letters for vectors, and upper-case bold letters for matrices. The Euclidean-norm of $\x$ is denoted by $\norm{\x}_2$. We denote the transpose of any column vector $\x$ by $\x^{T}$. For any vectors $\x$ and $\y$, we use $\langle \x,\y\rangle$ to denote their inner product. Let $\A$ be a positive definite $d\times d$ matrix and $\boldsymbol \nu \in\mathbb R^d$. The weighted 2-norm of $\boldsymbol \nu$ with respect to $\A$ is defined by $\norm{\boldsymbol \nu}_\A = \sqrt{\boldsymbol \nu^T \A \boldsymbol \nu}$. For positive integers $n$ and $m\leq n$, $[n]$ and $[m:n]$ denote the sets $\{1,2,\ldots,n\}$ and $\{m,\ldots,n\}$, respectively. We use $\mathbf{1}$ and $\mathbf{e}_i$ to denote the vector of all $1$'s and the $i$-th standard basis vector, respectively.

\subsection{Problem formulation} \label{sec:formulate}

{\bf Decentralized Linear Bandit.} We consider a network of $N$ agents and known convex compact decision set $\Dc \subset \mathbb R^d$ (our results can be easily extended to settings with time varying decision sets). Agents play actions synchronously. At each round $t$, each agent $i$ chooses an action $\x_{i,t} \in \Dc$ and observes reward $y_{i,t} =  \langle\boldsymbol\theta_\ast,\x_{i,t}\rangle+\eta_{i,t}$, where $\boldsymbol\theta_\ast \in \mathbb{R}^d$ is an unknown vector and $\eta_{i,t}$ is random additive noise.

{\bf Communication Model.}
The agents are represented by the nodes of an
undirected and connected graph $G$. Each agent can send and receive messages only to and from its immediate neighbors. The topology of  $G$ is known to all agents via a communication matrix $\mathbf{P}$ (see  Assumption \ref{assum:comm_matrix}).


{\bf Safety.}
The learning environment might be subject to unknown constraints that restrict the choice of actions. In this paper, we model the safety constraint by a linear function depending on an \emph{unknown} vector $\boldsymbol\mu_\ast \in \mathbb{R}^d$ and a \emph{known} constant $c\in \mathbb{R}$. Specifically, the chosen action $\x_{i,t}$ must satisfy
      $\langle\boldsymbol\mu_\ast,\x_{i,t}\rangle \leq c$,
for all $i$ and $t$, with high probability. We define the unknown safe set as
 $   \Ds(\boldsymbol\mu_\ast):=\{\x\in \Dc\,:\, \langle\boldsymbol\mu_\ast,\x\rangle \leq c\}$.
After playing $\x_{i,t}$, agent $i$ observes bandit-feedback measurements $z_{i,t}=\langle\boldsymbol\mu_\ast,\x_{i,t}\rangle+\zeta_{i,t}$. This type of safety constraint, but for single-agent settings, has been recently introduced and studied in \cite{amani2019linear,pacchiano2020stochastic,sui2015safe,sui2018stagewise,moradipari2019safe}. See also \cite{kazerouni2017conservative,khezeli2019safe} for related notions of safety studied recently in the context of single-agent linear bandits.

{\bf Goal.}
Let $T$ be the total number of rounds. We define the cumulative regret of the entire network as:
$ R_T := \sum_{t=1}^T\sum_{i=1}^N \langle\boldsymbol\theta_\ast,\x_\ast\rangle-\langle\boldsymbol\theta_\ast,\x_{i,t}\rangle.$
The optimal action $\x_{\ast}$ is defined with respect to $\Dc$ and $\Ds(\boldsymbol\mu_\ast)$ as $\argmax_{\x\in \Dc}\langle\boldsymbol\theta_\ast,\x\rangle$ and $\argmax_{\x\in \Ds(\boldsymbol\mu_\ast)}\langle\boldsymbol\theta_\ast,\x\rangle$ in the original and safe settings, respectively. 
The goal is to minimize the cumulative regret, while each agent is allowed to share $poly(Nd)$ values per round to its neighbors. Specifically, we wish to achieve a regret close to that incurred by an optimal \emph{centralized algorithm for $NT$ rounds} (the total number of plays). In the presence of safety constraint, in addition to the aforementioned goals, agents' actions must also
satisfy the safety constraint at each round.
\subsection{
Contributions}
\textbf{DLUCB.} We propose a fully decentralized linear bandit algorithm (DLUCB), at each round of which, the agents simultaneously share information among each other and pick their next actions. We prove a regret bound that captures both the degree of selected actions' optimality and the inevitable delay in information-sharing due to the network structure. See Section \ref{sec:gossip} and \ref{sec:DLUCB}. Compared to existing distributed LB algorithms, ours can be implemented (and remains valid) for any arbitrary (connected) network without requiring a peer-to-peer network structure or a master node. See Section \ref{sec:tradeoff}.
    
\textbf{RC-DLUCB.}~We propose a fully decentralized algorithm with rare communication (RC-DLUCB) to reduce the communication cost (total number of values communicated during the run of algorithm) for applications that are sensitive to high communication cost. See Section \ref{sec:rarecomm}


 \textbf{Safe-DLUCB.} We present and analyze the first fully decentralized algorithm for \emph{safe} LBs with linear constraints. Our algorithm provably achieves regret of the same order (wrt. $NT$) as if no constraints were present. See Section \ref{sec:Safe-DLUCB}
 We complement our theoretical results with numerical simulations under various settings in Section~\ref{sec:sim}.
%
\subsection{Related works}
%
{\bf Decentralized Bandits.} There are several recent works on decentralized/distributed stochastic MAB problems. 
In the context of the classical $K$-armed MAB, \cite{martinez2019decentralized,landgren2016distributed,landgr} proposed decentralized algorithms for a network of $N$ agents that can share information only with their immediate neighbors, while \cite{szorenyi2013gossip} studies the MAB problem on peer-to-peer networks. More recently, \cite{wang2019distributed} focuses on communication efficiency and presented $K$-armed MAB algorithms with significantly lower communication overhead.
In contrast to these, here, we study a LB model.
The most closely related  works on distributed/decentralized LB are \cite{wang2019distributed} and \cite{korda2016distributed}.
In  \cite{wang2019distributed}, the authors present a communication-efficient algorithm that operates under the coordination of a central server, such that 
every agent has instantaneous access to the full network information through the server.
This model differs from the fully decentralized one considered here. In another closely related work, \cite{korda2016distributed} studies distributed LBs in peer-to-peer networks, where each agent can
only send information to one other randomly chosen agent, not
necessarily its neighbor, per round.
A feature, in common with our algorithm, is the delayed use of bandit feedback, but the order of the delay differs between the two, owing to the  different  model. Please also see Section~\ref{sec:tradeoff} for a more elaborate comparison.
To recap, even-though motivated by the aforementioned works, our paper presents the first \emph{fully decentralized algorithm} for the \emph{multi-agent LB problem} on a general network topology, with communication between any two neighbors in the network. Furthermore, non of the above has studied the presence of safety constraints.

{\bf Safe Bandits.}
In a more general context, the notion of safe learning has many diverse definitions in the literature. Specifically, safety in bandit problems has itself received significant attention in recent years, e.g. \cite{sui2015safe,sui2018stagewise,kazerouni2017conservative,amani2019linear,amani2020regret,amani2020generalized,moradipari2019safe,khezeli2019safe,pacchiano2020stochastic}. 
To the best of our knowledge, all existing works on MAB/LB problems with safety constraints study a single-agent. As mentioned in Section \ref{sec:formulate}, the multi-agent safe  LB studied here is a canonical extension of the single-agent setting studied in \cite{amani2019linear,moradipari2019safe,pacchiano2020stochastic}. Accordingly, our algorithm and analysis builds on ideas introduced in this prior work and extends them to multi-agent collaborative learning.


\section{Decentralized Linear Algorithms}
In this section, we present \emph{Decentralized Linear Upper Confidence Bound} (DLUCB). Starting with a high-level description of the gossip communication protocol and of the benefits and challenges it brings to the problem in Section~\ref{sec:gossip}, we then explain DLUCB Algorithm \ref{alg:DLUCB} in Section~\ref{sec:DLUCB}. In Section \ref{sec:rarecomm} we present a communication-efficient version of DLUCB. Finally, in Section \ref{sec:tradeoff} we compare our algorithms to prior art.
Throughout this section, we do \emph{not} assume any safety constraints.
Below, we
introduce some necessary assumptions.

\begin{myassum}[Communication Matrix]\label{assum:comm_matrix}
For an undirected connected graph $G$ with $N$ nodes, $\mathbf{P}\in \mathbb{R}^{N\times N}$ is a symmetric communication matrix if it satisfies the following three conditions: (i) $\mathbf{P}_{i,j}=0$ if there is no connection between nodes $i$ and $j$; (ii) the sum of each row and column of $\mathbf{P}$ is 1; (iii) the eigenvalues are real and their magnitude is less than 1, 
i.e., $1=|\la_1|>|\la_2|\geq\ldots|\la_N|\geq0$.  We assume that agents have knowledge of communication matrix $\mathbf{P}$.
\end{myassum}

We remark that $\mathbf{P}$ can be constructed with little global information about the graph, such as its adjacency matrix and the graph's maximal degree; see Section~\ref{sec:sim} for an explicit construction.
Once $\mathbf{P}$ is known, the total number of agents $N$ and the graph's spectral gap $1-|\la_2|$ are also known. We show in Section~\ref{sec:DLUCB} that the latter two parameters fully capture how the network structure affects the algorithm's regret.




\begin{myassum}[Subgaussian Noise]\label{assum:noise} For $i\in[N]$ and $t>0$, $\eta_{i,t}$,  $\zeta_{i,t}$ are zero-mean $\sigma$-subGaussian random variables.
\end{myassum}

\begin{myassum}[Boundedness]\label{assum:boundedness} Without loss of generality, $\norm{\x}_2\leq 1$ for all $\x\in \Dc$, $\norm{\boldsymbol\theta_\ast}_2\leq 1,$ and $\norm{\boldsymbol\mu_\ast}_2\leq 1$.
\end{myassum}

\subsection{Information-sharing protocol}\label{sec:gossip}
%
DLCUB implements a UCB strategy. At the core of single-agent UCB algorithms, is the construction of
a proper confidence set around the true parameter $\boldsymbol\theta_\ast$ using past actions and their observed rewards. In multi-agent settings, each agent $i\in[N]$ maintains their own confidence set $\Cc_{i,t}$ at every round $t$. To exploit the network structure and enjoy the benefits of collaborative learning, it is important that $\Cc_{i,t}$ is built using information about past actions of not only agent $i$ itself, but also of agents $j\neq i\in[N]$.
For simplicity, we consider first a centralized setting of perfect information-sharing among agents. Specifically, assume that at every round $t$, agent $i$ knows the past chosen actions and their observed rewards by \emph{all} other agents in the graph. Having gathered all this information, each agent $i$ maintains knowledge of the following sufficient statistics 
during all rounds $t$:
\begin{equation}\label{eq:sufficient}
 \hspace{-0.04in}   \A_{\ast,t} = \la I+\sum_{\tau=1}^{t-1}\sum_{i=1}^N \x_{i,\tau}\x_{i,\tau}^T, ~~ \bb_{\ast,t} = \sum_{\tau=1}^{t-1}\sum_{i=1}^N y_{i,\tau}\x_{i,\tau}.
\end{equation}
Here, $\la \geq 1$ is a regularization parameter. Of course, in this idealized scenario, the confidence set constructed based on \eqref{eq:sufficient} is the same for every agent. In fact, it is the same as the confidence set that would be constructed by a single-agent that is allowed to choose $N$ actions at every round.
Here, we study a decentralized setting with imperfect information sharing. In particular, each agent $i$ can only communicate with its immediate neighbors $j\in\Nc(i)$ at any time $t$. As such, it does \emph{not} have direct access to the ``perfect statistics" $\A_{\ast,t}$ and $\bb_{\ast,t}$ in \eqref{eq:sufficient}. Instead, it is confined to approximations of them, which we denote by $\A_{i,t}$ and $\bb_{i,t}$. At worst-case, where no communication is used, $\A_{i,t}=\la I+\sum_{\tau=1}^{t-1} \x_{i,\tau}\x_{i,\tau}^T$  (similarly for $\bb_{i,t}$). But, this is a very poor approximation of $\A_{\ast,t}$ (correspondingly, $\bb_{\ast,t}$). Our goal is to construct a communication scheme that exploits exchange of information among agents to allow for drastically better approximations of \eqref{eq:sufficient}. Towards this goal, our algorithm implements an appropriate \emph{gossip protocol} to communicate each agent's past actions and observed rewards to the rest of the network (even beyond immediate neighbors). We describe the details of this protocol next. 


{\bf Running Consensus.}
In order to share information about agents' past actions among the network, we rely on \emph{running consensus}, e.g., \cite{lynch1996distributed,xiao2004fast}. The goal of running consensus is that after enough rounds of communication, each agent has an accurate estimate of the average (over all agents) of the initial values of each agent. Precisely, let $\boldsymbol{\nu}_0\in\mathbb{R}^N$ be a vector, where each entry $\boldsymbol{\nu}_{0,i}, i\in[N]$ represents agent's $i$ information at some initial round. Then, running consensus aims at providing an accurate estimate of the average $\frac{1}{N}\sum_{i\in[N]}\boldsymbol{\nu}_{0,i}$ at each agent.
Note that encoding $\boldsymbol{\nu}_0=X\mathbf{e}_j$, allows all agents to eventually get an estimate of the value $X=\sum_{i\in[N]}\boldsymbol{\nu}_{0,i}$ that was initially known only to agent $j$.
%
To see how this is relevant to our setting recall \eqref{eq:sufficient} and focus at $t=2$ for simplicity.
At round $t=2$, each agent $j$ only knows $\mathbf{x}_{j,1}$ and estimation of $\mathbf{A}_{*,2} = \sum_{i=1}^N{\mathbf{x}_{i,1}\x_{i,1}^T}$ by agent $j$ boils down to estimating each $\x_{i,1}, i\neq j$. 
In our previous example, let $X$ be $k$-th entry of $\x_{i,1}$ for some $i\neq j$.
By running consensus on $\boldsymbol\nu_0=[\x_{j,1}]_k\mathbf{e}_j$ for $k\in[d]$, every agent eventually builds an accurate estimate of $[\x_{j,1}]_k$, the $k$-th entry of $\x_{j,1}$ that would otherwise only be known to $j$.
%
 It turns out that the communication matrix $\mathbf{P}$ defined in Assumption \ref{assum:comm_matrix} plays a key role in reaching consensus.
The details are standard in the rich related literature \cite{xiao2004fast,lynch1996distributed}. Here, we only give a brief explanation of the high-level principles. Roughly speaking, a consensus algorithm updates $\boldsymbol{\nu}_{0}$ by $\boldsymbol{\nu}_1=\mathbf{P}\boldsymbol{\nu}_{0}$ and so on. Note that this operation respects the network structure since the updated value $\boldsymbol{\nu}_{1,j}$ is a weighted average of only $\boldsymbol{\nu}_{0,j}$ itself and neighbor-only values $\boldsymbol{\nu}_{0,i},i\in\Nc(j).$ Thus, after $S$ rounds, agent $j$ has access to entry $j$ of  $\boldsymbol{\nu}_{S}=\mathbf{P}^S\boldsymbol{\nu}_{0}$. This is useful because $\mathbf{P}$ is well-known to satisfy the following mixing property: $\lim_{S\to\infty} \mathbf{P}^S = \mathbf{1}\mathbf{1}^T/N$ \cite{xiao2004fast}. Thus, $\lim_{S\to\infty} [\boldsymbol{\nu}_{S}]_j = \frac{1}{N}\sum_{i=1}^N\boldsymbol{\nu}_{0,i}, \forall j\in[N]$, as desired. Of course, in practice, the number $S$ of communication rounds is finite, leading to an $\epsilon$-approximation of the average.

{\bf Accelerated-Consensus.} In this paper, we adapt \emph{polynomial filtering} introduced in \cite{martinez2019decentralized,seaman2017optimal} to speed up the mixing of information by following an approach whose convergence rate is faster than the standard multiplication method  above.
 Specifically, after $S$ communication rounds, instead of $\mathbf{P}^S$, agents compute  and apply to the initial vector $\boldsymbol{\nu}_0$ an appropriate re-scaled \emph{Chebyshev polynomial} $q_S(\mathbf{P})$ of degree $S$ of the communication matrix. 
Recall that Chebyshev polynomials 
are defined recursively. It turns out that the Chebyshev polynomial of degree $\ell$ for a communication matrix $\mathbf{P}$ is also given by a recursive formula as follows: 
$
    q_{\ell+1}(\mathbf{P})= \frac{2w_\ell}{|\la_2|w_{\ell+1}}\mathbf{P}q_{\ell}(\mathbf{P})-\frac{w_{\ell-1}}{w_{\ell+1}}q_{\ell-1}(\mathbf{P}),
$
where $w_0 = 0, w_1 = 1/|\la_2|$, $w_{\ell+1}=2w_\ell/|\la_2|-w_{\ell-1}$, $q_0(\mathbf{P})=I$ and $q_1(\mathbf{P})=\mathbf{P}$.
 Specifically, in a Chebyshev-accelerated gossip protocol \cite{martinez2019decentralized}, the agents update their estimates of the average of the initial vector's $\boldsymbol{\nu}_0$ entries as follows:
\begin{align}
   \boldsymbol \nu_{\ell+1} = {(2w_\ell)}/{(|\la_2|w_{\ell+1})}\mathbf{P}\boldsymbol \nu_{\ell}-{(w_{\ell-1}}/{w_{\ell+1})}\boldsymbol \nu_{\ell-1}.\label{eq:recursion}
\end{align}
Our algorithm DLUCB, which we present later in Section~\ref{sec:DLUCB}, implements the Checyshev-accelerated gossip protocol outlined above; see  \cite{martinez2019decentralized} for a similar implementation only for the classical $K$-Armed MAB. Specifically, we summarize the accelerated communication step described in \eqref{eq:recursion} with a function ${\rm Comm}(x_{\rm now},x_{\rm prev},\ell)$ with three inputs as follows: (1) $x_{\rm now}$ is the quantity of interest that the agent wants to update at the current round, (2) $x_{\rm prev}$ is the estimated value for the same quantity of interest that the agent updated in the previous round (cf. $\boldsymbol{\nu}_{\ell-1}$ in \eqref{eq:recursion}); (3) $\ell$ is the current communication round. Note that inputs here are scalars, however, matrices and vectors can also be passed as inputs, in which case $\rm Comm$ runs entrywise. For a detailed description of $\rm Comm$ please refer to Algorithm \ref{alg:comm} in Appendix~\ref{sec:comm}.



The accelerated consensus algorithm implemented in $\rm Comm$ guarantees fast mixing of information thanks to the following key property  \cite[Lem.~3]{martinez2019decentralized}: for $\epsilon\in(0,1)$ 
and any vector $\boldsymbol{\nu_0}$ in the $N$-dimensional simplex, it holds that 
\begin{equation}\label{eq:cheb}
\hspace{-0.08in}   \|Nq_S(\mathbf{P})\boldsymbol{\nu_0} - {\mathbf{1}}\|_2\leq {\epsilon},~\text{provided}~S= \frac{\log(2N/\epsilon)}{\sqrt{2\log(1/|\la_2|)}}.~
\end{equation}
In view of this, our algorithm properly calls $\rm Comm$ (see Algorithm \ref{alg:DLUCB}) such that for every $i\in[N]$ and $t\in[T]$, the action $\x_{i,t}$ and corresponding reward $y_{i,t}$ are communicated within the network for $S$ rounds. At round $t+S$, 
agent $i$ has access to $a_{i,j}\x_{j,t}$ and $a_{i,j}y_{j,t}$ where $a_{i,j}=N[q_S(\mathbf{P})]_{i,j}$. Thanks to \eqref{eq:cheb}, $a_{i,j}$ is $\epsilon$ close to $1$, thus, these are good approximations of the true $\x_{j,t}$ and $y_{j,t}$.
Accordingly, at the beginning of round $t>S$, each agent $i$ computes  
\begin{align}
    \hspace{-0.05in} \A_{i,t}&:= \la I+ \sum_{\tau=1}^{t-S}\sum_{j=1}^N a_{i,j}^2 \x_{j,\tau} \x_{j,\tau}^T,~\bb_{i,t}:= \sum_{\tau=1}^{t-S}\sum_{j=1}^N a_{i,j}^2  y_{j,\tau}\x_{j,\tau}
     \label{eq:apprx0},
\end{align}
 which are agent $i$'s approximations of the sufficient statistics $\A_{\ast,t-S+1}$ and $\bb_{\ast,{t-S+1}}$ defined in \eqref{eq:sufficient}. On the other hand, for rounds $1\leq t\leq S$ (before any mixing has been completed), let $\A_{i,t}=\la I+\sum_{\tau=1}^{t-1}\x_{i,\tau}\x_{i,\tau}^T$ and $\bb_{i,t}=\sum_{\tau=1}^{t-1}y_{i,\tau}\x_{i,\tau}$ for $i\in[N]$. With these, at the beginning of each round $t\in[T]$, agent $i$ constructs the confidence set
\begin{align}
    \Cc_{i,t}:=\{\boldsymbol \nu \in \mathbb{R}^d: \|{\boldsymbol \nu-\hat {\boldsymbol\theta}_{i,t}}\|_{\A_{i,t}} \leq \beta_t\}, \label{eq:thetaconfidenceset}
\end{align}
where $\hat{\boldsymbol\theta}_{i,t} = \A_{i,t}^{-1}\bb_{i,t}$ and $\beta_t$ is chosen as in Thm.~\ref{thm:confidencesettheta} below to guarantee $\boldsymbol\theta_\ast \in \Cc_{i,t}$ with high probability.
\begin{theorem}[Confidence sets]\label{thm:confidencesettheta} Let Assumptions \ref{assum:comm_matrix}, \ref{assum:noise} and \ref{assum:boundedness} hold. Fix $\epsilon\in(0,1)$ and $S$ as in \eqref{eq:cheb}. For $\delta\in(0,1)$, let 
    $$\beta_t:= (1+\epsilon)\sigma \sqrt{d\log\Big(\frac{2\la d N+2N^2t}{\la d \delta}\Big)}+\la^{1/2}.$$
Then with probability at least $1-\delta$, for all $i\in[N]$ and $t\in[T]$ it holds that $\boldsymbol\theta_\ast\in \Cc_{i,t}$.
\end{theorem}
The proof is mostly adapted from \cite[Thm.~2]{abbasi2011improved} with necessary modifications to account for the imperfect information; see Appendix~\ref{proof:thetaconfidenceset0}. 

\subsection{Decentralized Linear UCB}\label{sec:DLUCB}
We now describe DLUCB Algorithm \ref{alg:DLUCB} (see Appendix~\ref{app:detailedversion} for a more detailed version). Each agent runs DLUCB in a parallel/synchronized way. For concreteness, let us focus on agent $i\in [N]$. 
At every round $t$, the agent maintains the following first-in first-out (FIFO) queues of size at most $S$: $\Ac_{i,t}$, $\Bc_{i,t}$, $\Ac_{i,t-1}$, and $\Bc_{i,t-1}$. The queue $\Ac_{i,t}$ contains agent $i$'s estimates of all actions played at rounds $[t-S:t-1]$. Concretely, its $j$-th member, denoted by $\Ac_{i,t}(j)\in \mathbb{R}^{N\times d}$, is a matrix whose $k$-th row is agent $i$'s estimate of agent $k$'s action played at round $t+j-S-1$. Similarly, we define $\mathcal{B}_{i,t}$ as the queue containing agent $i$'s estimates of rewards observed at rounds $[t-S:t-1]$.
At each round $t$, agent $i$ sends every member of $\Ac_{i,t}$ and $\Bc_{i,t}$ (each entry of them) to its neighbors and at the same time it receives the corresponding values from them.  The received values are used to update the information stored in $\Ac_{i,t}$ and $\Bc_{i,t}$. 
The update is implemented by the sub-routine $\rm Comm$ outlined in Section~\ref{sec:gossip} and presented in detail in Appendix~\ref{sec:comm}.

At the beginning of rounds $t>S$ when, the information of rounds $[t-S]$ is mixed enough, agent $i$ updates its estimates $\A_{i,t}$ and $\bb_{i,t}$ of $\A_{\ast,t-S}$ and $\bb_{\ast,t-S}$, respectively. Using these, it creates the confidence set $\Cc_{i,t}$ and runs the UCB decision rule of Line \ref{line10} to select an action.
Next, agent $i$ updates $\Ac_{i,t}$ and $\Bc_{i,t}$ in Lines \ref{line11} and \ref{line12}, by eliminating  the first elements (dequeuing)  $\Ac_{i,t}(1)$ and $\Bc_{i,t}(1)$ of the queues $\Ac_{i,t}$ and $\Bc_{i,t}$ and adding the following elements at their end (enqueuing). At $\Ac_{i,t}$ it appends $\X_{i,t}\in \mathbb{R}^{N\times d}$, whose rows are all zero but its $i$-th row which is set to $\x_{i,t}^T$. Concurrently, at $\Bc_{i,t}$, it appends  $\y_{i,t}\in \mathbb{R}^{N}$, whose elements are all zero but its $i$-th element which is set to $y_{i,t}$. Note that $\X_{i,t}$ (similarly, $\y_{i,t}$) contains agent $i$'s estimates of actions at round $t$, and the zero rows will be updated with agent $i$'s estimates of other agents' information at round $t$ in  future rounds. This is achieved via 
calling the consensus algorithm $\rm Comm$ in Lines \ref{line13} and \ref{line14}, with which agent $i$  communicates all the members of $\Ac_{i,t}$ and $\Bc_{i,t}$ with its neighbors. 
\begin{algorithm}[t]
\DontPrintSemicolon
  \KwInput{$\Dc$, $N$, $d$, $|\la_2|$, $\epsilon$, $\la$, $\delta$, $T$}
$S = \log(2N/\epsilon)/\sqrt{2\log(1/|\la_2|)}$\;
$\A_{i,1} = \la I$, $\bb_{i,1}=\mathbf{0}$, $\Ac_{i,0} =  \Ac_{i,1}  =\Bc_{i,0} = \Bc_{i,1}=\emptyset$\;
\For{$t=1,\ldots,S$}
{Play $\x_{i,t} = \argmax_{\x\in \Dc}\max_{\boldsymbol \nu \in \Cc_{i,t}}\langle\boldsymbol \nu,\x\rangle$ and observe $y_{i,t}$. \label{DLUCB:line3}  \;
$\Ac_{i,t}.{\rm append}(\X_{i,t})$ and $\Bc_{i,t}.{\rm append}(\y_{i,t})$\;
$\Ac_{i,t+1} = \U(\Ac_{i,t},\Ac_{i,t-1},\{t,t-1,\ldots,1\})$\;
$\Bc_{i,t+1} = \U(\Bc_{i,t},\Bc_{i,t-1},\{t,t-1,\ldots,1\})$  \tcp*{Comm runs for each member of $\Ac_{i,t}$ and $\Bc_{i,t}$}
$\A_{i,t+1} = \A_{i,t} + \x_{i,t}\x^T_{i,t}$, $\bb_{i,t+1} = \bb_{i,t}+ y_{i,t}\x_{i,t}$\;
}
 $\A_{i,S}=\la I$, $\bb_{i,S}=\mathbf{0}$\;
\For{$t=S+1,\ldots,T$}
{
$\A_{i,t} = \A_{i,t-1} + N^2\Ac_{i,t}(1)^T\Ac_{i,t}(1)$, $\bb_{i,t} = \bb_{i,t-1} + N^2\Ac_{i,t}(1)^T\Bc_{i,t}(1)$\;
Play $\x_{i,t} = \argmax_{\x\in \Dc}\max_{\boldsymbol \nu \in \Cc_{i,t}}\langle\boldsymbol \nu,\x\rangle$ and observe $y_{i,t}$\label{line10}\;
 $\Ac_{i,t}.{\rm remove}\Big(\Ac_{i,t}(1)\Big).{\rm append}\left(\X_{i,t}\right)$\label{line11}\;
$\Bc_{i,t}.{\rm remove}\left(\Bc_{i,t}(1)\right).{\rm append}\left(\y_{i,t}\right)$\label{line12}\;
$\Ac_{i,t+1} = \U(\Ac_{i,t},\Ac_{i,t-1}(2:S),\{S,S-1,\ldots,1\})$ \label{line13}\;
$\Bc_{i,t+1} = \U(\Bc_{i,t},\Bc_{i,t-1}(2:S),\{S,S-1,\ldots,1\})$\label{line14}
}
  \caption{DLUCB for Agent $i$}
\label{alg:DLUCB}
\end{algorithm}

{\bf Regret analysis.} There are two key challenges in the analysis of DLUCB compared to that of single-agent LUCB. First, information sharing is imperfect: the consensus algorithm mixes information for a finite number $S$ of communication rounds resulting in $\epsilon$-approximations of the desired quantities (cf. \eqref{eq:cheb}). Second, agents can use this (imperfect) information to improve their actions only after an inevitable delay. To see what changes in the analysis of regret, consider the standard decomposition of agent $i$'s instantaneous regret at round $t$:
    $r_{i,t} =\langle\boldsymbol\theta_\ast,\x_\ast\rangle-  \langle\boldsymbol\theta_\ast,\x_{i,t}\rangle\leq 2\beta_t \norm{\x_{i,t}}_{\A_{i,t}^{-1}}.
    $
Using Cauchy-Schwartz inequality, an upper bound on the cumulative regret $\sum_{t=1}^{T}\sum_{i=1}^N r_{i,t}$ can be obtained by bounding the following key term:
\begin{align}\label{eq:wecannot}
    \sum_{t=1}^{T}\sum_{i=1}^N\norm{\x_{i,t}}_{\A_{i,t}^{-1}}^2.
\end{align} 
We do this in two steps, each addressing one of the above challenges. First, in Lemma \ref{lemm:first1}, we address the influence of imperfect information by relating the $\A_{i,t}^{-1}$--norms in \eqref{eq:wecannot}, with those in terms of their perfect information counterparts ${\A^{-1}_{\ast,t-S+1}}$. 
Hereafter, let $\A_{\ast,t} ={\la}I$ for $t = -S,\ldots,0,1$.

\begin{lemma}[Influence of imperfect information]\label{lemm:first1}
Fix any $\epsilon\in(0,1/(4d+1))$ and choose $S$ as in \eqref{eq:cheb}. Then, for all $i\in[N]$, $t\in[T]$ it holds that
$\norm{\x_{i,t}}_{\A_{i,t}^{-1}}^2 \leq e \norm{\x_{i,t}}_{{\A^{-1}_{\ast,t-S+1}}}^2.$
\end{lemma}

The intuition behind the lemma comes from the discussion on the accelerated protocol in Section~\ref{sec:gossip}. Specifically, with sufficiently small communication-error $\epsilon$ (cf. \eqref{eq:cheb}), $\A_{i,t}$ (cf. \eqref{eq:apprx0}) is a good approximation of ${\A_{\ast,t-S+1}}$ (cf. \eqref{eq:sufficient}).  The lemma replaces the task of bounding \eqref{eq:wecannot} with that of bounding $ \sum_{t=1}^{T}\sum_{i=1}^N\norm{\x_{i,t}}_{{\A^{-1}_{\ast,t-S+1}}}^2$. Unfortunately, this remains challenging. Intuitively, the reason for this is the mismatch of information about past actions in the gram matrix ${\A_{\ast,t-S+1}}$ at time $t$, compared to the inclusion of all terms $\x_{i,\tau}$ up to time $t$ in \eqref{eq:wecannot}. Our idea is to relate $\norm{\x_{i,t}}_{{\A^{-1}_{\ast,t-S+1}}}$ to $\norm{\x_{i,t}}_{\B_{i,t}^{-1}}$, where $\B_{i,t} = \A_{\ast,t}+\sum_{j=1}^{i-1} \x_{j,t} \x_{j,t}^T$. This is possible thanks to the following lemma.

\begin{lemma}[Influence of delays]\label{lemm:second1}
Let $S$ as in \eqref{eq:cheb}. Then,
$
     \norm{\x_{i,t}}_{\A_{\ast,t-S+1}^{-1}}^2 \leq e\norm{\x_{i,t}}_{\B_{i,t}^{-1}}^2, 
$
is true for all pairs $(i,t)\in [N]\times [T]$ except for at most  $\psi(\lambda,|\la_2|,\epsilon, d, N,T) := Sd\log\big(1+\frac{NT}{d\la}\big)$ of them.
\end{lemma}
Using Lemmas \ref{lemm:first1} and \ref{lemm:second1} allows controlling the regret of all actions, but at most $\psi$ of them, using standard machinery in the analysis of UCB-type algorithms. The proofs of Lemmas \ref{lemm:first1} and \ref{lemm:second1} and technical details relating the results to a desired regret bound are deferred to Appendix~\ref{proof:DLUCBregret}. The theorem below is our first main result and bounds the regret of DLUCB.
\begin{theorem}[Regret of DLUCB]\label{thm:DLUCBregret} Fix
$\epsilon\in(0,1/(4d+1))$ and $\delta\in(0,1)$. Let Assumptions \ref{assum:comm_matrix}, \ref{assum:noise}, \ref{assum:boundedness} hold, and $S$ be chosen as in \eqref{eq:cheb}. Then, with probability at least $1-\delta$, it holds that:
$$      R_T\leq
        2Sd\log\Big(1+\frac{NT}{d\la}\Big)\nn +2e\beta_{T}\,\sqrt{2dNT\log\Big(\la+\frac{NT}{d}\Big)}.$$
\end{theorem}
The regret bound has two additive terms: a small term $2\psi(\lambda,|\la_2|,\epsilon, d, N,T)$ (cf. Lemma \ref{lemm:second1}), which we call \emph{regret of delay}, and, a second main term that (notably) is of the same order as the regret of a centralized problem where communication is possible between any two nodes (see Table \ref{table:comp}). 
Thm.~\ref{thm:DLUCBregret} holds for small $\epsilon\leq 1/(4d+1)$. In Appendix~\ref{proof:DLUCBregret}, we also provide a general regret bound for arbitrary $\epsilon\in(0,1)$.

\subsection{DLUCB with Rare Communication}\label{sec:rarecomm}
As discussed in more detail in Section~\ref{sec:tradeoff}, DLUCB achieves order-wise optimal regret, but its communication cost scales as $\Oc(dN^2T)$, i.e., linearly with the horizon duration $T$ (see Table \ref{table:comp}). In this section, we present a modification tailored to communication settings that are sensitive to communication cost. The new algorithm -- termed RC-DLUCB -- is also a fully decentralized algorithm that trade-offs a slight increase in the regret performance, while guaranteeing a significantly reduced communication cost of  $\Oc\Big(d^3N^{2.5}\frac{\log(Nd)}{{\log^{1/2}(1/|\la_2|)}}\Big)$ {over the \emph{entire} horizon $[T]$}. We defer a detailed description (see Algorithm \ref{alg:DLUCB2}) and analysis (see Thms.~\ref{thm:RC-DLUCBregret} and \ref{thm:RC-DLUCBcommcost}) of RC-DLUCB in Appendix~\ref{sec:variants}. At a high-level, we design RC-DLUCB inspired by the Rarely Switching OFUL algorithm by \cite{abbasi2011improved}. In contrast to the Rarely Switching OFUL algorithm that is designed to save on computations in single-agent systems, RC-DLUCB incorporates a similar idea in our previous DLUCB to save on communication rounds.  Specifically, compared to DLUCB where communication happens at each round, in RC-DLUCB agents continue selecting actions individually (i.e., with no communication), unless a certain condition is triggered by any one of them. Then, they all switch to a communication phase, in which they communicate the unmixed information they have gathered for a duration of $S$ rounds. Roughly speaking, an agent triggers the communication phase only once it has gathered enough new information compared to the last update by the rest of the network. This can be measured by keeping track of the variations in the corresponding gram matrix.

\begin{table*}[t]
\centering 
\begin{tabular}{c c c} 
\hline\hline 
Algorithm & Regret & Communication \\ [0.5ex] 
\hline 
DLUCB & $\Oc\big(d\frac{\log(Nd)}{{\log^{0.5}(1/|\la_2|)}}\log(NT)+d\log(NT)\sqrt{NT}\big)$ & $\Oc\big(dN^2T\frac{\log(Nd)}{{\log^{0.5}(1/|\la_2|)}}\big)$  \\ 
RC-DLUCB & $\Oc\big(Nd^{1.5}\frac{\log(Nd)}{{\log^{0.5}(1/|\la_2|)}}\log^{1.5}(NT)+d\log^2(NT)\sqrt{NT}\big)$ & $\Oc\big(d^3N^{2.5}\frac{\log(Nd)}{{\log^{0.5}(1/|\la_2|)}}\big)$ \\
No Communication & $\Oc\big(dN\log(T)\sqrt{T}\big)$ & 0\\
Centralized & $\Oc\big(d\log(NT)\sqrt{NT}\big)$ & $\Oc(dN^2T)$ \\
DCB  & $\Oc\big(\big(dN\log(NT)\big)^3+\log(NT)\sqrt{NT}\big)$ & $\Oc(d^2NT\log(NT))$ \\
DisLinUCB  & $\Oc(\log^2(NT)\sqrt{NT})$ & $\Oc(d^3N^{1.5})$ \\[1ex] 
\hline 
\end{tabular}
\caption{Comparison of DLUCB and  RC-DLUCB to baseline, as well as, to state-of-the-art. See Section~\ref{sec:tradeoff} for details.
} 
\label{table:comp} 
\end{table*}

\subsection{Regret-Communication trade-offs and comparison to state of the art}\label{sec:tradeoff}
In Table \ref{table:comp}, we compare (in terms of regret and communication)  DLUCB and RC-DLUCB to two baselines: (i) a `No Communication' and (ii) a fully `Centralized' algorithm, as well as, to the state of the art: (iii) DCB  \cite{korda2016distributed} and (iv) DisLinUCB \cite{wang2019distributed}.

{\bf Baselines.}~In the absence of communication, each agent independently implements a single-agent LUCB \cite{abbasi2011improved}. This trivial `No Communication' algorithm has zero communication cost and applies to any graph, but its regret scales linearly with the number of agents. At another extreme, a fully `Centralized' algorithm assumes communication is possible between any two agents at every round. This achieves optimal regret $\tilde O(\sqrt{NT})$, which is a lower bound to the regret of any decentralized algorithm. However, it is only applicable in very limited network topologies, such as a star graph where the central node acts as a master node, or, a complete graph. Notably, DLUCB achieves order-wise optimal regret that is same as that of the `Centralized' algorithm modulo a small additive regret-of-delay term. 

{\bf DisLinUCB.}~In a motivating recent paper \cite{wang2019distributed}, the authors presented `DisLinUCB' a communication algorithm that applies to multi-agent settings, in which agents can communicate with a master-node/server, by sending or receiving information to/from it with zero latency. Notably, DisLinUCB is shown to achieve order-optimal regret performance same as the 'Centralized' algorithm, but at a significantly lower communication cost that does \emph{not} scale with $T$ (see Table \ref{table:comp}). In this paper, we do \emph{not} assume presence of a master-node. In our setting, this can only be assumed in very limited cases: a star or a complete graph. Thus, compared to DisLinUCB, our DLUCB can be used for arbitrary network topologies with similar regret guarantees. However, DLUCB requires that communication be performed at each round. This discrepancy motivated us to introduce RC-LUCB, which has communication cost (slightly larger, but) comparable to that of DisLinUCB (see Table \ref{table:comp}), while being applicable to general graphs. As a final note, as in RC-DLUCB,  the reduced communication cost in DisLinUCB  relies on the idea of the Rarely Switching OFUL algorithm of \cite{abbasi2011improved}.


{\bf DCB.}~In another closely related work \cite{korda2016distributed} presented DCB for decentralized linear bandits in \emph{peer-to-peer networks}. Specifically, it is assumed in \cite{korda2016distributed} that at every round each agent communicates with only \emph{one} other \emph{randomly} chosen
 agent per round. Instead, we consider fixed network topologies where each agent can only communicate with its immediate neighbors at every round. Thus, the two algorithms are not directly comparable. Nevertheless, we remark that, similar to our setting, DCB also faces the challenge of controlling a delayed use of information, caused by requiring enough mixing of the communicated information among agents. 
A key difference is that the duration of delay is  typically $\Oc(\log t)$ in DCB, while in  DLUCB it is fixed to $S$, i.e., independent of the round $t$. This explains the significantly smaller first-term in the regret of DLUCB as compared to the first-term in the regret of DCB in Table \ref{table:comp}.

\section{Safe Decentralized Linear Bandits} \label{sec:Safe-DLUCB}



For the safe decentralized LB problem, we propose Safe-DLUCB, an extension of DLUCB to the safe setting and an extension of single-agent safe algorithms \cite{amani2019linear,moradipari2019safe,pacchiano2020stochastic} to multi-agent systems. We defer a detailed description of Safe-DLCUB to Algorithm \ref{alg:SDLUCB} in Appendix~\ref{proof:safe-DLUCBregret}. Here, we give a high-level description of its main steps and present its regret guarantees. First, we need the following assumption and notation.

\begin{myassum}[Non-empty safe set]\label{assum:nonempty}
A safe action $\x_0\in \Dc$ and $c_0 := \langle\boldsymbol \mu_\ast,\x_0\rangle< c$ are known to all agents. Also, $\langle\boldsymbol\theta_\ast,\x_0\rangle\geq 0.$
\end{myassum}

Define the normalized safe action $\tilde\x_0 : =\frac{\x_0}{\norm{\x_0}}$. For any $\x\in \mathbb{R}^{d}$, denote by $\x^o := \langle \x,\tilde\x_0 \rangle \tilde\x_0 $ its projection on $\x_0$, and, by $\x^{\bot} := \x - \x^o$ its projection onto the orthogonal subspace.

In the presence of safety, the agents must act conservatively to ensure that the chosen actions $\x_{i,t}$ do not violate the safety constraint $\langle\boldsymbol\mu_\ast,\x_{i,t}\rangle\leq c$. To this end, agent $i$ communicates, not only $\x_{i,t}$ and $y_{i,t}$, but also the bandit-feedback measurements $z_{i,t}$, following the communication protocol implemented by $\rm Comm$ (cf. Section \ref{sec:gossip}). Once information is sufficiently mixed, it builds an additional confidence set $\Ec_{i,t}$ that includes $\boldsymbol\mu_\ast^{\bot}$ with high probability (note that $\boldsymbol\mu_\ast^{o}$ is already known by Assumption \ref{assum:nonempty}). Please refer to Appendix~\ref{sec:mustarconfidence} for the details on constructing $\Ec_{i,t}$.
Once $\Ec_{i,t}$ is constructed, agent $i$ creates the following \emph{safe} inner approximation of the true $\Ds(\boldsymbol\mu_\ast)$:
$$
    \Dts := \{\x \in \Dc:\frac{\langle \x^{o},\xx_0\rangle}{\norm{\x_0}} c_0+ \langle\hat{\boldsymbol\mu}_{i,t}^{\bot},\x^{\bot}\rangle
    +\beta_t\norm{\x^{\bot}}_{\A_{i,t}^{\bot,-1}}\leq c\}.
$$
Specifically, Proposition \ref{prop:safe} in Appendix~\ref{sec:mustarconfidence} guarantees for any $\delta\in(0,1)$ that for all $i\in[N]$, $t\in[T]$, all actions in $\Dts$ are safe with probability $1-\delta$.
After constructing $\Dts$, agent $i$ selects \emph{safe} action $\x_{i,t}\in\Dts$ following a UCB decision rule:
\begin{align}\label{eq:SafeDLUCBdecision}
    \langle\tilde {\boldsymbol\theta}_{i,t},\x_{i,t}\rangle= \max_{\x\in \Dts}\max_{\boldsymbol \nu \in\kappa_r\Cc_{i,t}}\langle\boldsymbol \nu,\x\rangle. 
\end{align}
A subtle, but critical, point in \eqref{eq:SafeDLUCBdecision} is that the inner maximization is over an appropriately \emph{enlarged confidence set} $\kappa_r\Cc_{i,t}$. Specifically, compared to Lines \ref{DLUCB:line3} and \ref{line10} in Algorithm \ref{alg:DLUCB}, we need here that $\kappa_r>1$. Intuitively, this is required because the outer maximization in \eqref{eq:SafeDLUCBdecision} is not over the entire $\Ds(\boldsymbol\mu_\ast)$, but only a subset of it. Thus, larger values of $\kappa_r$ are needed to provide enough exploration to the algorithm so that the selected actions in $\Dts$ are -often enough- \emph{optimistic}, i.e., $\langle\tilde {\boldsymbol\theta}_{i,t},\x_{i,t}\rangle\geq \langle\boldsymbol\theta_\ast,\x_\ast\rangle$; see Lemma \ref{lemm:optimisimsafety} in Appendix~\ref{proof:optimisimsafety} for the exact statement. We attribute the above idea that more aggressive exploration of that form is needed in the safe setting to \cite{moradipari2019safe}, only they considered a Thompson-sampling scheme and a single agent. \cite{pacchiano2020stochastic} extended this idea to UCB algorithms, again in the single-agent setting (and for a slightly relaxed notion of safety). Here, we show that the idea extends to multi-agent systems and when incorporated to the framework of DLUCB leads to a safe decentralized algorithm with provable regret guarantees stated in the theorem below. See Appendix~\ref{proof:safe-DLUCBregret} for the proof.



\begin{theorem}[Regret of Safe-DLUCB]\label{thm:safe-DLUCBregret}
Fix $\delta \in (0,0.5)$,  $\kappa_r = \frac{2}{c-c_0}+1$, $\epsilon \in (0,1/(4d+1))$. Let Assumptions \ref{assum:comm_matrix}, \ref{assum:noise}, \ref{assum:boundedness}, \ref{assum:nonempty} hold, and $S$ be chosen as in \eqref{eq:cheb}. Then, with probability at least $1-2\delta$, it holds that:
    $$R_{T}\leq 2Sd\log(1+\frac{NT}{d\la})+2e\kappa_r\beta_{T}\sqrt{2dNT\log(\la+\frac{NT}{d})}.$$
\end{theorem}
The regret bound is of the same order as DLUCB regret bound, with only an additional factor $\kappa_r$ in its second term.

\section{Experiments}\label{sec:sim}
In this section, we evaluate our algorithms' performance on synthetic data. Since the UCB decision rule at line \ref{line10} of Algorithm \ref{alg:DLUCB} involves a generally non-convex optimization problem, we use a standard computationally tractable modification that replaces $\ell_2$ with $\ell_1$ norms in the definition of confidence set \eqref{eq:thetaconfidenceset} (unless the decision set is finite); see \cite{dani2008stochastic}. All results  directly apply to this modified algorithm after only changing the radius $\beta_t$ with $\beta_t\sqrt{d}$ \cite[Section~3.4]{dani2008stochastic}. All the results shown depict averages over 20 realizations, for which we have chosen $d = 5, \Dc = [-1,1]^{5}$, $\la=1$, and $\sigma = 0.1$. Moreover, $\boldsymbol\theta_\ast$ is drawn from $\mathcal{N}(0,I_5)$ and then normalized to unit norm. We compute the communication matrix as $\mathbf{P} = I-\frac{1}{\delta_{{\rm max}}+1}\D^{-1/2}\mathcal{L}\D^{-1/2}$, where $\delta_{{\rm max}}$ is the maximum degree of the graph and $\mathcal{L}$ is the graph Laplacian (see \cite{duchi2011dual} for details).
In Figures \ref{fig:dlucb} and \ref{fig:rc-dlucb}, fixing $N=20$, we evaluate the performance of DLUCB and RC-DLUCB on 4 different topologies: Ring, Star, Complete, and a Random Erdős–Rényi graph with parameter $p=0.5$; see Figure~\ref{fig:graphtopologies} in Appendix~\ref{app:experiments} for graphical illustrations of the graphs. We also compare them to the performance of No Communication (see Section \ref{sec:tradeoff}). The plot verifies the sublinear growth for all graphs, the superiority over the setting of No Communication and the fact that  smaller $|\la_2|$ leads to a smaller regret (regret of delay term in Thm. \ref{thm:DLUCBregret}). 
A comparison between Figures \ref{fig:dlucb} and \ref{fig:rc-dlucb}, further confirms the slightly better regret performance of DLUCB compared to RC-DLUCB (but the latter has superior communication cost). 
Figure~\ref{fig:numberofagent} emphasizes the value of collaboration in speeding up the learning process. It depicts the per-agent regret of DLUCB on random graphs with $N=5$, $10$ and $15$ nodes and compares their performance with the single-agent LUCB. Clearly, as the number of agents increases, each agent learns the environment faster as an individual.
\begin{figure*}[t!]
\centering
\begin{subfigure}{2in}
\begin{tikzpicture}
\node at (0,0) {\includegraphics[scale=0.2]{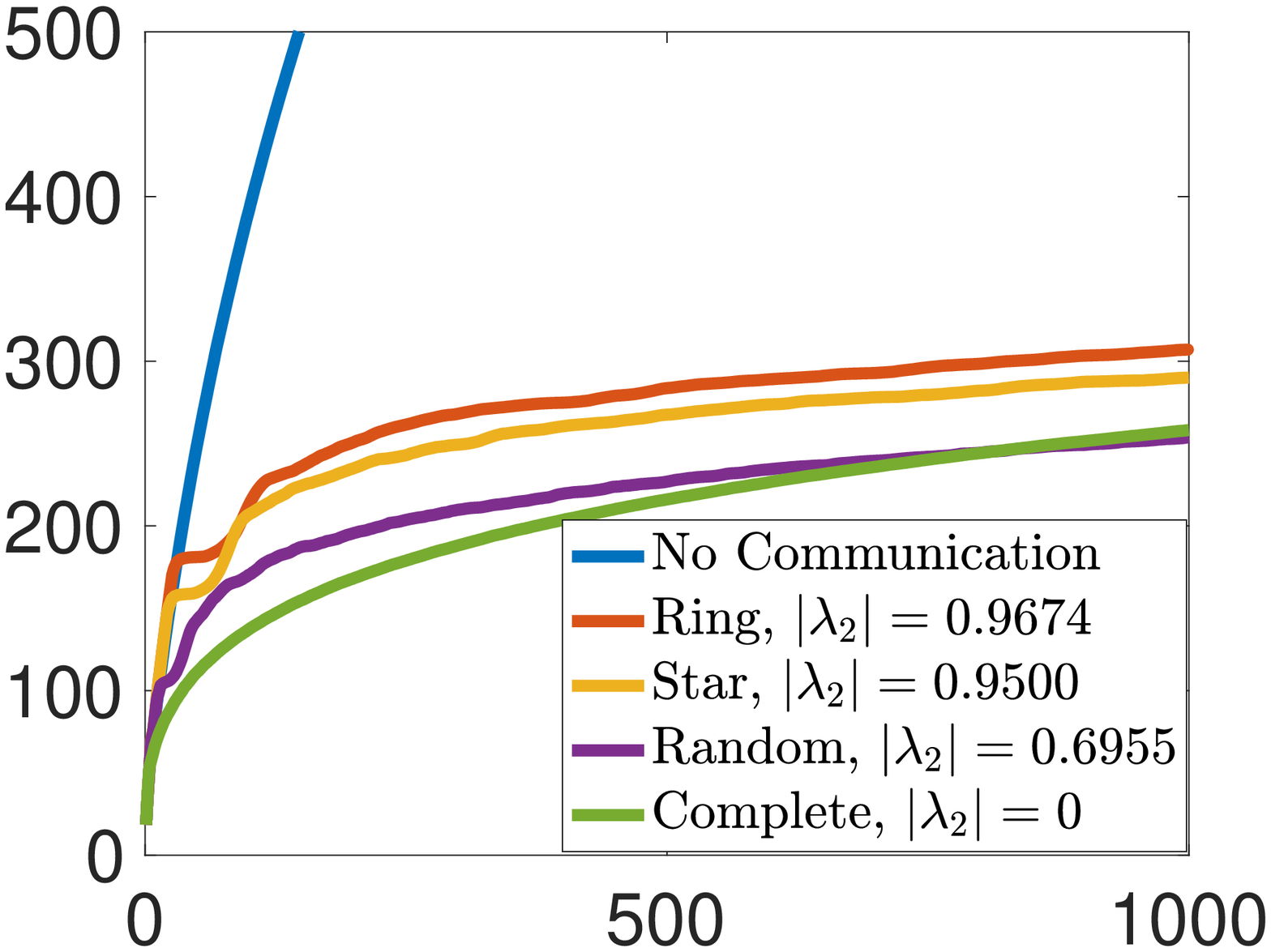}};
\node at (-2.3,0) [rotate=90,scale=0.9]{Regret, $R_t$};
\node at (0,-1.76) [scale=0.9]{Iteration, $t$};
\end{tikzpicture}
\caption{DLUCB}
\label{fig:dlucb}
\end{subfigure}
\centering
\begin{subfigure}{2in}
\begin{tikzpicture}
\node at (0,0) {\includegraphics[scale=0.2]{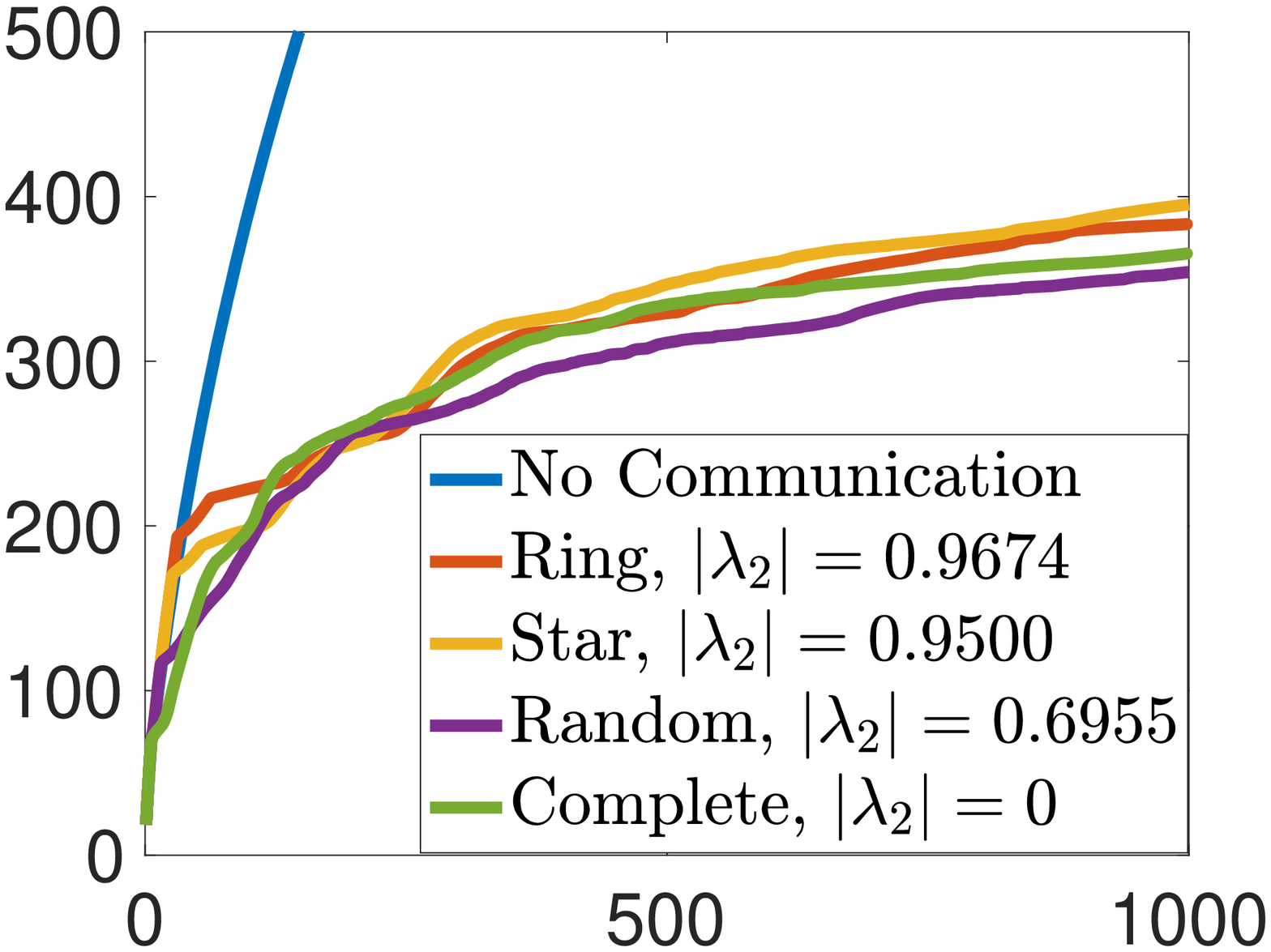}};
\node at (-2.3,0) [rotate=90,scale=1.]{Regret, $R_t$};
\node at (0,-1.76) [scale=0.9]{Iteration, $t$};
\end{tikzpicture}
\caption{RC-DLUCB}
\label{fig:rc-dlucb}
\end{subfigure}
\centering
\begin{subfigure}{2in}
\begin{tikzpicture}
\node at (0,0) {\includegraphics[scale=0.2]{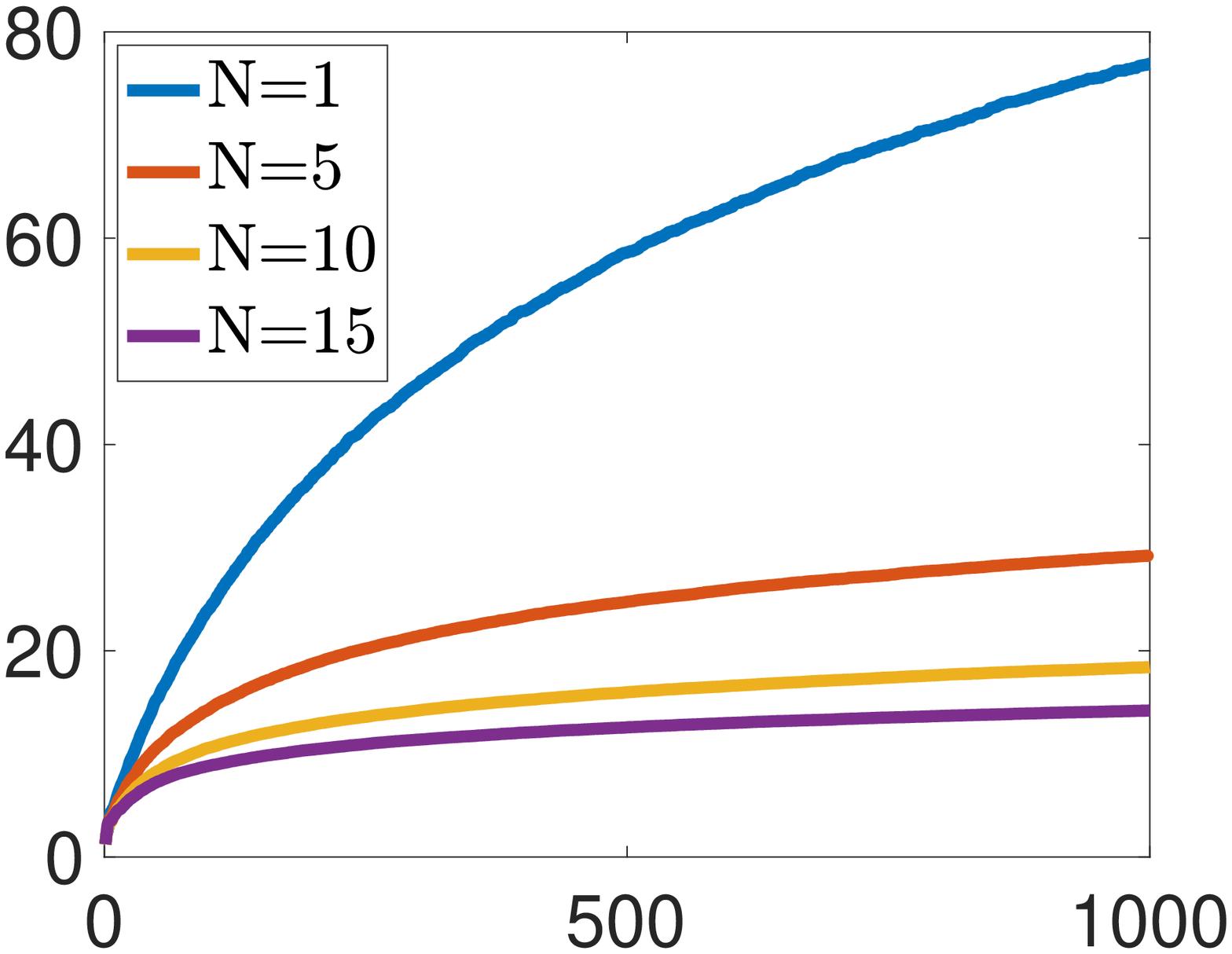}};
\node at (-2.3,0) [rotate=90,scale=0.9]{Per-agent regret, $\frac{R_t}{N}$};
\node at (0,-1.76) [scale=0.9]{Iteration, $t$};
\end{tikzpicture}
\caption{DLUCB}
\label{fig:numberofagent}
\end{subfigure}
\label{fig:MLM_class}
\end{figure*}

\section{Conclusion}
In this paper, we proposed two fully decentralized LB algorithms: 1) DLUCB  and 2) RC-DLUCB with small communication cost. We also proposed Safe-DLUCB to address 
the problem of safe LB in multi-agent settings. We derived near-optimal regret bounds for all the aforementioned algorithms that are applicable to arbitrary, but fixed networks. An interesting open problem is to design decentralized algorithms with
provable guarantees for settings with time-varying networks. Also, extensions to nonlinear settings and other types of safety-constraints are important future directions.

\section*{Acknowledgments}
This work is supported by the National Science Foundation under Grant Number (1934641).


\bibliographystyle{alpha}
\bibliography{reference}

\newcommand{\etalchar}[1]{$^{#1}$}
\begin{thebibliography}{KGYVR17}

\bibitem[AAT19]{amani2019linear}
Sanae Amani, Mahnoosh Alizadeh, and Christos Thrampoulidis.
\newblock Linear stochastic bandits under safety constraints.
\newblock In {\em Advances in Neural Information Processing Systems}, pages
  9252--9262, 2019.

\bibitem[AAT20a]{amani2020generalized}
Sanae Amani, Mahnoosh Alizadeh, and Christos Thrampoulidis.
\newblock Generalized linear bandits with safety constraints.
\newblock In {\em ICASSP 2020-2020 IEEE International Conference on Acoustics,
  Speech and Signal Processing (ICASSP)}, pages 3562--3566. IEEE, 2020.

\bibitem[AAT20b]{amani2020regret}
Sanae Amani, Mahnoosh Alizadeh, and Christos Thrampoulidis.
\newblock Regret bounds for safe gaussian process bandit optimization.
\newblock {\em arXiv preprint arXiv:2005.01936}, 2020.

\bibitem[AG13]{agrawal2013thompson}
Shipra Agrawal and Navin Goyal.
\newblock Thompson sampling for contextual bandits with linear payoffs.
\newblock In {\em International Conference on Machine Learning}, pages
  127--135, 2013.

\bibitem[AL{\etalchar{+}}17]{abeille2017linear}
Marc Abeille, Alessandro Lazaric, et~al.
\newblock Linear thompson sampling revisited.
\newblock {\em Electronic Journal of Statistics}, 11(2):5165--5197, 2017.

\bibitem[AM19]{avner2019multi}
Orly Avner and Shie Mannor.
\newblock Multi-user communication networks: A coordinated multi-armed bandit
  approach.
\newblock {\em IEEE/ACM Transactions on Networking}, 27(6):2192--2207, 2019.

\bibitem[AS14]{arioli2014chebyshev}
Mario Arioli and Jennifer Scott.
\newblock Chebyshev acceleration of iterative refinement.
\newblock {\em Numerical Algorithms}, 66(3):591--608, 2014.

\bibitem[AYPS11]{abbasi2011improved}
Yasin Abbasi-Yadkori, D{\'a}vid P{\'a}l, and Csaba Szepesv{\'a}ri.
\newblock Improved algorithms for linear stochastic bandits.
\newblock In {\em Advances in Neural Information Processing Systems}, pages
  2312--2320, 2011.

\bibitem[BKS16]{berkenkamp2016bayesian}
Felix Berkenkamp, Andreas Krause, and Angela~P Schoellig.
\newblock Bayesian optimization with safety constraints: safe and automatic
  parameter tuning in robotics.
\newblock {\em arXiv preprint arXiv:1602.04450}, 2016.

\bibitem[DAW11]{duchi2011dual}
John~C Duchi, Alekh Agarwal, and Martin~J Wainwright.
\newblock Dual averaging for distributed optimization: Convergence analysis and
  network scaling.
\newblock {\em IEEE Transactions on Automatic control}, 57(3):592--606, 2011.

\bibitem[DHK08]{dani2008stochastic}
Varsha Dani, Thomas~P Hayes, and Sham~M Kakade.
\newblock Stochastic linear optimization under bandit feedback.
\newblock 2008.

\bibitem[KB19]{khezeli2019safe}
Kia Khezeli and Eilyan Bitar.
\newblock Safe linear stochastic bandits.
\newblock {\em arXiv preprint arXiv:1911.09501}, 2019.

\bibitem[KGYVR17]{kazerouni2017conservative}
Abbas Kazerouni, Mohammad Ghavamzadeh, Yasin~Abbasi Yadkori, and Benjamin
  Van~Roy.
\newblock Conservative contextual linear bandits.
\newblock In {\em Advances in Neural Information Processing Systems}, pages
  3910--3919, 2017.

\bibitem[KSS16]{korda2016distributed}
Nathan Korda, Bal{\'a}zs Sz{\"o}r{\'e}nyi, and Li~Shuai.
\newblock Distributed clustering of linear bandits in peer to peer networks.
\newblock In {\em Journal of machine learning research workshop and conference
  proceedings}, volume~48, pages 1301--1309. International Machine Learning
  Societ, 2016.

\bibitem[LHLK13]{li2013medicine}
Shuai Li, Fei Hao, Mei Li, and Hee-Cheol Kim.
\newblock Medicine rating prediction and recommendation in mobile social
  networks.
\newblock In {\em International conference on grid and pervasive computing},
  pages 216--223. Springer, 2013.

\bibitem[LS18]{lattimore2018bandit}
Tor Lattimore and Csaba Szepesv{\'a}ri.
\newblock Bandit algorithms.
\newblock {\em preprint}, page~28, 2018.

\bibitem[LSL16a]{landgren2016distributed}
Peter Landgren, Vaibhav Srivastava, and Naomi~Ehrich Leonard.
\newblock Distributed cooperative decision-making in multiarmed bandits:
  Frequentist and bayesian algorithms.
\newblock In {\em 2016 IEEE 55th Conference on Decision and Control (CDC)},
  pages 167--172. IEEE, 2016.

\bibitem[LSL16b]{landgr}
Peter Landgren, Vaibhav Srivastava, and Naomi~Ehrich Leonard.
\newblock On distributed cooperative decision-making in multiarmed bandits.
\newblock In {\em 2016 European Control Conference (ECC)}, pages 243--248.
  IEEE, 2016.

\bibitem[Lyn96]{lynch1996distributed}
Nancy~A Lynch.
\newblock {\em Distributed algorithms}.
\newblock Elsevier, 1996.

\bibitem[MAAT19]{moradipari2019safe}
Ahmadreza Moradipari, Sanae Amani, Mahnoosh Alizadeh, and Christos
  Thrampoulidis.
\newblock Safe linear thompson sampling.
\newblock {\em arXiv preprint arXiv:1911.02156}, 2019.

\bibitem[MRKR19]{martinez2019decentralized}
David Mart{\'\i}nez-Rubio, Varun Kanade, and Patrick Rebeschini.
\newblock Decentralized cooperative stochastic bandits.
\newblock In {\em Advances in Neural Information Processing Systems}, pages
  4531--4542, 2019.

\bibitem[PGBJ20]{pacchiano2020stochastic}
Aldo Pacchiano, Mohammad Ghavamzadeh, Peter Bartlett, and Heinrich Jiang.
\newblock Stochastic bandits with linear constraints.
\newblock {\em arXiv preprint arXiv:2006.10185}, 2020.

\bibitem[RT10]{rusmevichientong2010linearly}
Paat Rusmevichientong and John~N Tsitsiklis.
\newblock Linearly parameterized bandits.
\newblock {\em Mathematics of Operations Research}, 35(2):395--411, 2010.

\bibitem[SBB{\etalchar{+}}17]{seaman2017optimal}
Kevin Seaman, Francis Bach, S{\'e}bastien Bubeck, Yin~Tat Lee, and Laurent
  Massouli{\'e}.
\newblock Optimal algorithms for smooth and strongly convex distributed
  optimization in networks.
\newblock In {\em Proceedings of the 34th International Conference on Machine
  Learning-Volume 70}, pages 3027--3036. JMLR. org, 2017.

\bibitem[SBFH{\etalchar{+}}13]{szorenyi2013gossip}
Bal{\'a}zs Sz{\"o}r{\'e}nyi, R{\'o}bert Busa-Fekete, Istv{\'a}n Heged{\H{u}}s,
  R{\'o}bert Orm{\'a}ndi, M{\'a}rk Jelasity, and Bal{\'a}zs K{\'e}gl.
\newblock Gossip-based distributed stochastic bandit algorithms.
\newblock In {\em Journal of Machine Learning Research Workshop and Conference
  Proceedings}, volume~2, pages 1056--1064. International Machine Learning
  Societ, 2013.

\bibitem[SGBK15]{sui2015safe}
Yanan Sui, Alkis Gotovos, Joel~W Burdick, and Andreas Krause.
\newblock Safe exploration for optimization with gaussian processes.
\newblock {\em Proceedings of Machine Learning Research}, 37:997--1005, 2015.

\bibitem[SZBY18]{sui2018stagewise}
Yanan Sui, Vincent Zhuang, Joel~W Burdick, and Yisong Yue.
\newblock Stagewise safe bayesian optimization with gaussian processes.
\newblock {\em arXiv preprint arXiv:1806.07555}, 2018.

\bibitem[WHCW19]{wang2019distributed}
Yuanhao Wang, Jiachen Hu, Xiaoyu Chen, and Liwei Wang.
\newblock Distributed bandit learning: How much communication is needed to
  achieve (near) optimal regret.
\newblock {\em arXiv preprint arXiv:1904.06309}, 2019.

\bibitem[XB04]{xiao2004fast}
Lin Xiao and Stephen Boyd.
\newblock Fast linear iterations for distributed averaging.
\newblock {\em Systems \& Control Letters}, 53(1):65--78, 2004.

\bibitem[You14]{young2014iterative}
David~M Young.
\newblock {\em Iterative solution of large linear systems}.
\newblock Elsevier, 2014.

\end{thebibliography}
\appendix

\section{Regret analysis of DLUCB}\label{proof:DLUCBregret1}
In this section, we provide complete proofs of  Theorem \ref{thm:confidencesettheta} and Lemmas \ref{lemm:first1} and \ref{lemm:second1}. We also show how this completes the regret bound analysis and leads to Theorem \ref{thm:DLUCBregret}.
\subsection{Proof of Theorem \ref{thm:confidencesettheta}} \label{proof:thetaconfidenceset0}
The proof is mostly adapted from \cite[Thm.~2]{abbasi2011improved}. We present the details for the reader's convenience.
First, we give the following lemma.
\begin{lemma}\label{lemm:martingle}
Let 
\begin{align}
    \s_{i,t}:=\begin{cases}
      \sum_{\tau=1}^{t-1}\eta_{i,\tau}\x_{i,\tau}&, \text{if}\quad t\in[S]\\
      \sum_{\tau=1}^{t-S}\sum_{j=1}^N  a_{i,j}^2\eta_{j,\tau}\x_{j,\tau}&, \text{if}\quad t\in[S+1:T].
    \end{cases}
\end{align}

Then, with probability at least $1-\delta$, for all $i \in [N]$ and $t\in[T]$:
\begin{equation}
    \norm{\s_{i,t}}_{\A_{i,t}^{-1}}^2\leq 2(1+\epsilon)^2\sigma^2\log\left(\frac{N\det(\A_{i,t})^{1/2}}{\delta\la^{d/2}}\right).
\end{equation}
\end{lemma}
\begin{proof}
Proof of Theorem 1 in \cite{abbasi2011improved} with slight changes can be applied here. The difference comes from the fact that the noise variance in this setting is bounded by:
\begin{equation}
   \max( {\rm Var}[\eta_{j,\tau}] , {\rm Var}[a_{i,j}\eta_{j,\tau}])\leq \sigma^2\max_{i,j}a_{i,j}^2 \leq (1+\epsilon)^2\sigma^2, \quad \forall i,j\in[N], \tau\in[S+1:T],
\end{equation}
where the last inequality follows from the choice of $S$ according to \eqref{eq:cheb} which guarantees $a_{i,j}\leq 1+\epsilon$ for all $i,j \in [N]$ (this can be seen by plugging $\boldsymbol \nu_0 = \mathbf{e}_i$ for all $i\in[N]$ in \eqref{eq:cheb}).
\end{proof}
We are now ready to complete the proof of Theorem \ref{thm:confidencesettheta}. For $t\in [S+1:T]$ and $i\in[N]$, we have:
\begin{align}
    \hat{\boldsymbol\theta}_{i,t}&=\A_{i,t}^{-1}\bb_{i,t} \nn\\
    &=\A_{i,t}^{-1}\left[-\la\boldsymbol\theta_\ast+\la\boldsymbol\theta_\ast+\sum_{\tau=1}^{t-S}\sum_{j=1}^N a_{i,j}^2\x_{j,\tau}( \x_{j,\tau}^T\boldsymbol\theta_\ast+ \eta_{j,\tau})\right]\nn \\
    &=\A_{i,t}^{-1}\left[\A_{i,t}\boldsymbol\theta_\ast-\la\boldsymbol\theta_\ast+\sum_{\tau=1}^{t-S}\sum_{j=1}^N a_{i,j}^2\eta_{j,\tau} \x_{j,\tau}\right]\nn\\
    &=\boldsymbol\theta_\ast-\la\A_{i,t}^{-1}\boldsymbol\theta_\ast+\A_{i,t}^{-1}\s_{i,t}.\nn
\end{align}
Hence, for any $\x \in \mathbb{R}^d$, we get:
\begin{align}
    |\langle \x,\hat{\boldsymbol\theta}_{i,t}-\boldsymbol\theta_\ast\rangle| &= |\langle\x,\A_{i,t}^{-1}\s_{i,t}\rangle-\la\langle\x,\A_{i,t}^{-1}\boldsymbol\theta_\ast\rangle|\nn\\
    &\leq \norm{\x}_{\A_{i,t}^{-1}}\left(\norm{\s_{i,t}}_{\A_{i,t}^{-1}}+\la^{1/2}\norm{\boldsymbol\theta_\ast}_2\right)\nn.
\end{align}
Lemma \ref{lemm:martingle} and plugging $\x = \A_{i,t}^{-1}(\hat{\boldsymbol\theta}_{i,t}-\boldsymbol\theta_\ast)$ imply that with probability at least $1-\delta$, for all $i \in [N]$ and $t\in[S+1:T]$:
\begin{align}
    \norm{\hat{\boldsymbol\theta}_{i,t}-\boldsymbol\theta_\ast}_{\A_{i,t}}&\leq (1+\epsilon)\sigma\sqrt{2\log\left(\frac{N\det(\A_{i,t})^{1/2}}{\delta\la^{d/2}}\right)}+\la^{1/2}\nn \\
    &\leq (1+\epsilon)\sigma\sqrt{dS\log\left(\frac{2\la d N+2N^2t}{\la d \delta}\right)}+\la^{1/2},\nn
\end{align}
where the last inequality follows from the fact that $\det(\A)=\prod_{i=1}^d \lambda_i(\A) \leq ({\rm trace}(\A)/d)^d$, ${\rm trace}(\A_{i,t})\leq (1+\epsilon)\rm{trace}(\A_{\ast,t-S})$ and $\epsilon \leq 1$.
We proved the theorem for $t\in [S+1:T]$. However, the result hold for rounds $t\in[S]$ by following the exact same argument.

\subsection{Proof of Theorem \ref{thm:DLUCBregret}}\label{proof:DLUCBregret}
Here we prove the regret bound of DLUCB.
\paragraph{Conditioning on $\boldsymbol\theta_\ast\in\Cc_{i,t},~\forall i\in[N],t\in[T]$.}  
Consider the event 
\begin{align}\label{eq:conditioned1}
\bar\Ec_1:=\{\boldsymbol\theta_\ast\in \Cc_{i,t},~\forall i\in[N],t\in[T]\},
\end{align}
that $\boldsymbol\theta_\ast$ is inside the confidence sets for all agents $i\in[N]$ and rounds $t\in[T]$. By Theorem \ref{thm:confidencesettheta} the event holds with probability at least $1-\delta$. Onwards, we condition on this event, and make repeated use of the fact that $\boldsymbol\theta_\ast\in\Cc_{i,t}$ for all $i\in[N],t\in[T]$, without further explicit reference. 
\paragraph{Decomposing the regret.} As is standard in regret analysis, we start with instantaneous regret decomposition. Let $\langle\tilde {\boldsymbol\theta}_{i,t},\x_{i,t}\rangle = \max_{\boldsymbol \nu\in\Cc_{i,t},\x\in \Dc}\langle\boldsymbol\nu,\x\rangle$. We decompose the instantaneous regret of agent $i$ at round $t\in[T]$ as follows:
\begin{align}
    r_{i,t} &=\langle\boldsymbol\theta_\ast,\x_\ast\rangle-  \langle\boldsymbol\theta_\ast,\x_{i,t}\rangle\nn \\
    &\leq \langle\tilde {\boldsymbol\theta}_{i,t},\x_{i,t}\rangle-  \langle\boldsymbol\theta_\ast,\x_{i,t}\rangle\nn\\
    &\leq \norm{\boldsymbol\theta_\ast-\tilde{\boldsymbol\theta}_{i,t}}_{\A_{i,t}}\norm{\x_{i,t}}_{\A_{i,t}^{-1}}\nn \\
    &\leq 2\min(\beta_t \norm{\x_{i,t}}_{\A_{i,t}^{-1}},1),\label{eq:regdec}
\end{align}
where the last inequality follows by conditioning on $\bar\Ec_1$ and Assumption \ref{assum:boundedness}.

If information sharing was perfect at each round and also agents ran DLUCB in an asynchronous manner with no delay in using information, the current gram matrix of agent $i$ round $t$ could be computed using all the information gathered
by all the agents in rounds $1,\ldots,t-1$, and also that gathered by agents $j=1,\ldots,i-1$ at round $t$:
\begin{equation}\label{eq:Bik}
    \B_{i,t} = \A_{\ast,t}+\sum_{j=1}^{i-1} \x_{j,t} \x_{j,t}^T.
\end{equation}
In what follows, we prove Lemmas \ref{lemm:first1} and \ref{lemm:second1} to establish a connection between norms in regret decomposition \eqref{eq:regdec} and those with respect to $\B_{i,t}$.

{\bf Proof of Lemma \ref{lemm:first1}.}

Since $\A_{\ast,t-S+1} ={\la}I \preceq \A_{i,t}$ for all $t\in[S]$ and $i\in[N]$, we observe that $\norm{\x_{i,t}}_{\A_{i,t}^{-1}}^2 \leq \norm{\x_{i,t}}_{{\A^{-1}_{\ast,t-S+1}}}^2 \leq e \norm{\x_{i,t}}_{{\A^{-1}_{\ast,t-S+1}}}^2$ for $t\in[S]$ and $i\in[N]$.

Thus, we focus on $t\in[S+1:T]$. Due to the choice of $S$ in \eqref{eq:cheb}, we know $|a_{i,j}-1|\leq \epsilon$ for all $i,j\in[N]$. Hence, for all $i\in[N]$ and $t\in[S+1:T]$, we get:
    \begin{equation}\label{eq:twoineq}
      (1-\epsilon)^2\A_{\ast,t-S+1} \leq \A_{i,t}\leq (1+\epsilon)^2\A_{\ast,t-S+1},
    \end{equation}
    where the inequalities hold element-wise. 
    Furthermore, for any positive semi-definite matrices $\A$, $\B$, and $\C$ such that $\A=\B+\C$, we have:
    \begin{align}\label{eq:dett}
    \det(\A)\geq \det(\B),~ \det(\A)\geq \det(\C),
    \end{align}
    and for any $\x \neq 0$ (\cite[Lemm.~12]{abbasi2011improved}):
    \begin{align}\label{eq:det}
        \frac{\norm{\x}^2_{\A}}{\norm{\x}^2_{\B}}\leq \frac{\det(\A)}{\det(\B)} \quad \text{and}\quad \frac{\norm{\x}^2_{\B^{-1}}}{\norm{\x}^2_{\A^{-1}}}\leq \frac{\det(\A)}{\det(\B)}.
    \end{align}

Combining \eqref{eq:twoineq}, \eqref{eq:dett} and \eqref{eq:det} yield that:
    \begin{align}
        \norm{\x_{i,t}}_{\A_{i,t}^{-1}}^2 &\stackrel{\eqref{eq:det}}{\leq} \norm{\x_{i,t}}_{[(1+\epsilon)^2 \A_{\ast,t-S+1}]^{-1}}^2\frac{\det [(1+\epsilon)^2 \A_{\ast,t-S+1}]}{\det \A_{i,t}}\nn \\
        &\stackrel{\eqref{eq:dett}}{\leq} \frac{1}{(1+\epsilon)^2}\norm{\x_{i,t}}_{ \A_{\ast,t-S+1}^{-1}}^2\frac{\det [(1+\epsilon)^2\A_{\ast,t-S+1}]}{\det [(1-\epsilon)^2\A_{\ast,t-S+1}]}\nn \\
        &\leq \left(\frac{1+\epsilon}{ 1-\epsilon}\right)^{2d} \norm{\x_{i,t}}_{\A_{\ast,t-S+1}^{-1}}^2 \label{eq:general} \\
        &= \left(1+\frac{2\epsilon}{1-\epsilon}\right)^{2d}\norm{\x_{i,t}}_{\A_{\ast,t-S+1}^{-1}}^2 \nn \\
        &\leq  e^{\frac{4\epsilon d}{1-\epsilon}}\norm{\x_{i,t}}_{\A_{\ast,t-S+1}^{-1}}^2\label{eq:namosavi1}\\
        &\leq  e\norm{\x_{i,t}}_{\A_{\ast,t-S+1}^{-1}}^2,\label{eq:namsovai2}
    \end{align}
   as desired. Inequality \eqref{eq:general} follows from the fact that $\det(\alpha \A) = \alpha^d \det(\A)$ for any $\alpha \in \mathbb{R}$. In inequalities \eqref{eq:namosavi1} and \eqref{eq:namsovai2}, we use  $1+a\leq e^a$ for all real $a$ and $\epsilon \leq 1/(4d+1)$, respectively. Also note that inequality \eqref{eq:general} gives a general upper bound when there is no assumption on how small $\epsilon$ is.

{\bf Proof of Lemma \ref{lemm:second1}.}

By the definition of $\B_{i,t}$ in \eqref{eq:Bik}, we have:
\begin{align}\label{eq:firstmain}
  \norm{\x_{i,t}}_{ \A_{\ast,t-S+1}^{-1}}^2\stackrel{\eqref{eq:det}}{\leq}
     \norm{\x_{i,t}}_{\B_{i,t}^{-1}}^2\frac{\det \B_{i,t}}{\det \A_{\ast,t-S+1}}\stackrel{\eqref{eq:dett}}{\leq}  \norm{\x_{i,t}}_{\B_{i,t}^{-1}}^2\frac{\det \A_{\ast,t+1}}{\det \A_{\ast,t-S+1}}.
\end{align}
Note that for any $ t\leq 1$, $\det\A_{\ast,t} ={\la}^d$ and $\det\A_{\ast,T}\leq ({\rm{trace}}(\A_{\ast,T})/d)^d \leq \left(\la+\frac{NT}{d}\right)^d$,
and consequently:
\begin{align}\label{eq:avaliii}
    \frac{\det\A_{\ast,T}}{\det\A_{\ast,2-S}}=\frac{\det\A_{\ast,T}}{\det\A_{\ast,3-S}}=\ldots=\frac{\det\A_{\ast,T}}{\det\A_{\ast,0}}=\frac{\det\A_{\ast,T}}{\det\A_{\ast,1}} \leq \left(1+\frac{NT}{d\lambda}\right)^d.
\end{align}
Without loss of generality, we assume $\A_{\ast,T^\prime} = \A_{\ast,T}$ for any $T^\prime\geq T$. Let $m_T = \lceil T/S \rceil$. Then for any $-1\leq n\leq S-2$,
\begin{align}\label{eq:dovomiii}
    \frac{\det\A_{\ast,T}}{\det\A_{\ast,-n}} &= \prod_{k=1}^{m_T+1} \frac{\det\A_{\ast,kS-n}}{\det\A_{\ast,(k-1)S-n}}.
\end{align}
Since $1\leq \frac{\det\A_{\ast,kS-n}}{\det\A_{\ast,(k-1)S-n}}$ for all $k\in[m_T+1]$ and $-1\leq n\leq S-2$, we can deduce from \eqref{eq:avaliii} and \eqref{eq:dovomiii} that    
\begin{align}
    e \leq \frac{\det \A_{\ast,t+1}}{\det \A_{\ast,t-S+1}}
\end{align}
for at most $Sd\log\left(1+\frac{NT}{d\la}\right)$ number of rounds $t\in[T]$. This result coupled with \eqref{eq:firstmain} implies that:
\begin{align}\label{eq:dovomiiii}
     \norm{\x_{i,t}}_{\A_{\ast,t-S+1}^{-1}}^2\leq 
     e \norm{\x_{i,t}}_{\B_{i,t}^{-1}}^2
\end{align}
is true for all but at most $\psi(\lambda,|\la_2|,\epsilon, d, N,T) = Sd\log\left(1+\frac{NT}{d\la}\right)$ pairs of $(i,t)\in [N]\times [T]$.

{\bf Completing the proof of Theorem \ref{thm:DLUCBregret}.}
We are now ready to complete the proof of Theorem \ref{thm:DLUCBregret}. We call the pairs $(i,t)$ that satisfy:
\begin{align}
    \norm{\x_{i,t}}_{\A_{\ast,t-S+1}^{-1}}^2\leq 
     e \norm{\x_{i,t}}_{\B_{i,t}^{-1}}^2\label{eq:goodpairs}
\end{align}
``good pairs'' and the rest ``bad pairs'', and define the set of good pairs by $\Bgood:=\left\{(i,t)\in[N]\times[T], \norm{\x_{i,t}}_{\A_{\ast,t-S+1}^{-1}}^2\leq 
     e \norm{\x_{i,t}}_{\B_{i,t}^{-1}}^2 \right\}$.

We first decompose the cumulative regret $R_T$ into components:
\begin{align}
    R_T({\rm{DLUCB}}) = R_{T,\rm{bad}}({\rm{DLUCB}}) +R_{T,\rm{good}}({\rm{DLUCB}}). \label{eq:decompose}
\end{align}
The first term is the regret of bad pairs that is also bounded by its maximum possible number of them multiplied by 2, which is $2\psi(\lambda,|\la_2|,\epsilon, d, N,T)$. The second term is the regret of good pairs. We bound $ R_{T,\rm{good}}({\rm{DLUCB}})$ as follows:
\begin{align}
    R_{T,\rm{good}}({\rm{DLUCB}}) = \sum_{(i,t)\in\Bgood}r_{i,t} &\leq 2\beta_{T}\sum_{(i,t)\in \Bgood}\min(\beta_t \norm{\x_{i,t}}_{\A_{i,t}^{-1}},1) \nn \\
    &\leq 2\beta_{T}\sqrt{NT\sum_{(i,t)\in \Bgood}\min\left(\norm{\x_{i,t}}^2_{\A_{i,t}^{-1}},1\right)}\nn\\
    &\stackrel{{\rm Lemma}~\ref{lemm:first1}}{\leq} 2\beta_{T}\sqrt{eNT\sum_{(i,t)\in \Bgood}\min\left(\norm{\x_{i,t}}^2_{\A_{\ast,t-S+1}^{-1}},1\right)}\nn\\
    &\stackrel{{\rm Lemma}~\ref{lemm:second1}}{\leq} 2e\beta_{T}\sqrt{NT\sum_{t=S}^{T}\sum_{i=1}^N\min\left(\norm{\x_{i,t}}^2_{\B_{i,t}^{-1}},1\right)}\nn\\
    &\leq 2e\beta_{T}\sqrt{2dNT\log\left(\la+\frac{NT}{d}\right)}.
\end{align}

In the last inequality, we used the standard argument in regret analysis of linear bandit algorithm stated in the following \cite[Lemma.~11]{abbasi2011improved}:
\begin{align}\label{eq:standardarg}
    \sum_{t=1}^n \min\left(\norm{\y_t}^2_{\Vb_{t}^{-1}},1\right)\leq 2\log\frac{\det \Vb_{n+1}}{\det\Vb_{1}}\quad \text{where}\quad \Vb_n = \Vb_{1}+\sum_{t=1}^{n-1}\y_t\y^T_t.
\end{align}
Putting things together, we conclude that, with probability at least $1-\delta$:
\begin{align}
    R_T({\rm{DLUCB}}) \leq 2Sd\log\left(1+\frac{NT}{d\la}\right)+2e\beta_{T}\sqrt{2dNT\log\left(\la+\frac{NT}{d}\right)},\label{eq:final}
\end{align}
as desired.
Note that \eqref{eq:final} holds for $\epsilon \in (0,1/(4d+1))$. However, for an arbitrary $\epsilon \in (0,1)$, we get a general regret bound following the inequality \eqref{eq:general} and same argument above as follows:
\begin{align}
    R_T({\rm{DLUCB}}) \leq 2Sd\log\left(1+\frac{NT}{dS}\right)+2\beta_{T}\left(\frac{1+\epsilon}{1-\epsilon}\right)^d\sqrt{2edNT\log\left(\la+\frac{NT}{d}\right)}.
\end{align}

\subsection{Detailed version of DLUCB}\label{app:detailedversion}
In this section, we present a detailed pseudo code for DLUCB where we explicitly clarify how the decision rule and the communication steps are conducted.
\begin{algorithm}[ht]
\DontPrintSemicolon
  \KwInput{$\Dc$, $N$, $d$, $|\la_2|$, $\epsilon$, $\la$, $\delta$, $T$}
$S = \log(2N/\epsilon)/\sqrt{2\log(1/|\la_2|)}$\;
$\A_{i,1} = \la I$, $\bb_{i,1}=\mathbf{0}$, $\Ac_{i,0} =  \Ac_{i,1}  =\Bc_{i,0} = \Bc_{i,1}=\emptyset$\;
\For{$t=1,\ldots,S$}
{
Construct $\Cc_{i,t}:=\{\boldsymbol \nu \in \mathbb{R}^d: \|{\boldsymbol \nu-\hat {\boldsymbol\theta}_{i,t}}\|_{\A_{i,t}} \leq \beta_t\}$
where $\hat{\boldsymbol\theta}_{i,t} = \A_{i,t}^{-1}\bb_{i,t}$ and $\beta_t$ is chosen as in Thm.~\ref{thm:confidencesettheta}.\;
Play $\x_{i,t} = \argmax_{\x\in \Dc}\max_{\boldsymbol \nu \in \Cc_{i,t}}\langle\boldsymbol \nu,\x\rangle$ and observe $y_{i,t}$.\;
$\Ac_{i,t}.{\rm append}(\X_{i,t})$ and $\Bc_{i,t}.{\rm append}(\y_{i,t})$\;
\For{$j=1:\ldots,t$}
{$[\Ac_{i,t+1}(j)]_{m,n} = \U([\Ac_{i,t}(j)]_{m,n},[\Ac_{i,t-1}(j+1)]_{m,n},S-j)$, $\forall m\in[N], n\in[d]$ \;
$[\Bc_{i,t+1}(j)]_m = \U([\Bc_{i,t}(j)]_m,[\Bc_{i,t-1}(j+1)]_m,S-j)$, $ \forall m\in[N]$
}
$\A_{i,t+1} = \A_{i,t} + \x_{i,t}\x^T_{i,t}$, $\bb_{i,t+1} = \bb_{i,t}+ y_{i,t}\x_{i,t}$\;
}
$\A_{i,S}=\la I$, $\bb_{i,S}=\mathbf{0}$ \label{linee11}\;
\For{$t=S+1,\ldots,T$}
{
$\A_{i,t} = \A_{i,t-1} + N^2\Ac_{i,t}(1)^T\Ac_{i,t}(1)$, $\bb_{i,t} = \bb_{i,t-1} + N^2 \Ac_{i,t}(1)^T\Bc_{i,t}(1)$\;
Construct $\Cc_{i,t}:=\{\boldsymbol \nu \in \mathbb{R}^d: \|{\boldsymbol \nu-\hat {\boldsymbol\theta}_{i,t}}\|_{\A_{i,t}} \leq \beta_t\}$
where $\hat{\boldsymbol\theta}_{i,t} = \A_{i,t}^{-1}\bb_{i,t}$ and $\beta_t$ is chosen as in Thm.~\ref{thm:confidencesettheta}.\;
Play $\x_{i,t} = \argmax_{\x\in \Dc}\max_{\boldsymbol \nu \in \Cc_{i,t}}\langle\boldsymbol \nu,\x\rangle$ and observe $y_{i,t}$\;
 $\Ac_{i,t}.{\rm remove}\Big(\Ac_{i,t}(1)\Big).{\rm append}\left(\X_{i,t}\right)$\;
$\Bc_{i,t}.{\rm remove}\left(\Bc_{i,t}(1)\right).{\rm append}\left(\y_{i,t}\right)$\;
\For{$j=1:\ldots,S$}
{$[\Ac_{i,t+1}(j)]_{m,n} = \U([\Ac_{i,t}(j)]_{m,n},[\Ac_{i,t-1}(j+1)]_{m,n},S-j)$, $\forall m\in[N], n\in[d]$ \;
$[\Bc_{i,t+1}(j)]_m = \U([\Bc_{i,t}(j)]_m,[\Bc_{i,t-1}(j+1)]_m,S-j)$, for all $m\in[N]$}

}
\caption{Detailed DLUCB for Agent $i$}
\label{alg:completeDLUCB}
\end{algorithm}

\begin{remark}[Alternative communication scheme]


In Section \ref{sec:gossip}, we described how each agent communicates their information to the rest of the network with the goal of consensus. Here, we remark on the following alternative scheme. It is possible that each agent  communicates the estimates of the following two quantities at rounds $t+1,\ldots,t+S$: (i) $\sum_{j=1}^N\x_{j,t}\x_{j,t}^T$ and (ii) $\sum_{j=1}^N y_{j,t}\x_{j,t}$. After obtaining accurate estimates of the  aforementioned quantities at the end of round $t+S$, each agent adds its estimate to its previous mixed information (its estimates of $\sum_{j=1}^N\x_{j,\tau}\x_{j,\tau}^T$ and $\sum_{j=1}^N y_{j,\tau}\x_{j,\tau}$ for all $\tau<t$) and construct its confidence set. On the one hand, following this protocol, each agent sends $d(d+1)$ values per round to each of its neighbors. In comparison, the information sharing protocol described in Section~\ref{sec:gossip} requires that each agent sends $N(d+1)$ values per round. On the other hand, the alternative information sharing protocol described here, requires memory of size $Sd(d+1)$ for each agent to keep their estimates at each round. This is in contrast to the storage requirement of size $SN(d+1)$ for the information sharing protocol of Section~\ref{sec:gossip}. Either of the two protocols lead to the same regret guarantees for DLUCB. The alternative presented here is to be preferred in terms of communication overhead (number of values that are communicated per round) and storage (number of values needed to be maintained by every agent at each round) when $N\gg d$.
\end{remark}

\begin{remark}[Decentralized Linear TS]
Thompson Sampling (TS) is another strategy to tackle the decentralized LB problem. A Decentralized Linear TS (DLTS) algorithm follows the same communication procedure as DLUCB. However, the decision rule is different, such that lines \ref{line10} of DLUCB changes to:
 \begin{align}
   \x_{i,t}= \argmax_{\x\in \Dc}\langle\tilde {\boldsymbol\theta}_{i,t},\x\rangle,
   \label{eq:DLTSdecision}
  \end{align}
where $\Tilde{\boldsymbol\theta}_{i,t} = {\hat{\boldsymbol\theta}_{i,t}} + \beta_t \A_{i,t}^{-\frac{1}{2}} \rho_{i,t}$ with $\rho_{i,t} \sim \mathcal{N}(\mathbf{0},I)$. The multivariate Gaussian distribution can be generalized to other appropriate distributions; see \cite{abeille2017linear}. Following similar arguments in the analysis of the single agent TS algorithm \cite{agrawal2013thompson,abeille2017linear} combined with our techniques in the  previous sections of this paper, we can easily show that the regret of the DLTS is $\Oc\big(d\frac{\log(Nd)}{{\log^{0.5}(1/|\la_2|)}}\log(NT)+d\log(NT)\sqrt{NT}\big)$.
\end{remark}

\section{Communication step}\label{sec:comm}
In this section, we summarize the accelerated Chebyshev communication step discussed in Section~\ref{sec:gossip} in Algorithm \ref{alg:comm}, which follows the same steps as those of the communication algorithm presented in \cite{martinez2019decentralized}.

Chebyshev polynomials \cite{young2014iterative} are defined as $T_0(x)=1, T_1(x)=x$ and $T_{k+1}(x)=2xT_k(x)-T_{k-1}(x)$. Define:
\begin{align}
    q_\ell(\mathbf{P}) = \frac{T_{\ell}(\mathbf{P}/|\la_2|)}{T_\ell(1/|\la_2|)}.
\end{align}

By the properties of Chebyshev polynomial \cite{arioli2014chebyshev}, it can be shown that:
\begin{align}
    q_{\ell+1}(\mathbf{P})= \frac{2w_\ell}{|\la_2|w_{\ell+1}}\mathbf{P}q_{\ell}(\mathbf{P})-\frac{w_{\ell-1}}{w_{\ell+1}}q_{\ell-1}(\mathbf{P}),
\end{align}
where $w_0 = 1, w_1 = 1/|\la_2|$, $w_{\ell+1}=2w_\ell/|\la_2|-w_{\ell-1}$, $q_0(\mathbf{P})=I$ and $q_1(\mathbf{P})=\mathbf{P}$. This implies that, when agents share an specific quantity, whose initial values given by agents are denoted by vector $\boldsymbol \nu_0\in \mathbb{R}^N$, by using the recursive Chebyshev-accelerated updating rule, they have:
\begin{align}
   \boldsymbol \nu_{\ell+1} = \frac{2w_\ell}{|\la_2|w_{\ell+1}}\mathbf{P}\boldsymbol \nu_{\ell}-\frac{w_{\ell-1}}{w_{\ell+1}}\boldsymbol \nu_{\ell-1}.
\end{align}
In light of the above mentioned recursive procedure, the accelerated communication step is summarized in Algorithm \ref{alg:comm} below for agent $i$. We denote the inputs by: 1) $x_{\rm now}$, which is the quantity of interest that agent $i$ wants to update at the current round, 2) $x_{\rm prev}$, which is the estimated value for a quantity of interest that agent $i$ updated at the previous round, and 3) $\ell$ which is the current round of communication. Note that inputs are scalars, however matrices and vectors also can be passed as inputs with Comm running for each of their entries. 
\begin{algorithm}[ht]
\DontPrintSemicolon
  \KwInput{$x_{\rm now}$, $x_{\rm prev}$, $\ell$}
$w_0 = 0, w_1 = 1/|\la_2|, w_{r} = 2w_{r-1}/|\la_2|-w_{r-2}~\forall 2\leq r \leq S$\;
$x_{i,{\rm now}} = x_{\rm now}$, $x_{i,{\rm prev}} = x_{\rm prev}$\;
Send $x_{i,{\rm now}}$ and receive the corresponding $x_{j,{\rm now}}$ to and from $j\in\mathcal{N}(i)$ \tcp*{Recall that all agents run Comm in parallel}
\If{$\ell=1$}
{$x_{i,{\rm next}} = \mathbf{P}_{i,i}x_{i,{\rm now}}+\sum_{j\in \Nc(i)}\mathbf{P}_{i,j}x_{j,{\rm now}}$

}
\Else
{
$x_{i,{\rm next}} = \frac{2w_{\ell-1}}{|\la_2|w_{\ell}}\mathbf{P}_{i,i}x_{i,{\rm now}}+ \frac{2w_{\ell-1}}{|\la_2|w_{\ell}}\sum_{j\in \Nc(i)}\mathbf{P}_{i,j}x_{j,{\rm now}}-\frac{w_{\ell-2}}{w_{\ell}}x_{i,{\rm prev}}$}

Return $x_{i,{\rm next}}$
  \caption{Comm for Agent $i$}
\label{alg:comm}
\end{algorithm}


\section{DLUCB with Rare Communication}\label{sec:variants}

In this section, we describe RC-DLUCB summarized in Algorithm \ref{alg:DLUCB2}. Technical details provided in this section are mostly adapted from the tools we used in the previous sections for the analysis of DLUCB. 

In RC-DLUCB, at round $t$, agent $i$ uses all its available information to first compute $\A_{i,t}$ and $\bb_{i,t}$ with new definitions which we will elaborate on later in this section, and then it creates confidence set $\Cc_{i,t}$ as in \eqref{eq:thetaconfidenceset}. The volume of the confidence ellipsoid $\Cc_{i,t}$ depends on $\det \A_{i,t}$. If for any agent $i$, the corresponding $\det \A_{i,t}$ varies greatly, it needs extra information to shrink its confidence set. As such, agents keep running UCB decision rule until at some round $t$, one of the agents, say agent $i$, faces a relatively large increase in $\det \A_{i,t}$. Under such a condition (Line \ref{line:cond} of Algorithm \ref{alg:DLUCB2}), a communication phase begins. When a communication phase begins, all the agents share their unmixed information with their neighbors for $S$ rounds. During these $S$ rounds, all agents keep playing the action that was selected based on the last time they updated their confidence sets. We define ${\rm CP}_k$ to be the event that the $k$-th communication phase has started.

We denote the round at which the $k$-th communication phase ends by $t_k$. The set of rounds between $t_{k-1}$ and $t_{k}$ are referred to as epoch $k$ and is denoted by the ordered set ${\rm EP}_k:=[t_{k-1}+1:t_k]$. In particular, at rounds $t_{k-1}+1,\ldots,t_k-S$, agents run UCB decision rule independently with no communication and at the last $S$ rounds of each epoch ${\rm EP}_k$ (the $k$-th communication phase), they communicate their unmixed information and play the actions that were selected based on the last time they updated their confidence sets, $\x_{i,t_k-S}$.

Note that, the quantities that agents communicate with each other during the communication phase are matrices and vectors (rather than vectors and scalars in DLUCB). To make this concrete, suppose that no communication has occurred between rounds $s$ and $t$. The goal of a communication phase starting at round $t+1$ is to provide each agent with accurate estimates of: 1) $\sum_{\tau=s}^t\sum_{j=1}^N\x_{j,\tau}\x_{j,\tau}^T$ and 2) $\sum_{\tau=s}^t \sum_{j=1}^N y_{j,\tau}\x_{j,\tau}$, at round $t+S$.
To this end, at round $t+1$, each agent $i$ sends $\sum_{\tau=s}^t\x_{i,\tau}\x_{i,\tau}^T$ and  $\sum_{\tau=s}^t y_{i,\tau}\x_{i,\tau}$ to its neighbors and receives the corresponding quantities from them. The agents keep updating their estimates of $\sum_{\tau=s}^t\sum_{j=1}^N\x_{j,\tau}\x_{j,\tau}^T$ and $\sum_{\tau=s}^t \sum_{j=1}^N y_{j,\tau}\x_{j,\tau}$ by running the accelerated consensus (Algorithm \ref{alg:comm}) for $S$ rounds.

Now, we define the sufficient statistics available for agent $i$ at round $t$ to construct $\Cc_{i,t}$ as follows:
\begin{align}\label{eq:rcdlucbsuff}
    \A_{i,t} = \begin{cases}
      \la I+\sum_{\tau=1}^{t_{k-1}}\sum_{j=1}^N a_{i,j}\x_{j,\tau}\x_{j,\tau}^T+\sum_{\tau=t_{k-1}+1}^{t-1}\x_{i,\tau}\x_{i,\tau}^T &, \text{if}\quad t\in[t_{k-1}+1:t_{k}-S],\\
      \la I+\sum_{\tau=1}^{t_{k-1}}\sum_{j=1}^N a_{i,j}\x_{j,\tau}\x_{j,\tau}^T+\sum_{\tau=t_{k-1}+1}^{t_{k}-S-1}\x_{i,\tau}\x_{i,\tau}^T &, \text{if}\quad t\in[t_{k}-S+1:t_{k}],
    \end{cases}
\end{align}

\begin{align}
    \bb_{i,t} = \begin{cases}
    \sum_{\tau=1}^{t_{k-1}}\sum_{j=1}^N a_{i,j} y_{j,\tau}\x_{j,\tau}+\sum_{\tau=t_{k-1}+1}^{t-1}y_{i,\tau}\x_{i,\tau} &, \text{if}\quad t\in[t_{k-1}+1:t_{k}-S],\\
    \sum_{\tau=1}^{t_{k-1}}\sum_{j=1}^N a_{i,j}  y_{j,\tau}\x_{j,\tau}+\sum_{\tau=t_k+1}^{t_{k}-S-1}y_{i,\tau}\x_{i,\tau} &, \text{if}\quad t\in[t_{k}-S+1:t_{k}].
    \end{cases}
\end{align}

\begin{algorithm}[t]
\DontPrintSemicolon
  \KwInput{$\Dc$, $N$, $d$, $|\la_2|$, $\epsilon$, $\la$, $\delta$, $T$}
$S = \log(2N/\epsilon)/\sqrt{2\log(1/|\la_2|)}$, $M=\frac{T\log(1+NT/d\la)}{dN}$\;
$\A_{i,1} = \la I$, $\bb_{i,1}=\mathbf{0}$, $\W_{i,{\rm new}}=\W_{i,{\rm syn}}=\mathbf{0}$, $\vb_{i,{\rm new}}=\vb_{i,{\rm syn}}=\mathbf{0}$, $t=1$, $k=1$, $t_0=0$\;

\While{$t\leq T$}
{
Construct $\Cc_{i,t}:=\{\boldsymbol \nu \in \mathbb{R}^d: \norm{\boldsymbol \nu-\hat {\boldsymbol\theta}_{i,t}}_{\A_{i,t}} \leq \beta_t\}$, where $ \hat {\boldsymbol\theta}_{i,t}= \A_{i,t}^{-1}\bb_{i,t}$ and $\beta_t$ is chosen as in Thm.~\ref{thm:confidencesettheta}.\;
Play $\x_{i,t} = \argmax_{\x\in \Dc}\max_{\boldsymbol \nu \in \Cc_{i,t}}\langle\boldsymbol \nu,\x\rangle$ and observe $y_{i,t}$.\;
$\W_{i,{\rm new}} = \W_{i,{\rm new}}+\x_{i,t}\x_{i,t}^T$, $\vb_{i,{\rm new}} = \vb_{i,{\rm new}}+y_{i,t}\x_{i,t}$\;
$\A_{i,t+1} = \lambda I+\W_{i,{\rm syn}}+\W_{i,{\rm new}}$, $\bb_{i,t+1} = \vb_{i,{\rm syn}}+\vb_{i,{\rm new}}$\;

\label{line:cond}\If{$\log\left(\det \A_{i,t+1}/ \det \A_{i,t_{k-1}+1} \right)(t-t_{k-1})>M$} 
{
Start $k$-th Communication phase.
}
 \If{$\mathbbm{1}\{{\rm CP}_{k}\}$ \tcp*{if $k$-th communication phase has started.}} 
{
$\M_{j,0}=\mathbf{0}, \M_{j,1}=\W_{j,{\rm new}}$, $\n_{j,0}=\mathbf{0}, \n_{j,1}=\vb_{j,{\rm new}}$ for all $j\in[N]$\;
\For{s=1,\ldots,S}
{
\For{ Agent $j=1,\ldots,N$}
{Play $\x_{j,t+s} = \x_{j,t}$ and observe $y_{j,t+s}$.\;
$\M_{j,s+1} = \U(\M_{j,s},\M_{j,s-1},s)$\;
$\n_{j,s+1} = \U(\n_{j,s},\n_{j,s-1},s)$\;}
}
\For{Agent $j=1,\ldots,N$}{$\W_{j,{\rm syn}}= \W_{j,{\rm syn}}+N\M_{j,S+1}$, $\vb_{j,{\rm syn}}= \vb_{j,{\rm syn}}+N\n_{j,S+1}$\;
$\W_{j,{\rm new}} = S\x_{j,t}\x_{j,t}^T$, $\vb_{j,{\rm new}}= \sum_{\ell=1}^Sy_{j,t+\ell}$\;
$\A_{j,t+S+1} = \lambda I+\W_{j,{\rm syn}}+\W_{j,{\rm new}}$, $\bb_{j,t+S+1} = \vb_{j,{\rm syn}}+\vb_{j,{\rm new}}$}
$t_k = t+S$, $t = t_k+1$, $k = k+1$}
\Else{$t=t+1$}
}
\caption{RC-DLUCB for Agent $i$}
\label{alg:DLUCB2}
\end{algorithm}

\subsection{Regret Analysis}

\begin{theorem}[Regret of RC-DLUCB]\label{thm:RC-DLUCBregret} Fix
$\epsilon\in(0,1/(2d+1))$ and $\delta\in(0,1)$. Let Assumptions \ref{assum:comm_matrix}, \ref{assum:noise}, \ref{assum:boundedness} hold, and $S$ be chosen as in \eqref{eq:cheb}. Then, with probability at least $1-\delta$, it holds that:
    \begin{align}\label{eq:RC-DLUCBreg}
       R_T({\text{\rm RC-DLUCB}})&\leq 4\beta_T\left(SNd\log\left(\la+\frac{NT}{d}\right)/\sqrt{\la}+ 4\log^{1.5}\left(\la+\frac{NT}{d}\right)\sqrt{dNT}\right)\nn\\
&= \Oc\left(Nd^{1.5}\frac{\log(Nd)}{\sqrt{\log(1/|\la_2|})}\log^{1.5}(NT)+d\log^2(NT)\sqrt{NT}\right).
    \end{align}
\end{theorem}
\begin{proof}
By the instantaneous regret decomposition of agent $i$ at round $t$, we have
\begin{align}
    r_{i,t} \leq 2\min(\beta_t \norm{\x_{i,t}}_{\A_{i,t}^{-1}},1).
\end{align}
Define:
\begin{align}
    \tilde\A_{i,t} = \begin{cases}
      \la I+ \sum_{\tau=1}^{t_{k-1}}\sum_{j=1}^N\x_{j,\tau}\x_{j,\tau}^T+\sum_{\tau=t_{k-1}+1}^{t-1}\x_{i,\tau}\x_{i,\tau}^T &, \text{if}\quad t\in[t_{k-1}+1:t_{k}-S]\\
      \la I+ \sum_{\tau=1}^{t_{k-1}}\sum_{j=1}^N\x_{j,\tau}\x_{j,\tau}^T+\sum_{\tau=t_{k-1}+1}^{t_{k}-S-1}\x_{i,\tau}\x_{i,\tau}^T &, \text{if}\quad t\in[t_{k}-S+1:t_{k}]
    \end{cases}
\end{align}
Let $K$ be the total number of epochs (communication phases). Recall the definition of $\A_{\ast,t}$ in \eqref{eq:sufficient}. Following the same argument in the proof of Lemma \ref{lemm:first1}, for any $\epsilon\in(0,1/(2d+1))$, $i\in[N]$, $k\in [K]$ and $t\in {\rm EP}_k$:
\begin{align}
    \norm{\x_{i,t}}^2_{\A_{i,t}^{-1}} \leq e\norm{\x_{i,t}}^2_{\tilde\A_{i,t}^{-1}}\leq e\norm{\x_{i,t}}^2_{\A_{\ast,t_{k-1}+1}^{-1}} \leq e\norm{\x_{i,t}}^2_{\A_{\ast,t_{k}+1}^{-1}}\frac{\det\A_{\ast,t_{k}+1}}{\det\A_{\ast,t_{k-1}+1}}.
\end{align}
 Note that $1\leq \frac{\det\A_{\ast,t_{k}+1}}{\det\A_{\ast,t_{k-1}+1}}$ for any $k\in[K]$, and also $ \frac{\det\A_{\ast,t_{K+1}+1}}{\det\A_{\ast,t_{0}+1}}\leq \left(1+\frac{NT}{d\la}\right)^d$. Thus, by the definition of $\B_{i,t}$ in \eqref{eq:Bik}, we conclude that
\begin{align}\label{eq:goodepochs}
    \frac{\det\A_{\ast,t_{k}+1}}{\det\A_{\ast,t_{k-1}+1}}\leq e \quad \text{and}\quad \norm{\x_{i,t}}^2_{\A_{i,t}^{-1}} \leq e^2\norm{\x_{i,t}}^2_{\A_{\ast,t_{k}+1}^{-1}}\leq e^2\norm{\x_{i,t}}^2_{\B_{i,t}^{-1}}
\end{align}
is true for all but at most $\phi(\la,d,N,T):=  d\log\left(1+\frac{NT}{d\la}\right)$ number epochs $k\in[K]$. We call the epoch $k\in[K]$ that satisfies
\begin{align}
   \frac{\det\A_{\ast,t_{k}+1}}{\det\A_{\ast,t_{k-1}+1}}\leq e 
\end{align}

a \emph{good epoch} and define the set of good epochs $\Pc_{\rm good}:=\{k\in[K]:\frac{\det\A_{\ast,t_{k}+1}}{\det\A_{\ast,t_{k-1}+1}}\leq e\}$. We decompose the cumulative regret $R_T({\text{\rm RC-DLUCB}})$ into two components as follows:
\begin{align}
    R_T({\text{\rm RC-DLUCB}}) = R_{T,\rm{bad}}({\text{\rm RC-DLUCB}})+R_{T,\rm{good}}({\text{\rm RC-DLUCB}}). \label{eq:decompose1}
\end{align}
The first term is the regret of bad epochs and the second term is the network's regret of rounds in good epochs. We first establish the upper bound on $R_{T,\rm{good}}({\text{\rm RC-DLUCB}})$ as follows:
\begin{align}
    R_{T,\rm{good}}({\text{\rm RC-DLUCB}}) &= 2\beta_T\sum_{k\in\Pc_{\rm good}} \sum_{t\in{\rm EP}_k}\sum_{i=1}^N\min\left(\norm{\x_{i,t}}_{\A_{i,t}^{-1}},1\right) \nn\\
    &\stackrel{\eqref{eq:goodepochs}}{\leq} 2e\beta_T\sum_{k\in\Pc_{\rm good}} \sum_{t\in{\rm EP}_k}\sum_{i=1}^N\min\left(\norm{\x_{i,t}}_{\B_{i,t}^{-1}},1\right) \nn\\
    &\leq 2e\beta_T\sum_{t=1}^T\sum_{i=1}^N \min\left(\norm{\x_{i,t}}_{\B_{i,t}^{-1}},1\right)\nn\\
    &\stackrel{ \eqref{eq:standardarg}}{\leq} 2e\beta_{T}\sqrt{2dNT\log\left(\la+\frac{NT}{d}\right)}.
\end{align}
Now suppose ${\rm EP}_{k}$ is a bad epoch and $\tilde R_k({\text{\rm RC-DLUCB}})$ is the regret incurred by the actions of this epoch.

\begin{align}
    \tilde R_k({\text{\rm RC-DLUCB}}) &= 2\beta_T\sum_{t=t_{k-1}+1}^{t_k}\sum_{i=1}^N\min\left(\norm{\x_{i,t}}_{\A_{i,t}^{-1}},1\right) \nn\\
    &= 2\beta_T\sum_{i=1}^N\left(S\norm{\x_{i,t_k-S}}_{\A_{i,t_k-S}^{-1}}+\sum_{t=t_{k-1}+1}^{t_k-S}\min\left(\norm{\x_{i,t}}_{\A_{i,t}^{-1}},1\right)\right)\nn\\
    &\leq 2\beta_T\left(SN/\sqrt{\la}+\sum_{i=1}^N\sqrt{\left(t_k-t_{k-1}-S\right)\sum_{t=t_{k-1}+1}^{t_k-S}\min\left(\norm{\x_{i,t}}^2_{\A_{i,t}^{-1}},1\right)}\right)\nn\\
    &\stackrel{\eqref{eq:standardarg}}{\leq} 2\beta_T\left(SN/\sqrt{\la}+\sum_{i=1}^N\sqrt{2\left(t_k-t_{k-1}-S\right)\log\left(\frac{\det \A_{i,t_k-S}}{\det\A_{i,t_{k-1}+1}}\right)}\right)\nn\\
    &\stackrel{\text{Line \ref{line:cond} of Alg.~\ref{alg:DLUCB2}}}{\leq} 4N\beta_T \left(S/\sqrt{\la}+\sqrt{M}\right).
\end{align}

Since the number of bad epochs is at most $\phi(\la,d,N,T)$, we deduce:
\begin{align}
    R_{T,{\rm bad}}({\text{\rm RC-DLUCB}})\leq 4N\beta_T\phi(\la,d,N,T)\left(S/\sqrt{\la}+\sqrt{M}\right),
\end{align}
and:
\begin{align}
R_T({\text{\rm RC-DLUCB}})\leq 4\beta_T\left(N\phi(\la,d,N,T)\left(S/\sqrt{\la}+\sqrt{M}\right)+e\sqrt{dNT\log\left(\la+\frac{NT}{d}\right)}\right)
\end{align}

By the choice of $M=\frac{T\log(1+NT/d\la)}{dN}$, we have:
\begin{align}
R_T({\text{\rm RC-DLUCB}})&\leq 4\beta_T\left(SNd\log\left(\la+\frac{NT}{d}\right)/\sqrt{\la}+ 4\log^{1.5}\left(\la+\frac{NT}{d}\right)\sqrt{dNT}\right)\nn\\
&= \Oc\left(Nd^{1.5}\frac{\log(Nd)}{\sqrt{\log(1/|\la_2|})}\log^{1.5}(NT)+d\log^2(NT)\sqrt{NT}\right).
\end{align}
\end{proof}

\subsection{Communication Cost}\label{sec:commcostofrcdlucb}
\begin{theorem}[Communication Cost of RC-DLUCB]\label{thm:RC-DLUCBcommcost} The communication cost of RC-DLUCB is at most $\Oc\left(d^3N^{2.5}\frac{\log(Nd)}{\sqrt{\log(1/|\la_2|)}}\right).$
\end{theorem}
\begin{proof}
We use a similar argument to the proof of Theorem 4 in \cite{wang2019distributed}.

Due to the choice of $S$ in \eqref{eq:cheb} and definition of $\A_{i,t}$ in \eqref{eq:rcdlucbsuff}, we observe that for any $k\in[K]$:
\begin{align} \label{eq:determ}
    (1-\epsilon)\A_{\ast,t_k+1}\leq \A_{i,t_k+1}\leq (1+\epsilon)\A_{\ast,t_k+1}.
\end{align}
Now, let the duration of $k$-th epoch be less than a fixed constant $C>0$ and assume agent $i$ has triggered the $k$-th communication phase. Following the same arguments in concluding \eqref{eq:namsovai2} from \eqref{eq:general}, we conclude from \eqref{eq:dett}, \eqref{eq:determ} and  $\epsilon\in(0,1/(2d+1))$ that

\begin{align}
    \left(1+\log\left(\det \A_{\ast,t_k+1}/ \det \A_{\ast,t_{k-1}+1} \right)\right)C\geq\log\left(\det \A_{i,t_k+1}/ \det \A_{i,t_{k-1}+1} \right)(t_k-t_{k-1})>M,
\end{align}
Hence:
\begin{align}\label{eq:logdet}
  \log\left(  \frac{\det \A_{\ast,t_k+1}}{\det \A_{\ast,t_{k-1}+1}}\right)> M/C-1.
\end{align}
Combining \eqref{eq:logdet} and the fact that $\sum_{k=1}^K \log\left(\frac{\det \A_{\ast,t_k+1}}{\det \A_{\ast,t_{k-1}+1}}\right) = \log\left(\frac{\det \A_{*,t_K+1}}{\det \A_{*,t_0+1}}\right) = \log\left(\frac{\det \A_{*,T+1}}{\det \A_{*,1}}\right)\leq \phi(\la,d,N,T)$, we conclude that there are at most $\lceil \frac{\phi(\la,d,N,T)}{M/C-1} \rceil$ number of epochs with length less than $C$. Clearly, the number of epochs with length larger than $C$ is at most $\lceil T/C \rceil$. Thus, plugging $C=\frac{M}{\sqrt{\phi(\la,d,N,T)M/T}+1}$, we deduce that the total number of epochs is at most:
\begin{align}\label{eq:numberofepochs}
    \lceil T/C \rceil+\left\lceil \frac{\phi(\la,d,N,T)}{M/C-1} \right\rceil =\Oc\left(\sqrt{\frac{T\phi(\la,d,N,T)}{M}}\right)=\Oc(d\sqrt{N}).
\end{align}
Moreover, since in the communication phases, agents communicate matrices of size $(d+1)\times d$ for $S$ rounds, we deduce the communication cost is at most $S\Oc\left(d^3N^{2.5}\right) =\Oc\left(d^3N^{2.5}\frac{\log(Nd)}{\sqrt{\log(1/|\la_2|)}}\right).$
\end{proof}

\section{Safe-DLUCB}\label{proof:safe-DLUCBregret}
In this section, we provide a more elaborate discussion on the content of Section~\ref{sec:Safe-DLUCB}. We begin with a simple example from ad placement that serves as a motivation for the setting of safe distributed LBs that we consider here.
On the distributed side, advertisement companies (aka agents) may decide to cooperate with a common goal of speeding up the learning of their common users' preferences. On the safety side, some users may be sensitive on certain types of ads (e.g. political, religious), which if they encounter, they will immediately terminate their subscription. Thus, the agents can classify such ads as unsafe and the remaining as safe, and try to avoid offering the unsafe ones with high probability. 
We present our Safe-DLUCB in Algorithm \ref{alg:SDLUCB}.

 In what follows, we introduce the missing technical details of Section~\ref{sec:Safe-DLUCB} that eventually lead to the proof of Theorem \ref{thm:safe-DLUCBregret}.

\subsection{Confidence Sets}\label{sec:mustarconfidence}
We start the technical details of this section by discussion on how the confidence sets $\Ec_{i,t}$ are constructed. At round $t$, agent $i$ computes the following:
\begin{align}
    \A_{i,t}^{\bot} = \begin{cases}
      \la \left(I-\xx_0\xx_0^T\right)+  \sum_{\tau=1}^{t-1} \left(\x^{\bot}_{i,\tau}\right)\left(\x^{\bot}_{i,\tau}\right)^T&, \text{if} \quad t\in[S],\\
      \la \left(I-\xx_0\xx_0^T\right)+  \sum_{\tau=1}^{t-S}\sum_{j=1}^N a_{i,j}^2 \left(\x^{\bot}_{j,\tau}\right)\left(\x^{\bot}_{j,\tau}\right)^T&, \text{if}\quad t\in[S+1:T],
    \end{cases}
\end{align}

\begin{align}
    \rb^\bot_{i,t} = \begin{cases}
    \sum_{\tau=1}^{t-1} z_{i,\tau}^{\bot} \x^{\bot}_{i,\tau}  &, \text{if}\quad t\in[S],\\
     \sum_{\tau=1}^{t-S}\sum_{j=1}^N a_{i,j}^2 z_{j,\tau}^{\bot} \x^{\bot}_{j,\tau} &, \text{if}\quad t\in[S+1:T],
    \end{cases}
\end{align}
where,
\begin{align}
    z_{j,t}^{\bot}:= \langle \boldsymbol\mu_\ast^{\bot},\x^{\bot}_{j,t}\rangle + \zeta_{j,t} = \langle \boldsymbol\mu_\ast,\x_{j,t}\rangle- \langle \boldsymbol\mu_\ast^{o},{\x}^{o}_{j,t}\rangle +\zeta_{j,t}= z_{j,t} - \langle \boldsymbol\mu_\ast^{o},{\x^{o}}_{j,t}\rangle = z_{j,t} - \frac{\langle {\x}_{j,t},\xx_0\rangle}{\norm{\x_0}} c_0.
\end{align}

Next, agent $i$ constructs the confidence set
\begin{align}\label{eq:gammaconfidenceset}
    \Ec_{i,t}:=\{\boldsymbol \nu \in \mathbb{R}^{d}: \norm{\boldsymbol \nu-\hat{\boldsymbol\mu}_{i,t}^{\bot}}_{\A_{i,t}^{\bot}} \leq \beta_t\},
\end{align}
where $\hat{\boldsymbol\mu}_{i,t}^{\bot} = \left({\A_{i,t}^{\bot}}\right)^{-1}\rb^{\bot}_{i,t}$.
Similar arguments to those in the proof of Thm.~\ref{thm:confidencesettheta}  guarantees that if $\beta_t$ is chosen
according to Thm.~\ref{thm:confidencesettheta}, then $\boldsymbol\mu_\ast^{\bot}\in \Ec_{i,t}$ for all $i\in[N]$ and $t\in[T]$ with probability at least $1-\delta$.

\paragraph{Conditioning on $\boldsymbol\mu_\ast^{\bot}\in\Ec_{i,t},~\forall i\in[N],t\in[T]$.}  
Consider the event 
\begin{align}
\bar\Ec_2:=\{\boldsymbol\mu_\ast^{\bot}\in \Ec_{i,t},~\forall i\in[N],t\in[T]\}\label{eq:conditioned2},
\end{align}
that $\boldsymbol\mu_\ast^{\bot}$ is inside the confidence sets for all agents $i\in[N]$ and rounds $t\in[T]$. Following the proof of Theorem \ref{thm:confidencesettheta}, we can prove that the event holds with probability at least $1-\delta$. Onwards, we condition on this event, and make repeated use of the fact that $\boldsymbol\mu_\ast^{\bot}\in\Ec_{i,t}$ for all $i\in[N],t\in[T]$, without further explicit reference.

Once $\Ec_{i,t}$ is constructed, agent $i$ creates the following \emph{safe} inner approximation of the true $\Ds(\boldsymbol\mu_\ast)$:
\begin{align}\label{eq:Ds}
   \Dts := \{\x \in \Dc:\frac{\langle \x^{o},\xx_0\rangle}{\norm{\x_0}} c_0+ \langle\hat{\boldsymbol\mu}_{i,t}^{\bot},\x^{\bot}\rangle
    +\beta_t\norm{\x^{\bot}}_{\left(\A_{i,t}^\bot\right)^{-1}}\leq c\}. 
\end{align}
Thanks to the following proposition, it is guaranteed that $\Dts$ is a set of safe actions with high probability. 
\begin{algorithm}[ht]
\DontPrintSemicolon
  \KwInput{$\x_0$, $c_0$, $c$, $\Dc$, $N$, $d$, $|\la_2|$, $\epsilon$, $\la$, $\delta$, $T$}
$S = \log(2N/\epsilon)/\sqrt{2\log(1/|\la_2|)}$, $\A_{i,1} = \la I$, $\bb_{i,1}=\mathbf{0}$, $\A^\bot_{i,1} = \la (I-\xx_0\xx_0^T)$, $\rb^\bot_{i,1}=\mathbf{0}$\;
$\Ac_{i,0} =  \Ac_{i,1}  =\Bc_{i,0} = \Bc_{i,1}= \Rc_{i,0} = \Rc_{i,1} = \emptyset$\;
\For{$t=1,\ldots,S$}
{Construct $\Cc_{i,t}:=\{\boldsymbol \nu \in \mathbb{R}^d: \|{\boldsymbol \nu-\hat {\boldsymbol\theta}_{i,t}}\|_{\A_{i,t}} \leq \beta_t\}$
where $\hat{\boldsymbol\theta}_{i,t} = \A_{i,t}^{-1}\bb_{i,t}$ and $\beta_t$ is chosen as in Thm.~\ref{thm:confidencesettheta}.\;
where $\hat{\boldsymbol\mu}_{i,t}^\bot = \left(\A_{i,t}^\bot\right)^{-1}\rb^\bot_{i,t}$ and $\beta_t$ is chosen as in Thm.~\ref{thm:confidencesettheta}.\;
Construct $\Dts = \{\x \in \Dc:\frac{\langle \x^{o},\xx_0\rangle}{\norm{\x_0}} c_0+ \langle\hat{\boldsymbol\mu}_{i,t}^{\bot},\x^{\bot}\rangle
    +\beta_t\norm{\x^{\bot}}_{\left(\A_{i,t}^{\bot}\right)^{-1}}\leq c\}$\;
Play $\x_{i,t} = \argmax_{\x\in \Dts}\max_{\boldsymbol \nu \in \kappa_r\Cc_{i,t}}\langle\boldsymbol \nu,\x\rangle$ and observe $y_{i,t}$ and $z_{i,t}$.\;
$\Ac_{i,t}.{\rm append}(\X_{i,t})$, $\Bc_{i,t}.{\rm append}(\y_{i,t})$ and $\Rc_{i,t}.{\rm append}(\z_{i,t})$\;
\For{$j=1:\ldots,t$}
{$[\Ac_{i,t+1}(j)]_{m,n} = \U([\Ac_{i,t}(j)]_{m,n},[\Ac_{i,t-1}(j+1)]_{m,n},S-j)$, $\forall m\in[N], n\in[d]$ \;
$[\Bc_{i,t+1}(j)]_m = \U([\Bc_{i,t}(j)]_m,[\Bc_{i,t-1}(j+1)]_m,S-j)$, for all $m\in[N]$\;
$[\Rc_{i,t+1}(j)]_m = \U([\Rc_{i,t}(j)]_m,[\Rc_{i,t-1}(j+1)]_m,S-j)$, for all $m\in[N]$}

$\A_{i,t+1} = \A_{i,t} + \x_{i,t}\x^T_{i,t}$, $\bb_{i,t+1} = \bb_{i,t}+ y_{i,t}\x_{i,t}$\;
$\A^\bot_{i,t+1} = \A^\bot_{i,t} + \left(\x^{\bot}_{i,t}\right)\left(\x^{\bot}_{i,t}\right)^T$, $\rb^\bot_{i,t+1} = \rb^\bot_{i,t}+ z^\bot_{i,t}\x^\bot_{i,t}$\;
}

$\A_{i,S}=\la I$, $\bb_{i,S}=\mathbf{0}$, $\A_{i,S}^\bot=\la (I-\xx_0\xx_0^T)$, $\rb_{i,S}^\bot=\mathbf{0}$\;
\For{$t=S+1,\ldots,T$}
{
$\A_{i,t} = \A_{i,t-1} + N^2\Ac_{i,t}(1)^T\Ac_{i,t}(1)$, $\bb_{i,t} = \bb_{i,t-1} +  N^2\Ac_{i,t}(1)^T\Bc_{i,t}(1)$\;
$\A^\bot_{i,t} = \A^\bot_{i,t-1}+ N^2\left(\Ac_{i,t}(1)^\bot\right)^T\left(\Ac_{i,t}(1)^\bot\right)$, $\rb^\bot_{i,t} = \rb^\bot_{i,t-1} +  N^2\left(\Ac_{i,t}(1)^\bot\right)^T\Rc^\bot_{i,t}(1)$\;
Construct $\Cc_{i,t}:=\{\boldsymbol \nu \in \mathbb{R}^d: \|{\boldsymbol \nu-\hat {\boldsymbol\theta}_{i,t}}\|_{\A_{i,t}} \leq \beta_t\}$
where $\hat{\boldsymbol\theta}_{i,t} = \A_{i,t}^{-1}\bb_{i,t}$ and $\beta_t$ is chosen as in Thm.~\ref{thm:confidencesettheta}.\;
where $\hat{\boldsymbol\mu}_{i,t}^\bot = \left(\A_{i,t}^\bot\right)^{-1}\rb^\bot_{i,t}$ and $\beta_t$ is chosen as in Thm.~\ref{thm:confidencesettheta}.\;
Construct $\Dts = \{\x \in \Dc:\frac{\langle \x^{o},\xx_0\rangle}{\norm{\x_0}} c_0+ \langle\hat{\boldsymbol\mu}_{i,t}^{\bot},\x^{\bot}\rangle
    +\beta_t\norm{\x^{\bot}}_{\left(\A_{i,t}^\bot\right)^{-1}}\leq c\}$\;

Play $\x_{i,t} = \argmax_{\x\in \Dts}\max_{\boldsymbol \nu \in \kappa_r\Cc_{i,t}}\langle\boldsymbol \nu,\x\rangle$ and observe $y_{i,t}$ and $z_{i,t}$.\;

 $\Ac_{i,t}.{\rm remove}\left(\Ac_{i,t}(1)\right).{\rm append}\left(\X_{i,t}\right)$\;
$\Bc_{i,t}.{\rm remove}\left(\Bc_{i,t}(1)\right).{\rm append}\left(\y_{i,t}\right)$\;
$\Rc_{i,t}.{\rm remove}\left(\Rc_{i,t}(1)\right).{\rm append}\left(\z_{i,t}\right)$\;
\For{$j=1:\ldots,S$}
{$[\Ac_{i,t+1}(j)]_{m,n} = \U([\Ac_{i,t}(j)]_{m,n},[\Ac_{i,t-1}(j+1)]_{m,n},S-j)$, $\forall m\in[N], n\in[d]$ \;
$[\Bc_{i,t+1}(j)]_m = \U([\Bc_{i,t}(j)]_m,[\Bc_{i,t-1}(j+1)]_m,S-j)$, for all $m\in[N]$\;
$[\Rc_{i,t+1}(j)]_m = \U([\Rc_{i,t}(j)]_m,[\Rc_{i,t-1}(j+1)]_m,S-j)$, for all $m\in[N]$}
}
  \caption{Safe-DLUCB for Agent $i$}
\label{alg:SDLUCB}
\end{algorithm}

\begin{proposition}\label{prop:safe}
For all $i\in[N]$, $t\in[T]$, and $\delta\in(0,1)$, with probability at least $1-\delta$ all actions in $\Dts$ are safe, i.e., $\langle \boldsymbol\mu_\ast,\x \rangle \leq c,~\forall \x\in\Dts$ . 
\end{proposition}

\begin{proof}
Let $\x\in \Dts$. By the definition of $\Dts$ in \eqref{eq:Ds}, we have
\begin{align}\label{eq:equ1}
    \frac{\langle \x^{o},\xx_0\rangle}{\norm{\x_0}} c_0+ \langle\hat{\boldsymbol\mu}_{i,t}^{\bot},\x^{\bot}\rangle+\beta_t\norm{\x^{\bot}}_{\left({\A_{i,t}^{\bot}}\right)^{-1}}\leq c.
\end{align}
Moreover, using Cauchy-Schwarz inequality and conditioned on the event $\bar \Ec_2$ in \eqref{eq:conditioned2}, we get
\begin{align}\label{eq:equ2}
    |\langle\hat{\boldsymbol\mu}_{i,t}^{\bot}-\boldsymbol\mu_\ast^{\bot},\x^{\bot}\rangle|\leq \beta_t \norm{\x^{\bot}}_{\left({\A_{i,t}^{\bot}}\right)^{-1}} \Rightarrow \langle\boldsymbol\mu_\ast^{\bot},\x^{\bot}\rangle\leq \langle\hat{\boldsymbol\mu}_{i,t}^{\bot},\x^{\bot}\rangle+\beta_t \norm{\x^{\bot}}_{\left({\A_{i,t}^{\bot}}\right)^{-1}}.
\end{align}
Note that $ \langle \boldsymbol\mu_\ast^{\bot},\x^{\bot}\rangle = \langle \boldsymbol\mu_\ast,\x\rangle- \langle \boldsymbol\mu_\ast^{o},{\x^{o}}\rangle = \langle\boldsymbol\mu_\ast,\x\rangle - \frac{\langle {\x},\xx_0\rangle}{\norm{\x_0}} c_0$. Combining this fact with \eqref{eq:equ1} and \eqref{eq:equ2} concludes that
\begin{align}
    \langle \boldsymbol\mu_\ast,\x\rangle = \frac{\langle {\x},\xx_0\rangle}{\norm{\x_0}} c_0+\langle \boldsymbol\mu_\ast^{\bot},\x^{\bot}\rangle \leq \frac{\langle {\x},\xx_0\rangle}{\norm{\x_0}} c_0+\langle\hat{\boldsymbol\mu}_{i,t}^{\bot},\x^{\bot}\rangle+\beta_t \norm{\x^{\bot}}_{\left({\A_{i,t}^{\bot}}\right)^{-1}}\leq c,
\end{align}
which implies that $\x$ is safe, as desired.
\end{proof}

\subsection{Optimism in the face of Safety Constraint}\label{proof:optimisimsafety}
In this section, we show how large the value of $\kappa_r$ must be to provide enough exploration to the algorithm so that the selected actions in $\Dts$ are -often enough- \emph{optimistic}, i.e., $\langle\tilde {\boldsymbol\theta}_{i,t},\x_{i,t}\rangle\geq \langle\boldsymbol\theta_\ast,\x_\ast\rangle$ \cite{abeille2017linear}.

First, we state the following lemma borrowed from \cite{pacchiano2020stochastic} used in the proof of Lemma \ref{lemm:optimisimsafety}. 
\begin{lemma}[\cite{pacchiano2020stochastic}] \label{lemm:normineqality}
For any vector $\x\in \mathbb{R}^d$, the following inequality holds:
\begin{align}
    \norm{\x^{\bot}}_{\left({\A_{i,t}^{\bot}}\right)^{-1}} \leq \norm{\x}_{\A_{i,t}^{-1}}.
\end{align}
\end{lemma}
\begin{proof} For a proof see \cite[Lemma 3]{pacchiano2020stochastic}.
\end{proof}

Now, we state the main Lemma \ref{lemm:optimisimsafety}.

\begin{lemma}[Optimism in the face of safety constraint]\label{lemm:optimisimsafety}
Let $\kappa_r \geq \frac{2}{c-c_0}+1$. Then for all $i\in [N]$ and $t\in[T]$, with probability at least $1-\delta$, $\langle\tilde {\boldsymbol\theta}_{i,t},\x_{i,t}\rangle\geq \langle\boldsymbol\theta_\ast,\x_\ast\rangle$.
\end{lemma}

\begin{proof}
The proof of the lemma uses ideas introduced recently for single-agent LBs with linear constraints in \cite{amani2019linear,moradipari2019safe,pacchiano2020stochastic}. We consider the following two cases:

    1) First, if $\x_\ast\in \Dts$, then by the definition of $\langle\tilde {\boldsymbol\theta}_{i,t},\x_{i,t}\rangle$ in \eqref{eq:SafeDLUCBdecision}, one can easily observe that
\begin{align}
    \langle\tilde {\boldsymbol\theta}_{i,t},\x_{i,t}\rangle\geq \langle\boldsymbol\theta_\ast,\x_\ast\rangle,
\end{align}
as desired.

2) Now, we focus on the other case when $\x_\ast \notin \Dts$, which means
\begin{align}\label{eq:mgreaterthanc}
    \frac{\langle \x_\ast^{o},\xx_0\rangle}{\norm{\x_0}} c_0+ \langle\hat{\boldsymbol\mu}_{i,t}^{\bot},\x_\ast^{\bot}\rangle
    +\beta_t\norm{\x_\ast^{\bot}}_{\left(\A_{i,t}^\bot\right)^{-1}}> c.
\end{align}

Recall that $\xx_0=\frac{\x_0}{\norm{\x_0}_2}$ and note that $\x_0^o=\x_0$ and $\x_0^\bot=\mathbf{0}$. Thus, $\frac{\langle \x_0^{o},\xx_0\rangle}{\norm{\x_0}} c_0+ \langle\hat{\boldsymbol\mu}_{i,t}^{\bot},\x_0^{\bot}\rangle
    +\beta_t\norm{\x_0^{\bot}}_{\left(\A_{i,t}^\bot\right)^{-1}} = c_0 < c$, which implies that $\x_0\in\Dts$. Now, for each $i\in[N]$ and $t\in[T]$, let $\w_{i,t}: = \alpha_{i,t} \x_\ast+(1-\alpha_{i,t})\x_0$ be a point on the line connecting $\x_\ast$ and $\x_{i,t}$ such that $\w_{i,t}$ is on the boundary of $\Dts$. Formally:
\begin{align}
    \alpha_{i,t}:= \max\{\alpha\in[0,1]:\alpha\x_\ast+(1-\alpha)\x_{0}\in \Dts\}
\end{align}
Note that $\left(\alpha\x_\ast+(1-\alpha)\x_{0}\right)^o = \alpha\x^{o}_\ast+(1-\alpha)\x_{0} $ and $\left(\alpha\x_\ast+(1-\alpha)\x_{0}\right)^{\bot} = \alpha\x^{\bot}_\ast$. 
By the definition of $\Dts$ in \eqref{eq:Ds} and Proposition \ref{prop:safe} we have
\begin{align}
\alpha_{i,t}:= \max\left\{\alpha\in[0,1]: \frac{\alpha\langle  \x_\ast^{o},\xx_0\rangle+(1-\alpha)\langle\x_0,\xx_0\rangle}{\norm{\x_0}} c_0+ \alpha \left[ \langle\hat{\boldsymbol\mu}_{i,t}^{\bot},\x_\ast^{\bot}\rangle+\beta_t\norm{\x_\ast^{\bot}}_{\left({\A_{i,t}^{\bot}}\right)^{-1}} \right] \leq c \right\}.\label{eq:alphat}
\end{align}
The definition of $\alpha_{i,t}$ in \eqref{eq:alphat} implies that
\begin{align}
    \frac{(1-\alpha_{i,t})\langle\x_0,\xx_0\rangle}{\norm{\x_0}} c_0+ \alpha_{i,t} \left[\frac{\langle  \x_\ast^{o},\xx_0\rangle}{\norm{\x_0}}c_0+ \langle\hat{\boldsymbol\mu}_{i,t}^{\bot},\x_\ast^{\bot}\rangle+\beta_t\norm{\x_\ast^{\bot}}_{\left({\A_{i,t}^{\bot}}\right)^{-1}} \right] = c.
\end{align}
Let $M = \frac{\langle\x_\ast^{o},\xx_0\rangle}{\norm{\x_0}}c_0+ \langle \hat{\boldsymbol\mu}_{i,t}^{\bot},\x_\ast^{\bot}\rangle+\beta_t\norm{\x_\ast^{\bot}}_{\left({\A_{i,t}^{\bot}}\right)^{-1}}$. Note that due to \eqref{eq:mgreaterthanc}, $M>c$, and recall that $\xx_0 = \frac{\x_0}{\norm{\x_0}_2}$. Then, we have
 \begin{align}
     0<\alpha_{i,t} = \frac{c-\frac{\langle\x_0,\xx_0\rangle}{\norm{\x_0}} c_0}{M-\frac{\langle\x_0,\xx_0\rangle}{\norm{\x_0}} c_0} = \frac{c-c_0}{M-c_0}<1.
 \end{align}

Now, we show that $\alpha_{i,t}\geq \frac{c-c_0}{c-c_0+2\beta_t\norm{\x_\ast^{\bot}}_{\left({\A_{i,t}^{\bot}}\right)^{-1}}}$, which eventually leads to a proper value for $\kappa_r>1$ that guarantees that the selected safe action $\x_{i,t}$ from the conservative safe set $\Dts$ is always optimistic, i.e.,  $\langle\tilde {\boldsymbol\theta}_{i,t},\x_{i,t}\rangle\geq \langle\boldsymbol\theta_\ast,\x_\ast\rangle$, with high probability.

In order to lower bound $\alpha_{i,t}$ (upper bound $M$), we first rewrite $M$ as the following:
\begin{align}
    M = \frac{\langle \x_\ast^{o},\xx_0\rangle}{\norm{\x_0}}c_0+\langle\boldsymbol\mu_\ast,\x_\ast^{\bot}\rangle+ \langle\hat{\boldsymbol\mu}_{i,t}^{\bot}-\boldsymbol\mu_\ast,\x_\ast^{\bot}\rangle+\beta_t\norm{\x_\ast^{\bot}}_{\left({\A_{i,t}^{\bot}}\right)^{-1}}, \label{eq:M}
\end{align}
and show that: a) $\frac{\langle \x_\ast^{o},\xx_0\rangle}{\norm{\x_0}}c_0+\langle\boldsymbol\mu_\ast,\x_\ast^{\bot}\rangle \leq c$ because:
    \begin{align}
        \langle\boldsymbol\mu_\ast,\x_\ast\rangle = \langle\boldsymbol\mu_\ast,\x_\ast^{o}\rangle + \langle\boldsymbol\mu_\ast,\x_\ast^{\bot}\rangle = \langle\boldsymbol\mu_\ast,\langle\x_\ast,\xx_0\rangle\xx_0\rangle + \langle\boldsymbol\mu_\ast,\x_\ast^{\bot}\rangle = \frac{\langle \x_\ast^{o},\xx_0\rangle}{\norm{\x_0}}c_0+\langle\boldsymbol\mu_\ast,\x_\ast^{\bot}\rangle\leq c, \label{eq:M1}
    \end{align}
and b) $\langle\hat{\boldsymbol\mu}_{i,t}^{\bot}-\boldsymbol\mu_\ast,\x_\ast^{\bot}\rangle \leq \beta_t\norm{\x_\ast^{\bot}}_{\left({\A_{i,t}^{\bot}}\right)^{-1}} $ because:
    \begin{align}
        \langle\hat{\boldsymbol\mu}_{i,t}^{\bot}-\boldsymbol\mu_\ast,\x_\ast^{\bot}\rangle = \langle\hat{\boldsymbol\mu}_{i,t}^{\bot}-\boldsymbol\mu_\ast^{\bot},\x_\ast^{\bot}\rangle \stackrel{\eqref{eq:conditioned2}}{\leq} \beta_t\norm{\x_\ast^{\bot}}_{\left({\A_{i,t}^{\bot}}\right)^{-1}}\label{eq:M2}.
    \end{align}

Next, we combine \eqref{eq:M} with \eqref{eq:M1} and \eqref{eq:M2} to conclude that
\begin{align}
    M \leq c+ 2\beta_t\norm{\x_\ast^{\bot}}_{\left({\A_{i,t}^{\bot}}\right)^{-1}}\quad\Rightarrow\quad \alpha_{i,t}\geq \frac{c-c_0}{c-c_0+2\beta_t\norm{\x_\ast^{\bot}}_{\left({\A_{i,t}^{\bot}}\right)^{-1}}}.
\end{align}
Since $\w_{i,t}: = \alpha_{i,t} \x_\ast+(1-\alpha_{i,t})\x_0$, we have:
\begin{align}
    \langle\tilde{\boldsymbol\theta}_{i,t},\x_{i,t}\rangle &= \langle\hat{\boldsymbol\theta}_{i,t},\x_{i,t}\rangle+\kappa_r\beta_t\norm{\x_{i,t}}_{\A_{i,t}^{-1}} \geq  \langle\hat {\boldsymbol\theta}_{i,t},\w_{i,t}\rangle+\kappa_r\beta_t\norm{\w_{i,t}}_{\A_{i,t}^{-1}} =  \langle\hat{\boldsymbol\theta}_{i,t}-\boldsymbol\theta_\ast,\w_{i,t}\rangle+\langle\boldsymbol\theta_\ast,\w_{i,t}\rangle+\kappa_r\beta_t\norm{\w_{i,t}}_{\A_{i,t}^{-1}} \nn \\
   &\stackrel{\eqref{eq:conditioned1}}{\geq} \langle\boldsymbol\theta_\ast,\w_{i,t}\rangle+(\kappa_r-1)\beta_t\norm{\w_{i,t}}_{\A_{i,t}^{-1}}\stackrel{\text{Lemma~\ref{lemm:normineqality}}}{\geq} \alpha_{i,t}\langle\boldsymbol\theta_\ast,\x_\ast\rangle+(1-\alpha_{i,t})\langle\boldsymbol\theta_\ast,\x_0\rangle+(\kappa_r-1)\alpha_{i,t}\beta_t\norm{\x_\ast^{\bot}}_{\left({\A_{i,t}^{\bot}}\right)^{-1}}\nn\\
   &\stackrel{\langle\boldsymbol\theta_\ast,\x_0\rangle\geq 0}{\geq}\underbrace{\left(\frac{c-c_0}{c-c_0+2\beta_t\norm{\x_\ast^{\bot}}_{\left({\A_{i,t}^{\bot}}\right)^{-1}}}\right)\left[\langle\boldsymbol\theta_\ast,\x_\ast\rangle+(\kappa_r-1)\beta_t\norm{\x_\ast^{\bot}}_{\left({\A_{i,t}^{\bot}}\right)^{-1}}\right]}_{M_1} .
\end{align}
Let $M_2 = \beta_t\norm{\x_\ast^{\bot}}_{\left({\A_{i,t}^{\bot}}\right)^{-1}}$. We observe that $\langle\tilde{\boldsymbol\theta}_{i,t},\x_{i,t}\rangle\geq \langle\boldsymbol\theta_{\ast},\x_\ast\rangle$ iff:
\begin{align}
     M_1 = \frac{(c-c_0)\left(\langle\boldsymbol\theta_\ast,\x_\ast\rangle+(\kappa_r-1)M_2\right)}{(c-c_0)+2M_2} \geq \langle\boldsymbol\theta_\ast,\x_\ast\rangle &\iff (c-c_0)\left(\langle\boldsymbol\theta_\ast,\x_\ast\rangle+(\kappa_r-1)M_2\right) \geq (c-c_0+2M_2)\langle\boldsymbol\theta_\ast,\x_\ast\rangle\nn\\
     &\iff (c-c_0)(\kappa_r-1) \geq 2\langle\boldsymbol\theta_\ast,\x_\ast\rangle\nn\\
     &\iff (c-c_0)(\kappa_r-1) \geq 2 \nn\\
     &\iff \kappa_r\geq \frac{2}{c-c_0}+1,
     \end{align}
as desired.

\end{proof}

\subsection{Completing the proof of Theorem \ref{thm:safe-DLUCBregret}}

We use the following decomposition for bounding the regret:

\begin{align}
r_{i,t} = \langle\boldsymbol\theta_\ast,\x_\ast\rangle - \langle\boldsymbol\theta_\ast,\x_{i,t}\rangle =  \underbrace{\langle\boldsymbol\theta_\ast,\x_\ast\rangle-\langle\tilde {\boldsymbol\theta}_{i,t},\x_{i,t}\rangle}_{\rm{Term~I}}+ \underbrace{\langle\tilde{\boldsymbol\theta}_{i,t},\x_{i,t}\rangle- \langle\boldsymbol\theta_\ast,\x_{i,t}\rangle}_{\rm{Term~II}}
\end{align}
Lemma \ref{lemm:optimisimsafety} implies that for any $i\in[N]$ and $t\in[T]$ with probability at least $1-\delta$, ${\rm Term~II}\leq 0$. Hence all that remains to complete the proof of Theorem \ref{thm:safe-DLUCBregret} is to bound ${\rm Term~II}$ which follows the same procedure as proof of Theorem \ref{thm:DLUCBregret}. Putting things together, we conclude that for any $\epsilon \in (0,1/(4d+1))$, $\delta\in(0,0.5)$, and $\kappa_r \geq \frac{2}{c-c_0}+1$, with probability at least $1-2\delta$, we have:
    \begin{align}\label{eq:regret2}
R_{T}({\text{\rm Safe-DLUCB}})&\leq4Sd\log\left(1+\frac{NT}{d\la}\right)+2e\kappa_r\beta_{T}\sqrt{2dNT\log\left(\la+\frac{NT}{d}\right)}.
\end{align}

\begin{remark}[On Assumption \ref{assum:nonempty}]
Before closing, we remark that Assumption \ref{assum:nonempty} introduced in Section \ref{sec:Safe-DLUCB} is a slightly relaxed version of a corresponding assumption in \cite{moradipari2019safe}. Concretely, \cite{moradipari2019safe} assumes the known safe action to be the origin, $\x_0=\mathbf{0}$ and $c_0=0$. The relaxed Assumption \ref{assum:nonempty} is borrowed by \cite{pacchiano2020stochastic} relying on the same assumption, which was introduced therein under a slightly relaxed notion of safety.
\end{remark}

\section{Additional Experiments}\label{app:experiments}

\begin{figure}[ht]
    \centering
     \begin{subfigure}[b]{0.35\textwidth}
         \centering
         \includegraphics[width=\textwidth]{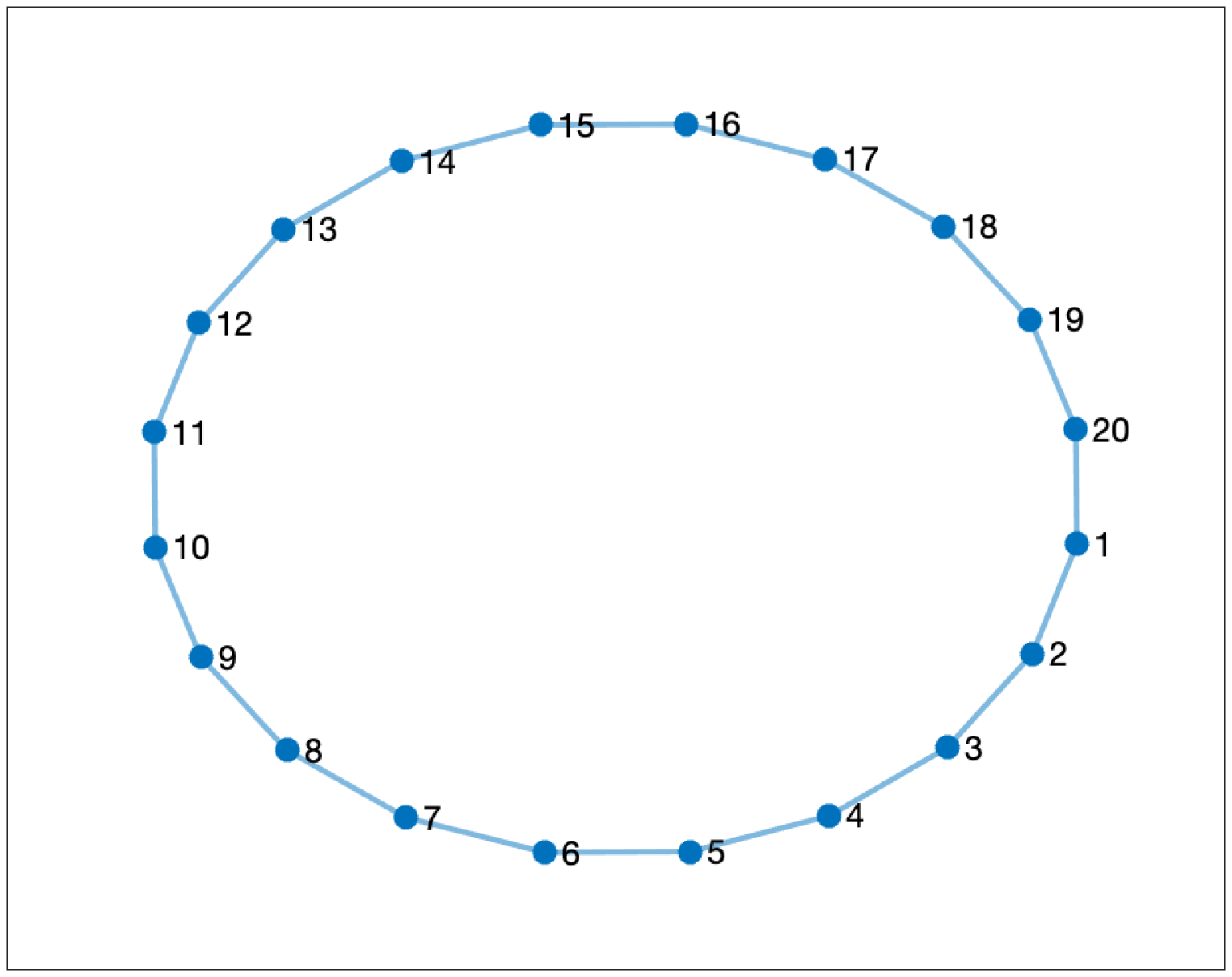}, 
          \caption{Ring, $N=20$, $|\la_2|=0.9674$}
          \label{fig:ring}
     \end{subfigure}
    \centering
     \begin{subfigure}[b]{0.35\textwidth}
         \centering
         \includegraphics[width=\textwidth]{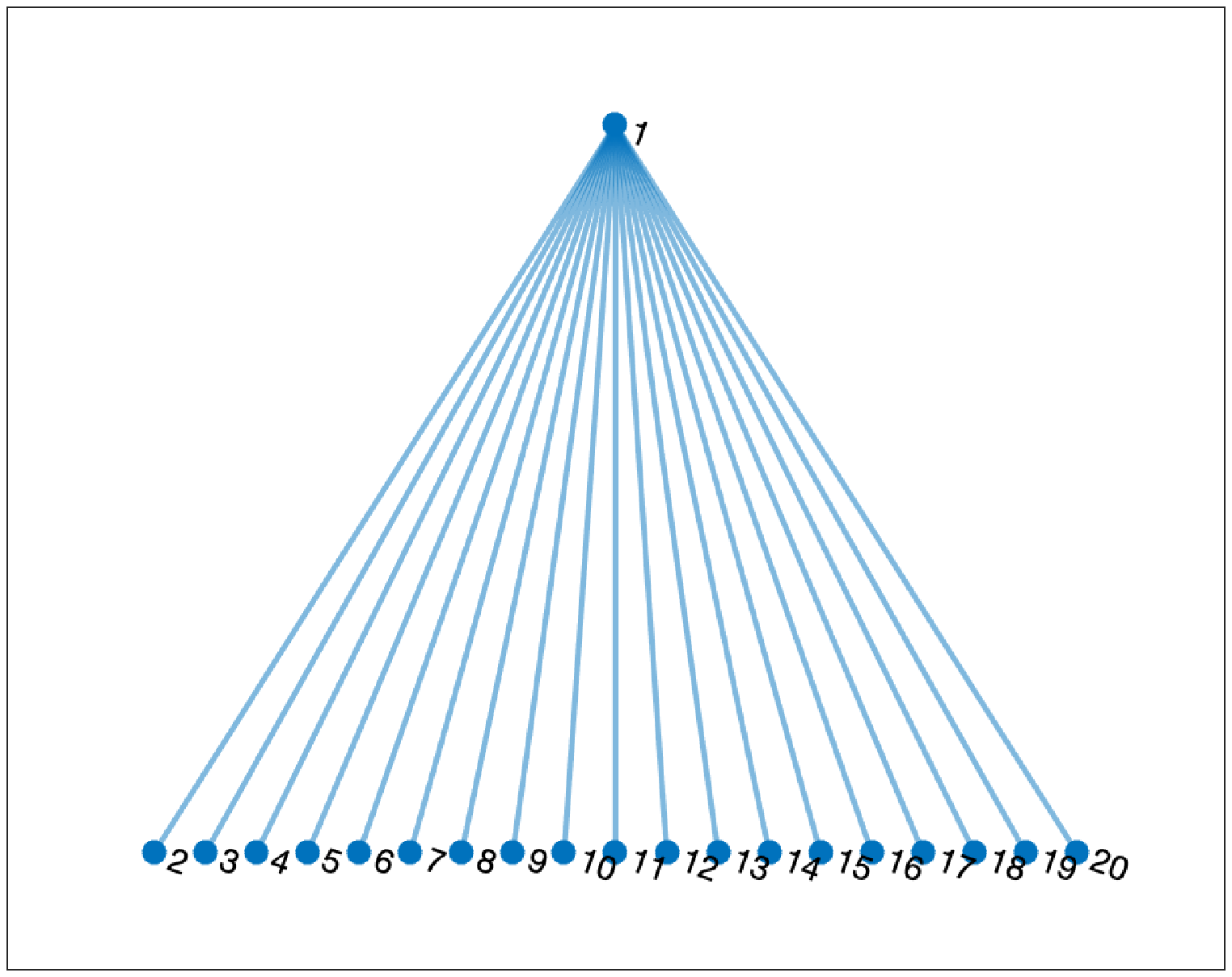}, 
          
        \caption{Star, $N=20$, $|\la_2|=0.9500$}
          \label{fig:star}
     \end{subfigure}
    \centering
     \begin{subfigure}[b]{0.35\textwidth}
         \centering
         \includegraphics[width=\textwidth]{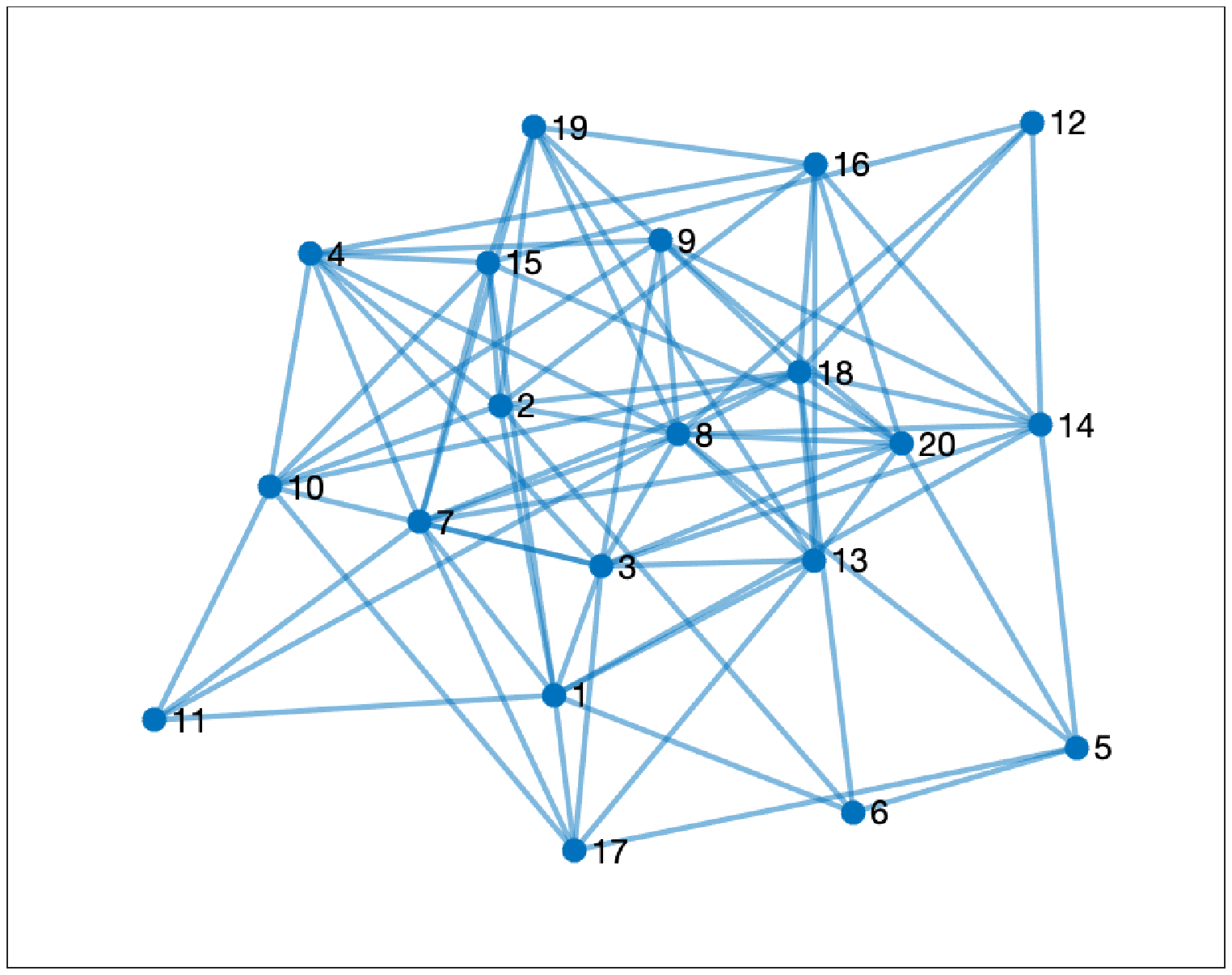},
         \caption{Random, $N=20$, $|\la_2|=0.6955$}
         \label{fig:random}
     \end{subfigure}
    \centering
     \begin{subfigure}[b]{0.35\textwidth}
         \centering
         \includegraphics[width=\textwidth]{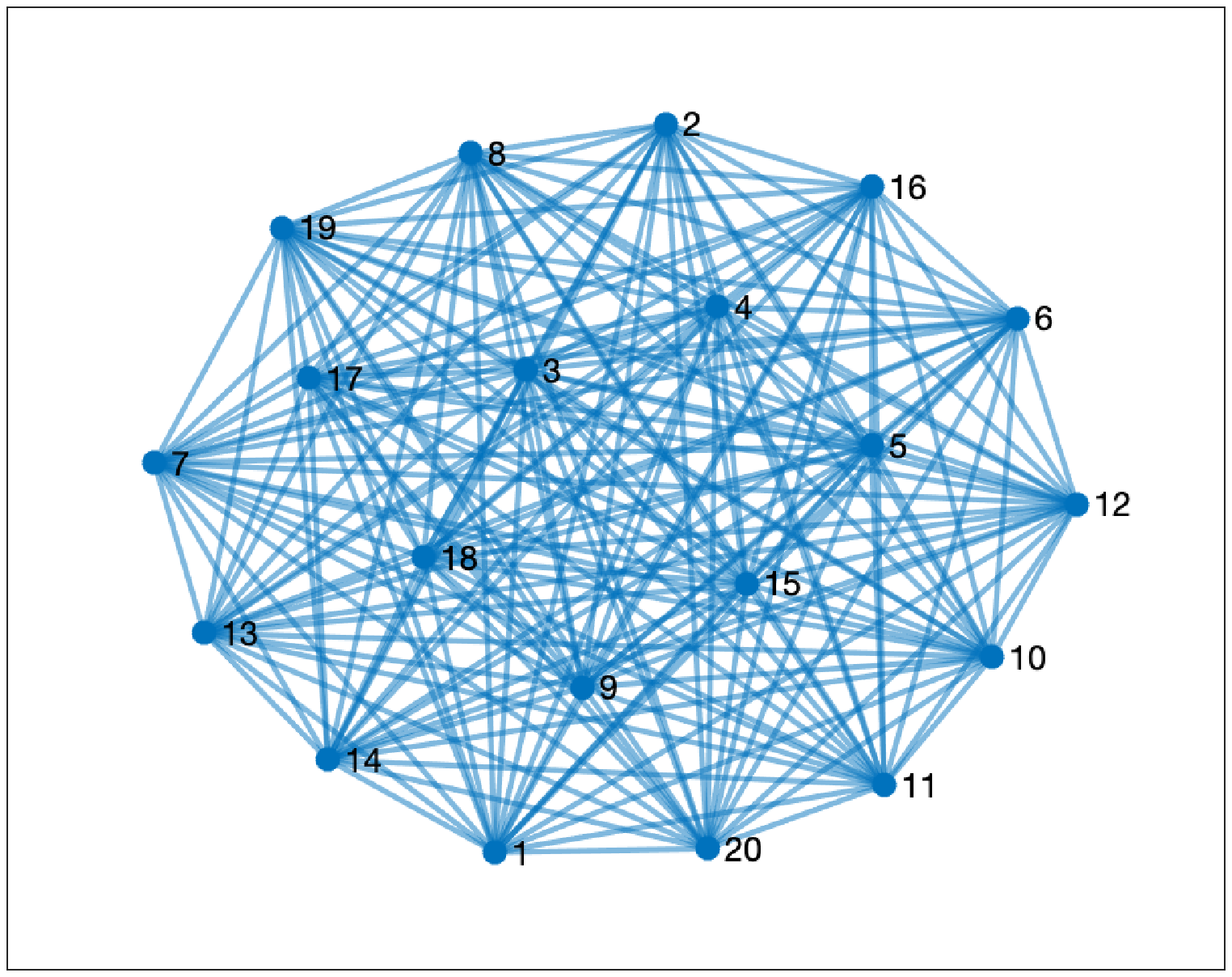},
        \caption{Complete, $N=20$, $|\la_2|=0$}
         \label{fig:complete}
     \end{subfigure}
     \caption{Graph topologies}
     \label{fig:graphtopologies}
\end{figure}

\begin{figure}[ht]
\centering
\begin{subfigure}{2.3in}
\begin{tikzpicture}
\node at (0,0) {\includegraphics[scale=0.25]{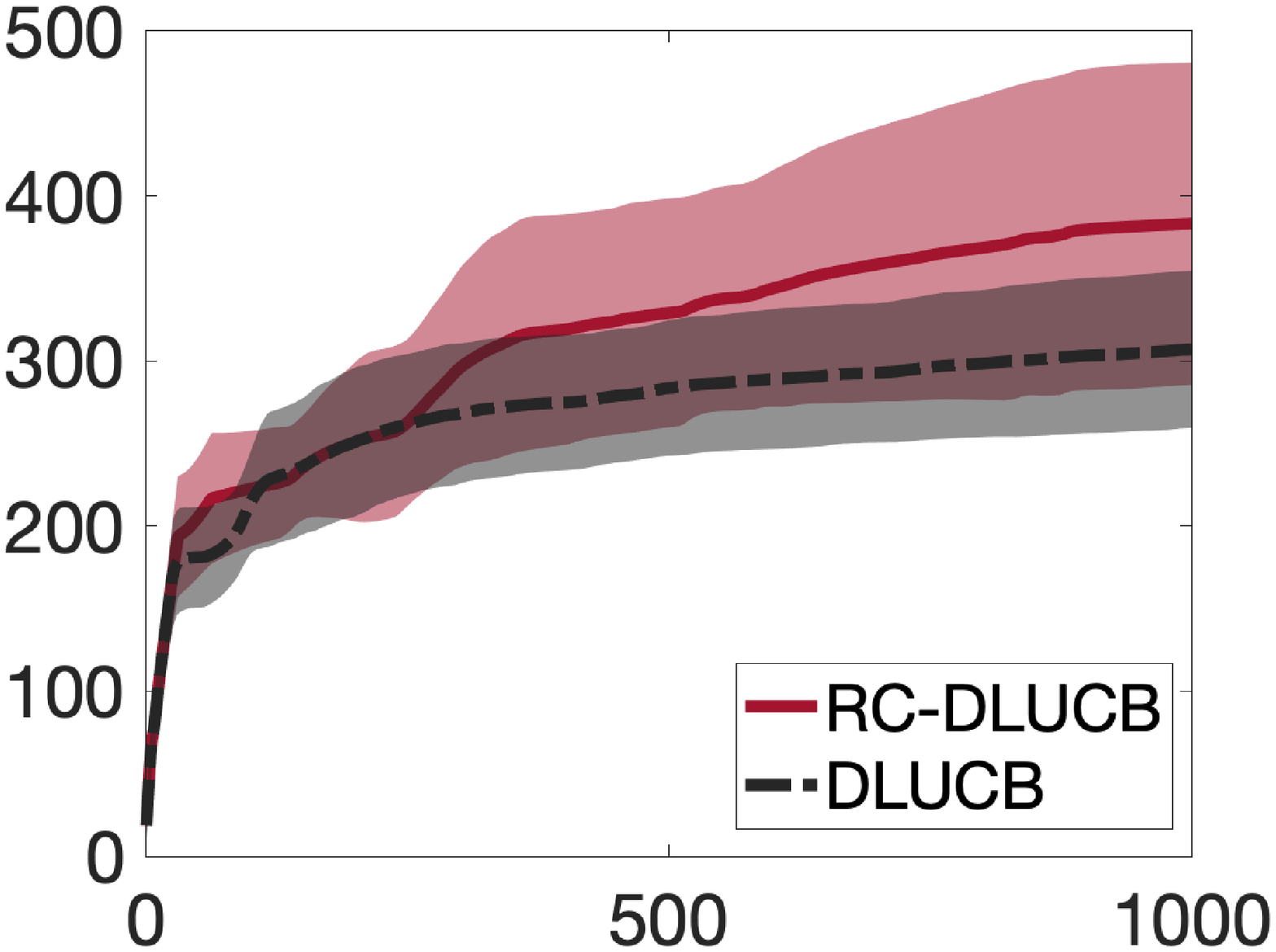}};
\node at (-2.6,0) [rotate=90,scale=0.9]{Regret, $R_t$};
\node at (0,-2) [scale=0.9]{Iteration, $t$};
\end{tikzpicture}
\caption{Ring, $N=20$, $|\la_2|=0.9674$, $S=26$}
\label{fig:ringcompare}
\end{subfigure}
\centering
\begin{subfigure}{2.3in}
\begin{tikzpicture}
\node at (0,0) {\includegraphics[scale=0.25]{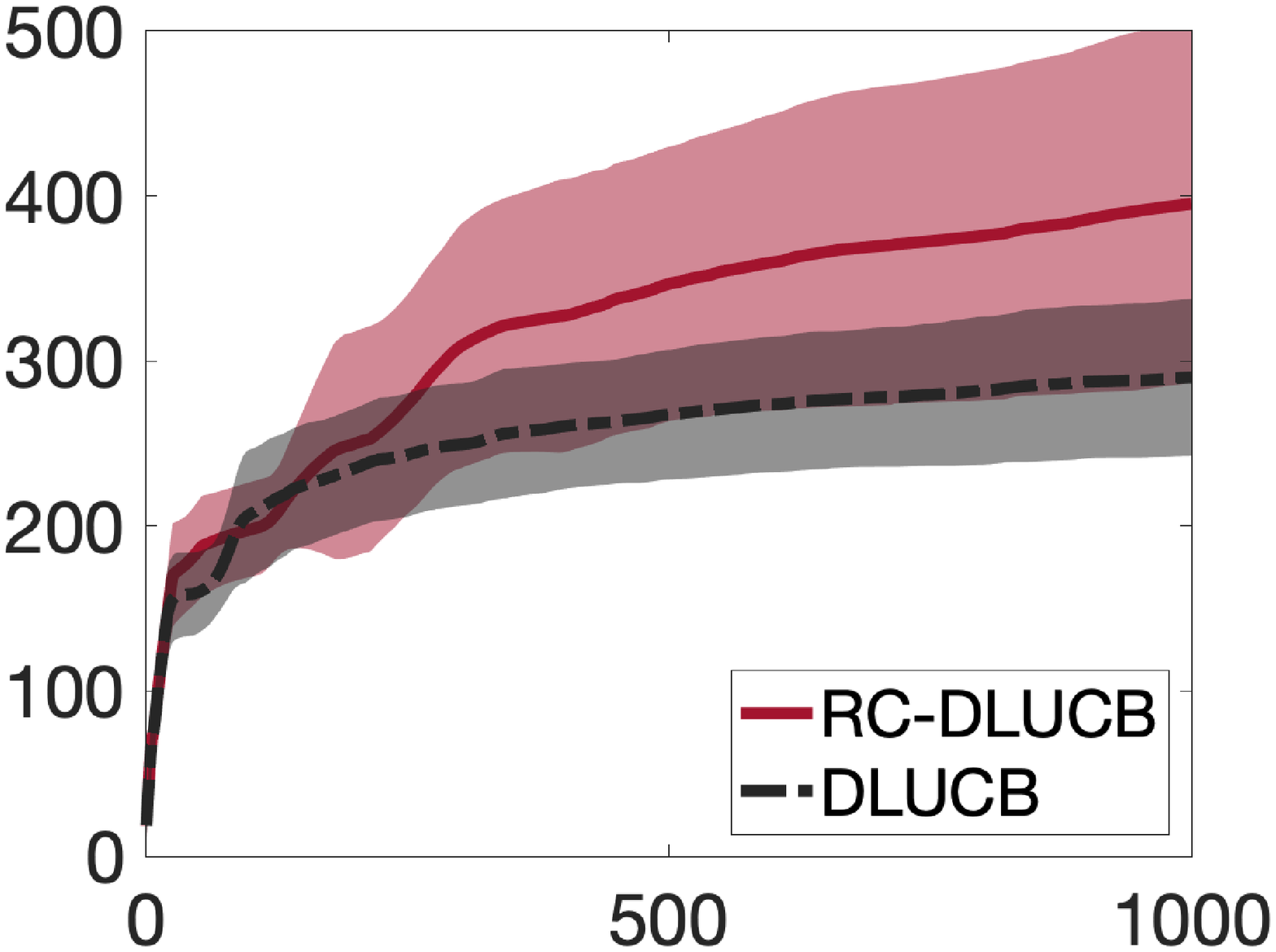}};
\node at (-2.6,0) [rotate=90,scale=1.]{Regret, $R_t$};
\node at (0,-2) [scale=0.9]{Iteration, $t$};
\end{tikzpicture}
\caption{Star, $N=20$, $|\la_2|=0.9500$, $S=21$}
\label{fig:starcompare}
\end{subfigure}
\centering
\begin{subfigure}{2.3in}
\begin{tikzpicture}
\node at (0,0) {\includegraphics[scale=0.25]{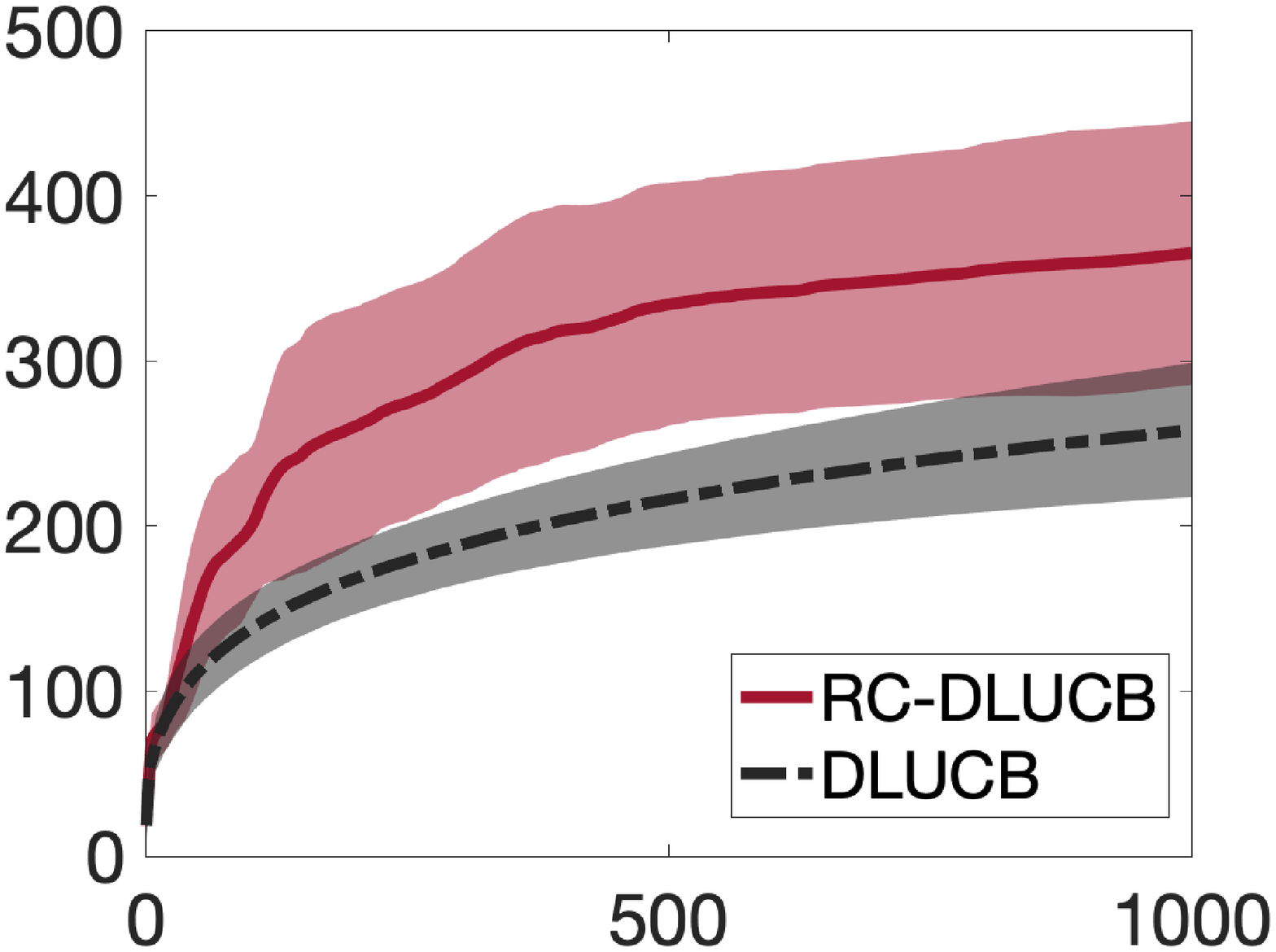}};
\node at (-2.6,0) [rotate=90,scale=0.9]{Regret, $R_t$};
\node at (0,-2) [scale=0.9]{Iteration, $t$};
\end{tikzpicture}
\caption{Random, $N=20$, $|\la_2|=0.6955$, $S=10$}
\label{fig:randomcompare}
\end{subfigure}
\centering
\begin{subfigure}{2.3in}
\begin{tikzpicture}
\node at (0,0) {\includegraphics[scale=0.25]{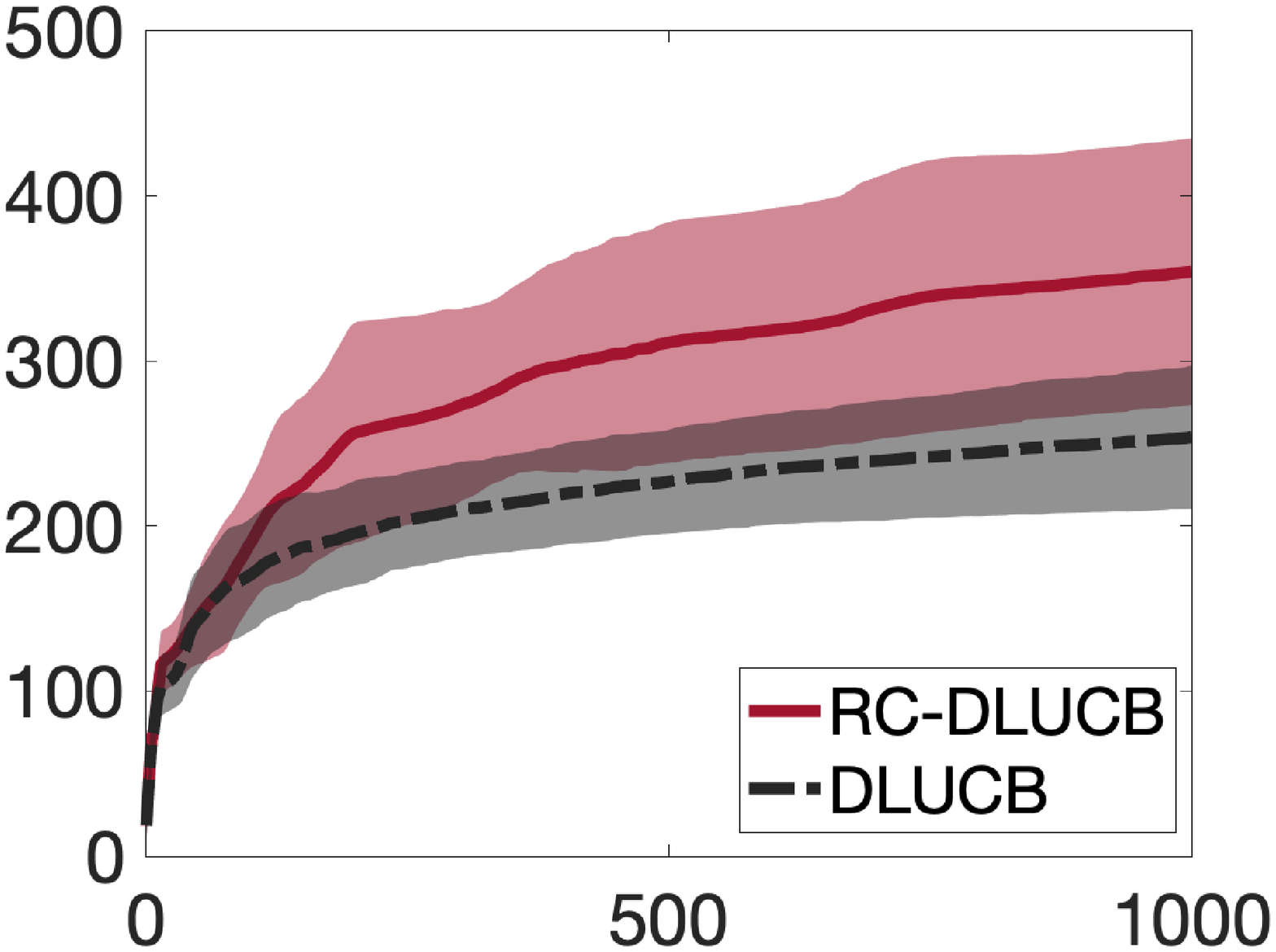}};
\node at (-2.6,0) [rotate=90,scale=0.9]{Regret, $R_t$};
\node at (0,-2) [scale=0.9]{Iteration, $t$};
\end{tikzpicture}
\caption{Complete, $N=20,$ $|\la_2|=0$, $S=1$}
\label{fig:completecompare}
\end{subfigure}
\centering
\begin{subfigure}{2.3in}
\begin{tikzpicture}
\node at (0,0) {\includegraphics[scale=0.25]{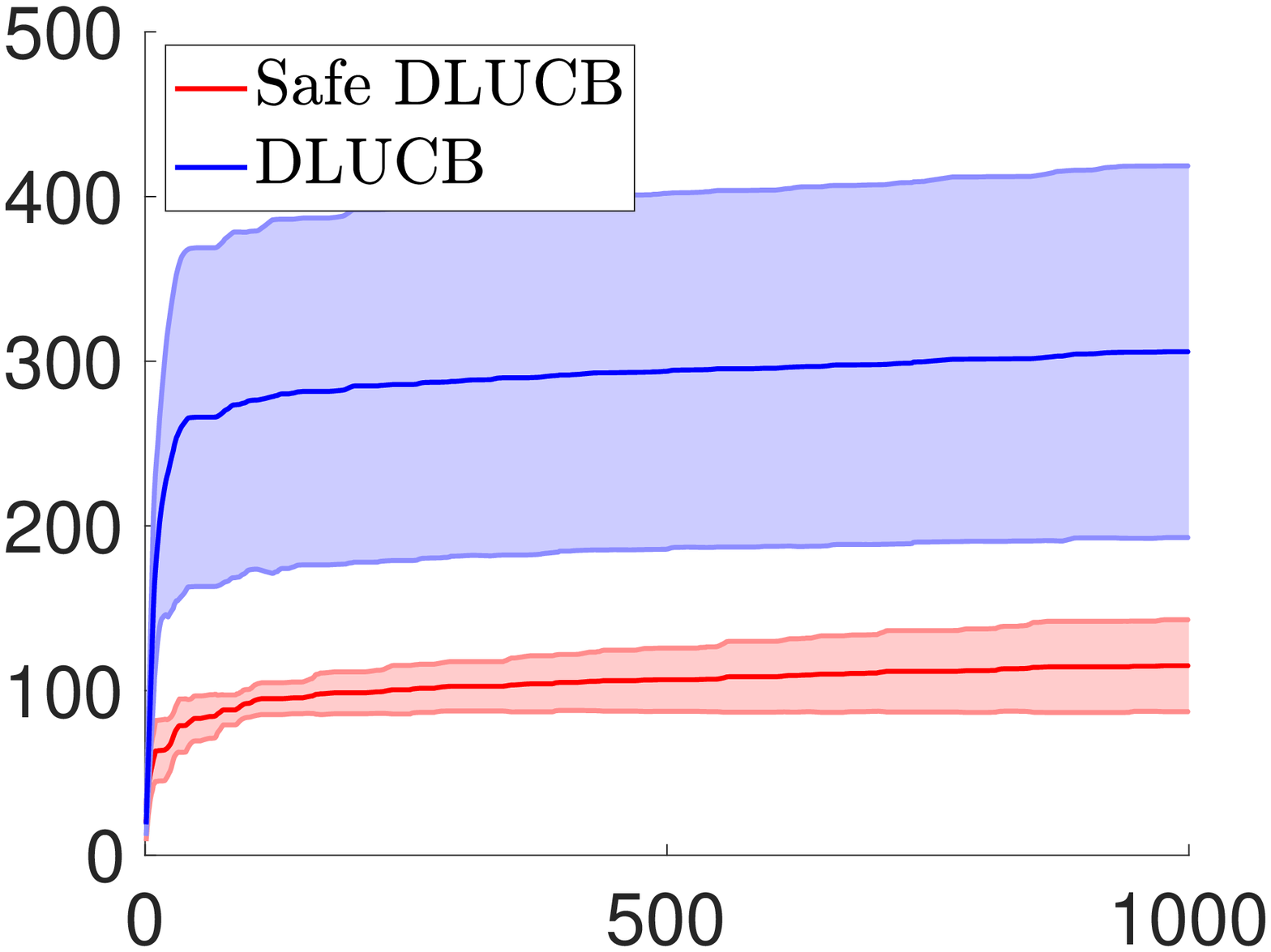}};
\node at (-2.6,0) [rotate=90,scale=0.9]{Regret,$R_t$};
\node at (0,-2) [scale=0.9]{Iteration, $t$};
\end{tikzpicture}
\caption{Safe-DLUCB vs DLUCB}
\label{fig:dlucbvssafe-dlucb}
\end{subfigure}
\centering
\begin{subfigure}{2.3in}
\begin{tikzpicture}
\node at (0,0) {\includegraphics[scale=0.25]{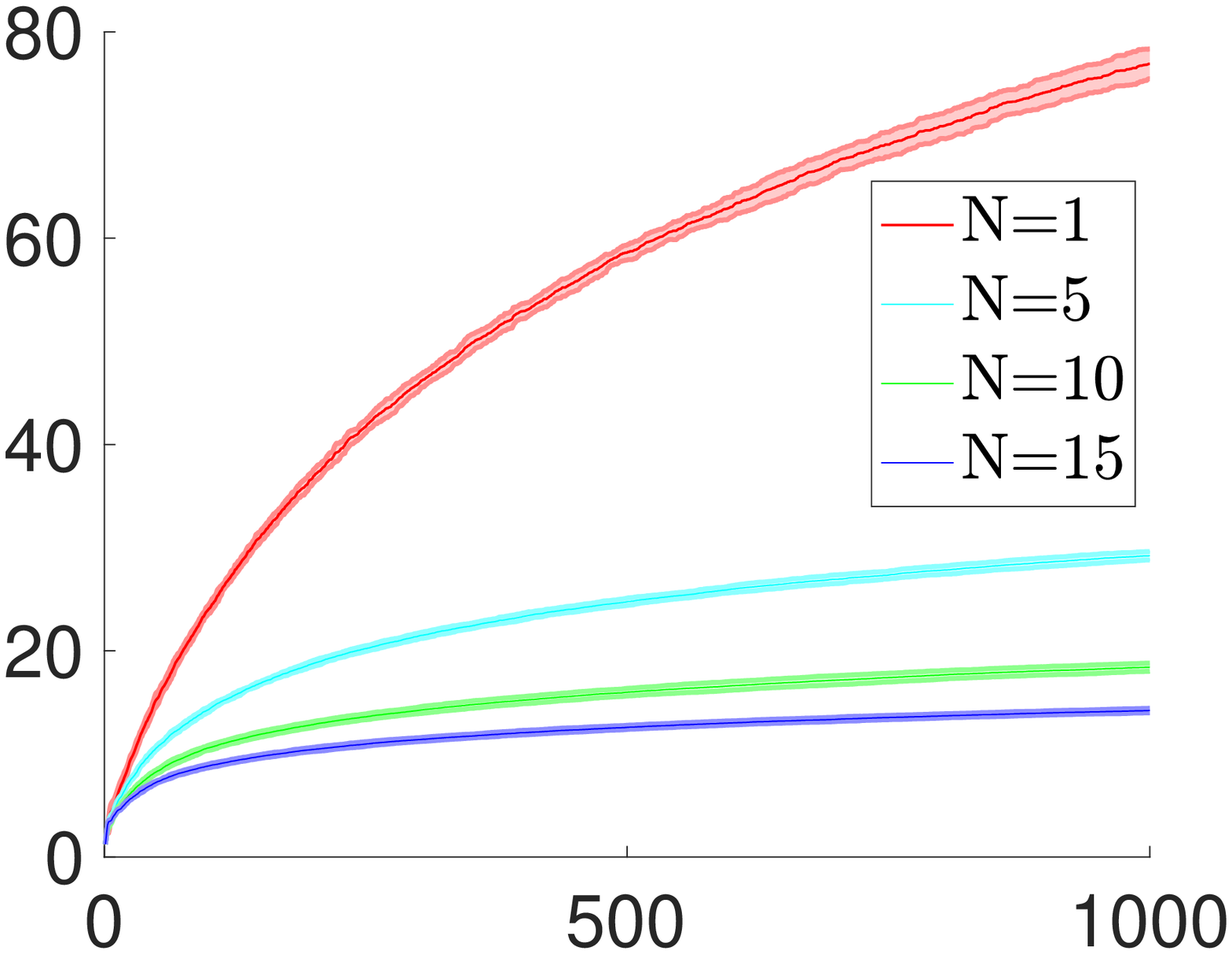}};
\node at (-2.6,0) [rotate=90,scale=0.9]{Per-agent regret, $\frac{R_t}{N}$};
\node at (0,-2) [scale=0.9]{Iteration, $t$};
\end{tikzpicture}
\caption{DLUCB}
\label{fig:numberofagenterror}
\end{subfigure}
\caption{The shaded regions show standard deviation around the mean. The results are averages over 20 problem realizations. See text for details.}
\end{figure}

\begin{figure}
\centering
\begin{subfigure}{2.3in}
\begin{tikzpicture}
\node at (0,0) {\includegraphics[scale=0.25]{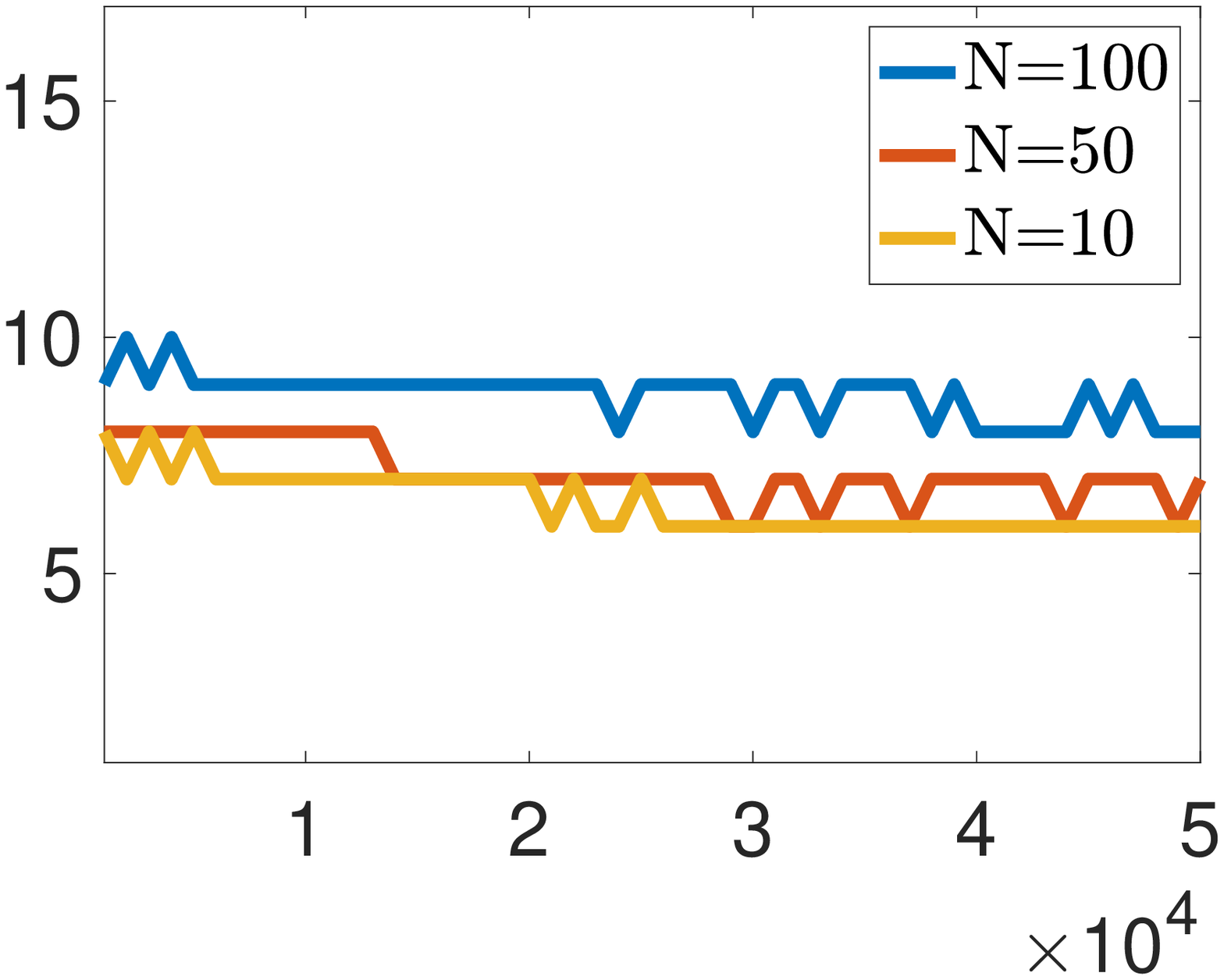}};
\node at (-2.6,0) [rotate=90,scale=0.9]{\# of communication phases};
\node at (0,-2) [scale=0.9]{Horizon, $T$};
\end{tikzpicture}
\caption{$d=5$}
\label{fig:numberofepochsN}
\end{subfigure}
\centering
\begin{subfigure}{2.3in}
\begin{tikzpicture}
\node at (0,0) {\includegraphics[scale=0.25]{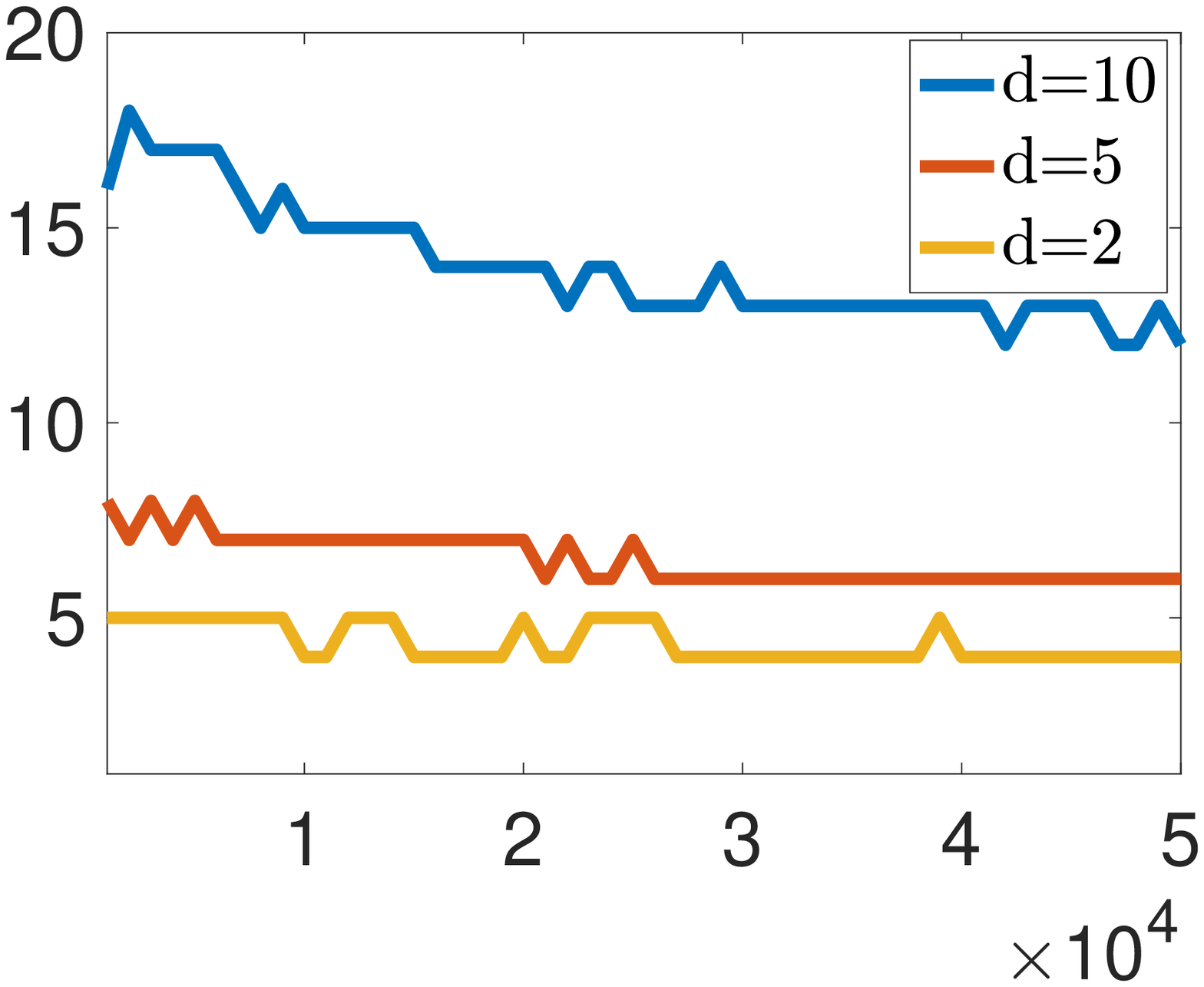}};
\node at (-2.6,0) [rotate=90,scale=0.9]{\# of communication phases};
\node at (0,-2) [scale=0.9]{Horizon, $T$};
\end{tikzpicture}
\caption{$N=10$}
\label{fig:numberofepochsd}
\end{subfigure}
\caption{Number of communication phases for different values of the time-horizon T for the RC-DLUCB algorithm. In agreement with Theorem \ref{thm:RC-DLUCBcommcost} the number of communication phases is independent of $T$.}
\label{fig:numberofepochs}
\end{figure}

All the experiments are implemented in Matlab on a 2020 MacBook Pro with 32GB of RAM. This section complements our discussions in Section \ref{sec:sim} with extra details and experiments. In all our experiments, we have set $\epsilon=1/(4d+1)$. 
In our simulations, we have implemented a slightly modified version of DLUCB such that at each round $t>S$, each agent further exploits its own information that was gathered in rounds $1,\ldots,S$. Concretely, we did not execute Line \ref{linee11} of Algorithm \ref{alg:completeDLUCB}. Our numerical results show that this minor modification performs marginally better than the original Algorithm \ref{alg:completeDLUCB}. It is worth noting that our theoretical results (which hold for a more conservative algorithm where agents do not use their own information of rounds $1,\ldots,S$) remain true for the slightly improved version that we have implemented in this section. 

In Figure \ref{fig:graphtopologies}, we present the graph topologies that were used in our experiments. For instance, recall that, in Figures \ref{fig:dlucb} and \ref{fig:rc-dlucb} in Section \ref{sec:sim}, we ran DLUCB and RC-DLUCB on the graphs in Figures \ref{fig:star}, \ref{fig:ring}, \ref{fig:random}, and \ref{fig:complete} to show the effect of various graph topologies and connectivity level on the regret performance of DLUCB.

Figures \ref{fig:ringcompare}, \ref{fig:starcompare}, \ref{fig:randomcompare}, and \ref{fig:completecompare} show the standard deviation around the regret averages that were depicted in Figures \ref{fig:dlucb} and \ref{fig:rc-dlucb}. Averages are over 20 problem realizations. Specifically, we compare the regret of DLUCB and RC-DLUCB on different graph topologies. There are several observations worth emphasizing here. First, the plots confirm our theoretical findings that DLUCB slightly outperforms RC-DLUCB in all network topologies. However, recall that RC-DLUCB has significantly lower communication cost. Second, the simulations confirm that regret performance of either of the two algorithms is better the larger the spectral gap $1-|\la_2|$. Third, note that in the first $S$ rounds, before any communication has yet happened, the regret grows almost linearly. The performance improves drastically, confirming the sub-linear trend of our bounds, once using mixed information begins. Further note that the value of $S$ (depicted at each figure) increases as the $|\lambda_2|$ increases; see Eqn. \eqref{eq:cheb}.

Next, in Figure \ref{fig:dlucbvssafe-dlucb}, we compare the average regret of (i) Safe-DLUCB and (ii) DLUCB with oracle access to the unknown safe set for the reference. Here, we considered a Random Erdős–Rényi graph with parameters $p=0.5$ and $N=20$.
Moreover, $\mu_\ast$ is drawn from $\mathcal{N}(0,I_5)$ and then normalized to unit norm and the constraint boundary $c$ is drawn uniformly from [0,1] ($\x_0=\mathbf{0}, c_0=0$). As expected, DLUCB that assumes knowledge of $\boldsymbol\mu_\ast$ outperforms Safe-DLUCB. However, despite choosing actions conservatively from inner approximation safe sets, the performance of Safe-DLUCB appears acceptable. In particular it shows sublinear growth of similar order as that of DLUCB.

Also, Figure \ref{fig:numberofagenterror} shows the standard deviation around the regret averages that were depicted in Figure \ref{fig:numberofagent} over 20 problem realizations.

Finally, in Figure \ref{fig:numberofepochs} we have numerically confirmed the result of Theorem \ref{thm:RC-DLUCBcommcost}, according to which the number of communication phases of RC-DLUCB is $\Oc(d\sqrt{N})$, i.e., \emph{independent on the time horizon $T$}. Specifically, in Figure \ref{fig:numberofepochsN}, we have fixed $d=5$ and random graphs with $N=10$, $50$, and $100$ nodes and in Figure \ref{fig:numberofepochsd}, we have used a fixed random graph with $N=10$ and settings with $d=2$, $5$, and $10$. We ran RC-DLUCB for different values of the time-horizon $T$, and, for each of them, we have recorded the number of communication phases. The result confirms our theoretical findings in Section~\ref{sec:commcostofrcdlucb} and \eqref{eq:numberofepochs} that the total number of communication phases does not depend on $T$.

\end{document}